\documentclass[twoside]{article}

\usepackage[accepted]{aistats2022}
\usepackage{minitoc}
\usepackage[hidelinks]{hyperref}
\usepackage[round]{natbib}

\usepackage{amsmath}         
\usepackage{amssymb}         
\usepackage{mathtools}       
\usepackage{amsthm}          
\usepackage{thmtools}        
\usepackage{bbm}             
\usepackage{bm}              
\usepackage{cancel}          
\usepackage{upgreek}
\usepackage{aligned-overset}
\usepackage{xcolor}
\usepackage[toc,page,header]{appendix}
\usepackage{titletoc}
\usepackage{float}
\usepackage{tikz}

\usetikzlibrary{shapes,decorations,arrows,calc,arrows.meta,fit,positioning}
\tikzset{
	-Latex,auto,node distance =1 cm and 1 cm,semithick,
	state/.style ={ellipse, draw, minimum width = 0.7 cm},
	point/.style = {circle, draw, inner sep=0.04cm,fill,node contents={}},
	bidirected/.style={Latex-Latex,dashed},
	el/.style = {inner sep=2pt, align=left, sloped}
}

\renewcommand{\ss}[1]{_{\text{#1}}}

\usepackage{etoolbox}        
\usepackage[inline]{enumitem}              
\usepackage[
	noabbrev,
	capitalize,
	nameinlink]{cleveref}    
\usepackage[symbols,nogroupskip,sort=none]{glossaries-extra} 
\input{math}
\makeatletter
\patchcmd\thmt@mklistcmd
    {\thmt@thmname}
    {\check@optarg{\thmt@thmname}}
    {}{}
\patchcmd\thmt@mklistcmd
    {\thmt@thmname\ifx}
    {\check@optarg{\thmt@thmname}\ifx}
    {}{}
\protected\def\check@optarg#1{%
    \@ifnextchar\thmtformatoptarg\@secondoftwo{#1}%
}
\makeatother

\let\oldlistoftheorems\listoftheorems

\renewcommand{\listoftheorems}{
    \renewcommand{\listtheoremname}{List of Theorems}
    \oldlistoftheorems[ignoreall, show={theorem}]
}

\newlength{\thmtopsep}
\newlength{\thmbotsep}
\newtheoremstyle{theoremstyle}
    {\thmtopsep}{\thmbotsep}
    {}           
    {}           
    {\bfseries}  
    {.}          
    {.5em}       
    {}           
\theoremstyle{theoremstyle}
\newtheorem{theorem}{Theorem}

\newtheorem{corollary}[theorem]{Corollary}
\newtheorem{definition}[theorem]{Definition}
\newtheorem{example}[theorem]{Example}

\newtheorem{lemma}[theorem]{Lemma}

\newtheorem{proposition}[theorem]{Proposition}

\newtheorem{remark}[theorem]{Remark}
\newtheorem{conjecture}[theorem]{Conjecture}
\newtheorem{openproblem}[theorem]{Open Problem}
\crefname{assumption}{Assumption}{Assumptions}
\Crefname{assumption}{Assumption}{Assumptions}
\crefname{corollary}{Corollary}{Corollaries}
\Crefname{corollary}{Corollary}{Corollaries}
\crefname{definition}{Definition}{Definitions}
\Crefname{definition}{Definition}{Definitions}
\crefname{example}{Example}{Examples}
\Crefname{example}{Example}{Examples}
\crefname{fact}{Fact}{Facts}
\Crefname{fact}{Fact}{Facts}
\crefname{lemma}{Lemma}{Lemmas}
\Crefname{lemma}{Lemma}{Lemmas}
\crefname{model}{Model}{Models}
\Crefname{model}{Model}{Models}
\crefname{proposition}{Proposition}{Propositions}
\Crefname{proposition}{Proposition}{Propositions}
\crefname{question}{Question}{Questions}
\Crefname{question}{Question}{Questions}
\crefname{remark}{Remark}{Remarks}
\Crefname{remark}{Remark}{Remarks}
\crefname{theorem}{Theorem}{Theorems}
\Crefname{theorem}{Theorem}{Theorems}

\newlist{asslist}{enumerate}{1}
\setlist[asslist]{
    ref=\theassumption.(\arabic*),
    label=(\arabic*),
    itemsep=-0.25\baselineskip,
    topsep=0.25\baselineskip
}
\crefname{asslisti}{Assumption}{Assumptions}
\Crefname{asslisti}{Assumption}{Assumptions}
\newlist{corlist}{enumerate}{1}
\setlist[corlist]{
    ref=\thecorollary.(\arabic*),
    label=(\arabic*),
    itemsep=-0.25\baselineskip,
    topsep=0.25\baselineskip
}
\crefname{corlisti}{Corollary}{Corollaries}
\Crefname{corlisti}{Corollary}{Corollaries}
\newlist{deflist}{enumerate}{1}
\setlist[deflist]{
    ref=\thedefinition.(\arabic*),
    label=(\arabic*),
    itemsep=-0.25\baselineskip,
    topsep=0.25\baselineskip
}
\crefname{deflisti}{Definition}{Definitions}
\Crefname{deflisti}{Definition}{Definitions}
\newlist{exlist}{enumerate}{1}
\setlist[exlist]{
    ref=\theexample.(\arabic*),
    label=(\arabic*),
    itemsep=-0.25\baselineskip,
    topsep=0.25\baselineskip
}
\crefname{exlisti}{Example}{Examples}
\Crefname{exlisti}{Example}{Examples}
\newlist{factlist}{enumerate}{1}
\setlist[factlist]{
    ref=\thefact.(\arabic*),
    label=(\arabic*),
    itemsep=-0.25\baselineskip,
    topsep=0.25\baselineskip
}
\crefname{factlisti}{Fact}{Facts}
\Crefname{factlisti}{Fact}{Facts}
\newlist{lemlist}{enumerate}{1}
\setlist[lemlist]{
    ref=\thelemma.(\arabic*),
    label=(\arabic*),
    itemsep=-0.25\baselineskip,
    topsep=0.25\baselineskip,
    topsep=0.25\baselineskip
}
\crefname{lemlisti}{Lemma}{Lemmas}
\Crefname{lemlisti}{Lemma}{Lemmas}
\newlist{modlist}{enumerate}{1}
\setlist[modlist]{
    ref=\themodel.(\arabic*),
    label=(\arabic*),
    itemsep=-0.25\baselineskip,
    topsep=0.25\baselineskip
}
\crefname{modlisti}{Model}{Models}
\Crefname{modlisti}{Model}{Models}
\newlist{proplist}{enumerate}{1}
\setlist[proplist]{
    ref=\theproposition.(\arabic*),
    label=(\arabic*),
    itemsep=-0.25\baselineskip,
    topsep=0.25\baselineskip
}
\crefname{proplisti}{Proposition}{Propositions}
\Crefname{proplisti}{Proposition}{Propositions}
\newlist{qlist}{enumerate}{1}
\setlist[qlist]{
    ref=\theremark.(\arabic*),
    label=(\arabic*),
    itemsep=-0.25\baselineskip,
    topsep=0.25\baselineskip
}
\crefname{qlisti}{Question}{Questions}
\Crefname{qlisti}{Question}{Questions}
\newlist{remlist}{enumerate}{1}
\setlist[remlist]{
    ref=\theremark.(\arabic*),
    label=(\arabic*),
    itemsep=-0.25\baselineskip,
    topsep=0.25\baselineskip
}
\crefname{remlisti}{Remark}{Remarks}
\Crefname{remlisti}{Remark}{Remarks}
\newlist{thmlist}{enumerate}{1}
\setlist[thmlist]{
    ref=\thetheorem.(\arabic*),
    label=(\arabic*),
    itemsep=-0.25\baselineskip,
    topsep=0.25\baselineskip
}
\crefname{thmlisti}{Theorem}{Theorems}
\Crefname{thmlisti}{Theorem}{Theorems}

\begin{document}
\doparttoc 
\faketableofcontents 

%

%

\twocolumn[

\aistatstitle{Wide Mean-Field Bayesian Neural Networks Ignore the Data}

\aistatsauthor{ 
Beau Coker$^*$\footnotemark[1] \And 
Wessel P.~Bruinsma$^*$\footnotemark[2]\footnotemark[3] \And  
David R.~Burt$^*$\footnotemark[2] \And 
Weiwei Pan\footnotemark[1] \And
Finale Doshi-Velez\footnotemark[1]}


\aistatsaddress{
\footnotemark[1]Harvard University\;\;
\footnotemark[2]University of Cambridge\;\;
\footnotemark[3]Invenia Labs
} 
]

\begin{abstract}
Bayesian neural networks (BNNs) combine the expressive power of deep learning with the advantages of Bayesian formalism. In recent years, the analysis of wide, deep BNNs has provided theoretical insight into their priors and posteriors. However, we have no analogous insight into their posteriors under approximate inference. In this work, we show that mean-field variational inference \emph{entirely fails to model the data} when the network width is large and the activation function is odd. Specifically, for fully-connected BNNs with odd activation functions and a homoscedastic Gaussian likelihood, we show that the \emph{optimal} mean-field variational posterior predictive (i.e., function space) distribution converges to the prior predictive distribution as the width tends to infinity. We generalize aspects of this result to other likelihoods. 
Our theoretical results are suggestive of underfitting behavior previously observered in BNNs. While our convergence  bounds are non-asymptotic and constants in our analysis can be computed, they are currently too loose to be applicable in standard training regimes. Finally, we show that the optimal approximate posterior need not tend to the prior if the activation function is not odd, showing that our statements cannot be generalized arbitrarily.
\end{abstract}


\section{INTRODUCTION}\label{sec:introduction}
Bayesian neural networks (BNNs) provide a systematic method of capturing uncertainty in neural networks by placing priors on the weights of the network. Although it has been speculated for decades that BNNs are capable of combining the benefits of Bayesian inference and deep learning, we are only beginning to understand the theoretical properties of this model class and its associated inference techniques. 
One tool for understanding the behavior of modern BNNs with large architectures is to study the limiting behavior of this model as the number of hidden units in each layer, i.e., the \emph{width} of the model, goes to infinity. In this case, the prior predictive distribution of a BNN converges in distribution to the \textit{NNGP}, a Gaussian process (GP) with the \emph{neural network kernel} that depends on the prior on the weights and architecture of the network \citep{neal_1996, matthews_2018}. Analogously, in the case of regression with a Gaussian likelihood, the associated BNN posterior converges to the NNGP posterior \citep{hron_2020}.

However, since exact inference for BNNs is intractable, approximate inference is commonly used in practical settings. While asymptotically exact sampling MCMC methods have been successfully applied to BNNs \citep{neal_1996,izmailov2021bayesian}, these methods can require considerable amounts of computation and it is generally not feasible to ensure mixing.  Variational inference offers a computationally appealing alternative by converting the problem of (approximate) inference into a gradient-based optimization problem. 

Unfortunately, the properties of commonly used approximations of BNN posteriors, like mean-field variational inference (MFVI), have not been extensively studied. MFVI assumes complete posterior independence between the weights, but generalizing asymptotic analysis  to  this  approximation  is  non-trivial. Unlike the true posterior predictive distribution, we do not know if the variational posterior predictive distribution approaches a GP as the width approaches infinity. We also do not know if documented properties of BNNs in the finite-width regime generalize to the wide limit. Empirical evidence suggests that finite BNNs trained with MFVI underestimate certain types of uncertainty \citep{Foong:2019:On_the_Expressiveness_of_Approximate} and underfit the data \citep{tomczak_2021, dusenberry2020efficient}. In the case of single hidden layer networks with ReLU activation, \citep{Foong:2019:On_the_Expressiveness_of_Approximate} showed that MFVI networks underestimate uncertainty in-between clusters of data, but their proof fundamentally cannot be extended to the case of several hidden layers. We establish the strong theoretical results for MFVI networks, under the assumption that they are sufficiently wide. 
In this paper, we show that, unfortunately, a number of notable deficiencies of these approximate posteriors become more severe as width increases. For mean-field variational Bayesian neural networks of arbitrary depth with odd, Lipschitz activation functions, we prove a surprising result: the optimal variational posterior predictive distribution converges to the prior predictive distribution as the width tends to infinity. That is, asymptotically, the mean-field variational posterior predictive distribution of a wide BNN completely ignores the data, unlike the true posterior predictive distribution. Furthermore, we derive non-asymptotic, computable bounds that offer insight into the relative rates with which the number of observations, depth, and width of the network affect this convergence. The bounds we prove in their current form are generally too loose to provide numerically useful results for networks of the commonly trained widths, but they offer theoretical support for previously observed issues of underfitting in these networks. 
Finally, we show by a counterexample that this result does not hold for non-odd activation functions, including ReLU, but we provide an example showing that ReLU BNNs can nonetheless underfit data. Code to reproduce all of the experiments is available on GitHub.\footnote{\url{https://github.com/dtak/wide-bnns-public}}


\section{RELATED WORK}\label{sec:related-work}

    
	\paragraph{Wide-limits of BNNs.}
	There are many works that analyze distributions over wide neural networks with the goal of gaining theoretical insight into neural network performance. As the width tends to infinity, \citet{neal_1996} showed that single-layer, fully-connected BNN priors with bounded activation functions converge to GPs.  \Citet{Lee:2017:Deep_Neural_Networks_as_Gaussian} and \citet{ matthews_2018} extend this result to deeper networks with activations that satisfy a ``linear envelope'' condition (which includes ReLU, and is implied by Lipschitz-ness). \citet{hron_2020} extend the result by showing BNN posteriors converge to GP posteriors. All of these works can be seen as offering insights into modeling assumptions made when employing BNNs. Unfortunately, the \textit{true} BNN posterior is computationally intractable for all but the smallest networks. In contrast, we analyze properties of approximate inference, which allows us to make statements about the BNN posterior typically used in practice.
	
	\paragraph{Neural Tangent Kernel.}
	Other works analyze wide neural networks after training the weights with gradient descent, showing that the network output approaches kernel regression with the neural tangent kernel (NTK) \citep{jacot_2020, lee_2019}.
	This also provides a Bayesian interpretation to ensembles of trained neural networks \citep{he_2020}.
	The key insight in these works is that as the width increases, the weight parameters change less and less during training, permitting the network to be approximated by a first-order Taylor expansion around the initial weights.
	For this phenomenon to happen, these works assume the weights are unregularized during training \citep{chen_2020}.
	In contrast, in variational inference, one trains the variational parameters of a distribution over the weights, rather than the weights themselves. Because the variational parameters are regularized by the Kullback-Leibler (KL) divergence to the prior over the weights, the variational parameters do not stick near their initial values and thus the same first-order approximation cannot be used. We show that this regularization is too strong for wide networks, since it forces the resulting approximation of the posterior to converge to the prior as the width tends to infinity.
 Unlike NTK our result does not rely on the dynamics of any particular optimization algorithm, and instead characterizes the optimal posterior.
	
	\paragraph{Issues with MFVI Inference in BNNs.}
	 Many works have empirically observed challenges with mean-field  approximations to Bayesian neural networks. \Citet{mackay1992practical} noted deficiencies of factorized Laplace approximations in single-hidden layer BNNs. However, little is known theoretically about mean-field variational inference in BNNs.  \Citet{Foong:2019:On_the_Expressiveness_of_Approximate} showed that single-hidden layer networks with ReLU activations and mean-field distributions over the weights cannot have high variance between two regions with low variance. However, the authors also show that BNNs with two hidden layers can uniformly approximate \textit{any} function-space mean and variance so long as the width is sufficiently large. \cite{farquhar_2020} suggested the universality result could be extended to other properties (e.g.,~higher moments) of the approximate posterior, leading them to recommend training deeper networks. This means that there \textit{exist} mean-field variational distributions that do not exhibit the known pathologies of approximate inference in BNNs. However, despite this existence, we show that even for wide, deep networks the \textit{optimal} mean-field variational distribution (i.e., the one that maximizes the evidence lower bound (ELBO)) converges to the prior, regardless of the data. 
	 
\vspace{1em} 
\citet{trippe_2018} discuss \textit{over-pruning}, which is the phenomenon whereby the variational posterior over many of the output-layer weights concentrates to a point mass around zero, allowing the variational posterior over any corresponding incoming weights to revert to the prior.  This is undesirable behavior because the amount of over-pruning increases with the degree of over-parameterization and because over-pruning degrades performance --- simpler models that do not permit pruning often perform better. As in our work, the explanation for over-pruning centers around the tension between the likelihood term and the KL divergence term in the objective function, the ELBO. 
To reduce the KL divergence, the optimization procedure may result in hidden units being pruned from the model (i.e., since many weights before the last layer can be set to the prior). Ultimately, we show that the KL divergence of the optimal variational posterior can only be so large, which prevents the variational posterior of wide networks from modeling anything but the prior. 

Our work offers theoretical insight into earlier works on underfitting. 
Empirically, it has been found that re-scaling the regularization to the prior improves the performance of BNNs trained with variational inference \citep{Osawa:2019:Practical_Deep_Learning_With_Bayesian}. This is closely related to observations regarding the performance of \textit{cold posteriors}, which is the empirical phenomenon that down-weighting the importance of the KL divergence in the ELBO (and/or overcounting the data in the likelihood) yields better model performance \citep{wenzel_2020}. It is possible this practice serves to undo the over-regularization of the KL divergence that we investigate.

\section{BACKGROUND}\label{sec:background}

We consider the application of Bayesian neural networks in supervised learning: we have observed a dataset with $N$ points, $\{(\vx_n,\vy_n)\}_{n=1}^N$ with inputs $\vx_n \in \R^{D\ss{i}}$ and outputs $\vy_n \in Y$. Our goal is to infer a (probabilistic) mapping from $\R^{D\ss{i}}$ to $Y$ that is consistent with the data and generalizes to new, unseen observations. We use a Bayesian neural network as the model for this mapping.

\paragraph{Bayesian Neural Networks (BNNs).}
Consider the feed-forward neural network of width $M$ and depth $L$ given by
\begin{align}
	\vf(\vx) &= \tfrac{1}{\sqrt{M}}\mW_{L+1} \phi(\vz_{L}) + \vb_{L+1}, \label{eqn:nn-output}\\
	\vz_{l} &= \tfrac{1}{\sqrt{M}}\mW_{l} \phi(\vz_{l - 1}) + \vb_{l} \quad \text{for} \quad \text{$l = 2, \ldots, L$}, \label{eqn:nn-recursion}\\
	\vz_{1} &= \tfrac{1}{\sqrt{D\ss{i}}}\mW_{1} \vx + \vb_{1}\label{eqn:nn-base}
\end{align}
$(\mW_{L+1}, \vb_{L+1}) \in \R^{D\ss{o} \times M} \times \R^{D\ss{o}}$,
$(\mW_\ell, \vb_\ell) \in \R^{M \times M} \times \R^M$ for $\ell = 2, \ldots, L$, and $(\mW_1, \vb_1) \in \R^{M \times D\ss{i}} \times \R^M$ are the weight and bias parameters, respectively; $\phi\colon \R\to\R$ is the activation function, applied element-wise.  

Let $\vtheta$ represent the concatenation of all parameters. A \textit{Bayesian} neural network places a prior distribution $P$ over $\vtheta$ and a likelihood distribution $\mathcal{L}(\vtheta)$ over $Y$ conditional on $\vtheta$.
In this paper, we study the prior composed of independent standard Gaussian distributions over the weights: $\vtheta \sim \Normal(\vnull, \mI)$. Often, we will be interested in the distribution induced over $\vf = \vf_{\vtheta}$ through the randomness in $\vtheta$. For a distribution over the weights, $P'$, we will refer to the distribution induced over $\vf_{\vtheta}$ by $P'$ as the $P'$ \emph{predictive distribution}. We note that this is a minor abuse of terminology, as a predictive distribution would typically be defined over subsets of $Y$ and depends on the likelihood function. For example, in classification, the predictive refers to the distribution over the output of the network (i.e., logits). 

\paragraph{Convergence to Gaussian Processes (GPs).}
As the width $M$ tends to infinity, an application of the central limit theorem reveals that for any finite collection of inputs $\{\vx_s\}_{s=1}^S$, the distribution over the neural network $\{\vf(\vx_s)\}_{s=1}^S$ induced by the prior $P$ converges in distribution to a multivariate normal distribution \citep{neal_1996, matthews_2018}.
In other words, as the width tends to infinity $\vf$ converges to a multi-output Gaussian process, called the \emph{neural network Gaussian process} (NNGP). 

\paragraph{Variational Inference.}
Unfortunately, the posterior distribution of a finite-width BNN is not available in closed form. Markov chain Monte Carlo (MCMC) methods can be employed to approximately sample from the posteriors (e.g.,~\citealp{izmailov2021bayesian}); however due to a high-dimensional and multi-modal posterior, these methods will generally not mix in a practical amount of time. Because of its advantageous computational properties on high-dimensional problems, variational inference is an appealing alternative \citep{blundell2015weight}. Variational inference proposes a tractable family of distributions $\mathcal{Q}$ and finds an approximation of the true posterior $Q\in\mathcal{Q}$. This approximation is found by minimizing the KL divergence between $Q$ and the true posterior, which is equivalent to maximizing a lower bound on the marginal likelihood called the evidence lower bound (ELBO):
\begin{equation}
	\label{eq:elbo}
	\text{ELBO}(Q) = \mathbb{E}_{\theta\sim Q}[\log \mathcal{L}(\vtheta)] - \KL(Q, P),
\end{equation}
The first term in the ELBO is the expected log likelihood, which measures how well the model fits the data, and the second term is a regularization term, which measures how close $Q$ is to the prior $P$. 

A common choice for the family of variational distributions $\mathcal{Q}$ is the set of factorized (independent) Gaussian distributions.
Under $Q\in\mathcal{Q}$, we write $\vtheta \sim \Normal(\vmu_Q, \diag(\vsigma^2_Q))$. Since both the prior and variational distribution are Gaussian, the KL divergence can be calculated in closed-form:
\begin{equation}
	\KL(\Q, \P)
	=
	\tfrac12
	(\norm{\vmu_\Q}_2^2 + \norm{r(\vsigma^2_\Q)}_1), \label{eq:kl}
\end{equation}
where $r\colon (0, \infty) \to [0, \infty)$, $r(a) = a - 1 - \log(a)$ is applied element-wise.
Notice that \cref{eq:kl} acts like $\ell^2$-regularization of the mean parameters, which will play an important role in the proof of \cref{thm:convergence-of-moments}.
For this variational family $\mathcal{Q}$, under weak regularity conditions, it can be shown that an optimal solution $Q^* \in \argmax_{Q \in \mathcal{Q}} \operatorname{ELBO}(Q)$ always exists (see \cref{app:existence-mf-solution}).
Note, however, that an optimal solution is certainly not unique, because permutations of neurons have the same expected log-likelihood and KL divergence to the prior.

While mean-field variational inference scales gracefully from a computational perspective, its success ultimately relies on the variational family being sufficiently large so that the maximizer of the ELBO qualitatively resembles the posterior. In the next section, we prove that this fails badly for certain BNN models.

\section{THE VARIATIONAL POSTERIOR PREDICTIVE REVERTS TO THE PRIOR PREDICTIVE}\label{sec:main_analysis}

In this section, we analyze the convergence of optimal mean-field Gaussian variational posterior predictive distributions for Gaussian and other likelihoods. We give a sketch of the proof strategy. Additionally, we discuss the quantitative effect of depth and the number of observations on our results.  

\subsection{Gaussian Likelihood}

We begin by stating a simplified version of our main result for a homoscedastic Gaussian likelihood: under fairly broad conditions, the variational BNN posterior predictive converges to the prior predictive.

We assume in our statements that the prior is $\Normal(\mathbf{0}, \mI)$; we additionally assume the network has no bias after the final hidden layer; an analogous result holds in the case with a final bias.
 
\begin{theorem}[\emph{Convergence in distribution to the prior, simplified}] \label{thm:main}
	Assume a Gaussian likelihood and an odd, Lipschitz activation function.
	Then, for any fixed dataset, as the width tends to infinity,
	any finite-dimensional distribution of any optimal mean-field variational posterior predictive distribution of a BNN of any depth converges to the corresponding finite-dimensional distribution of the NNGP prior predictive distribution.
\end{theorem}
\begin{figure}[t]
	\centering
	\includegraphics[width=.47\textwidth]{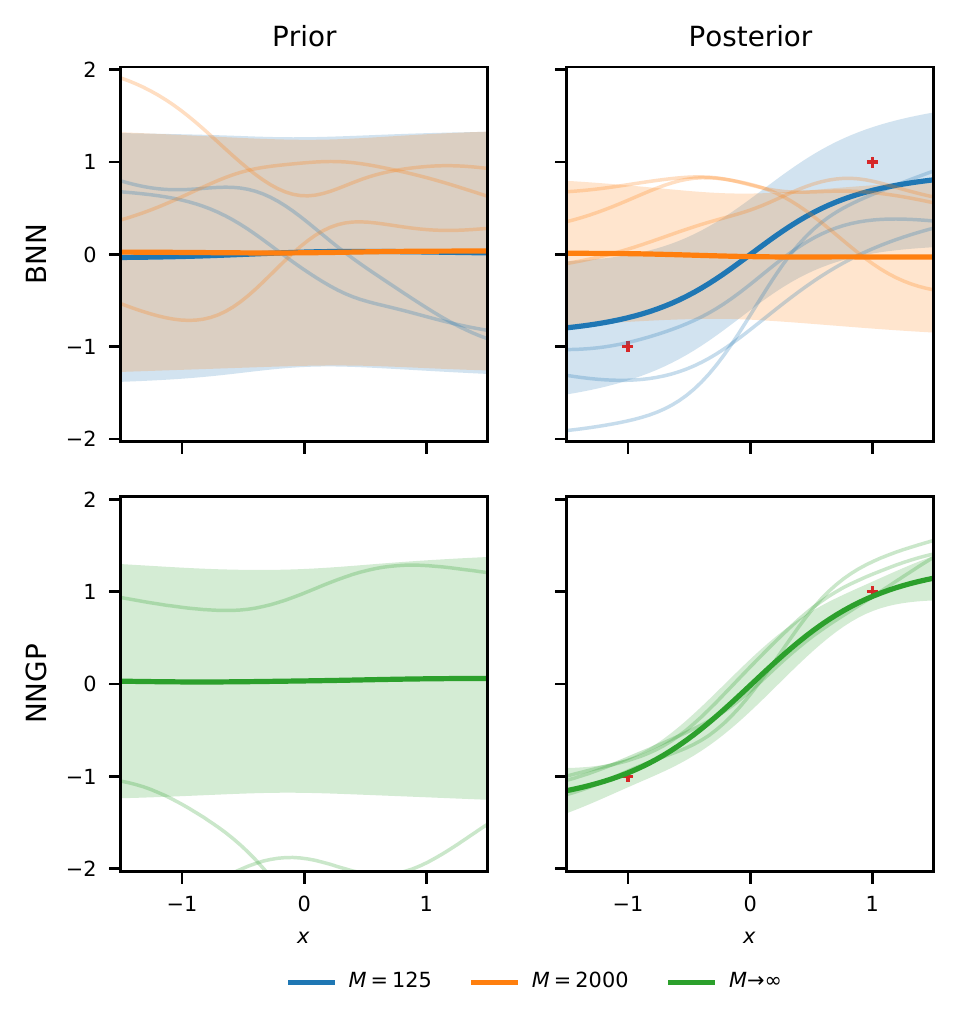}
	\caption{Prior and posterior predictive distributions for single-layer mean-field variational BNNs of different widths compared to the NNGP, to which the true posterior of the BNN converges. For a large width, the mean-field variational BNN ignores the data, unlike the NNGP. The shaded regions constitute $\pm 1$ standard deviation around the means (solid lines). All estimates are based on 1,000 function samples (a few of which are drawn faintly).
	}
	\label{fig:posteriors}
\end{figure}
\cref{fig:posteriors} illustrates our result on a small dataset. In contrast to the \textit{true} BNN posterior predictive, which converges to the NNGP posterior in the limit as the width approaches infinity, the \textit{variational} BNN posterior predictive converges to the NNGP prior, completely ignoring the data. 

A more general version of the theorem, which incorporates the final layer bias and allows for odd functions with a constant offset (e.g., a sigmoid activation), can be found in \cref{app:convergence-in-dist}. The output bias serves only to shift the network by a constant and can sometimes be optimized in closed-form (e.g., in the Gaussian likelihood case it accounts for the overall mean of the observations, $\bar{y}$). \Cref{thm:main} and its generalization apply to several commonly used activation functions, notably tanh, sigmoid and linear.

While a Gaussian likelihood is necessary for our proof of convergence of the entire variational posterior predictive distribution to the prior predictive distribution, we also prove convergence of the first two moments of the variational posterior predictive to the corresponding prior predictive moments for a variety of other likelihoods (logistic, Student's $t$). Additionally, we derive computable bounds on the first two moments of the variational posterior predictive distributions that show that for large, finite widths they must resemble the corresponding prior moments. In contrast, will see in \cref{sec:counterexample} that the oddness assumption in \Cref{thm:main} is necessary for any of these results.

\subsection{General Likelihoods}


\Cref{thm:main} follows from a more general result that holds for a large class of likelihoods. In particular, for a range of likelihoods including Gaussian, Student's $t$, and logistic, we show convergence of the first two moments of the posterior predictive to the corresponding prior predictive moments. The convergence statement has two parts. First, we provide a non-asymptotic bound on the difference between the first two moments of the prior and approximate posterior predictive distributions (\cref{thm:convergence-of-moments}) and goes to $0$ like $O(\tfrac1{\sqrt{M}})$. This aspect is independent of the likelihood and the upper bounds depend on $\KL(\Q,\P)$. Second, we provide an upper bound on $\KL(\Q,\P)$ that depends on the dataset and likelihood, but importantly, is independent of the width of the network (\cref{lem:kl-bound}).



\begin{theorem}[\emph{Bounds on the mean and variance, simplified}]
    \label{thm:convergence-of-moments}
	Under the same conditions as \cref{thm:main} (except for the likelihood assumption), there exist universal constants $c_1,c_2,c_3,c_4>0$ such that
    \begin{align*}
    &\norm{\E_Q[\vf(\vx)] - \E_P[\vf(\vx)]}_2
    \\
    &\;\;\; \le c_1c_2^{L-1} \frac{1+ \tfrac{1}{\sqrt{D\ss{i}}} \|\vx\|_2}{\sqrt{M}} \KL(Q,P) \parens{\KL(Q,P)^{\frac{L-1}{2}} \!\lor 1}, \\
    &\norm{\E_Q[\vf^2(\vx)] - \E_P[\vf^2(\vx)]}_\infty
    \\
    &\;\;\;  \le c_3c_4^{L-1} \frac{1+ \tfrac{1}{D\ss{i}}\|\vx\|_2^2}{\sqrt{M}} \KL(Q,P)^{\tfrac12} \parens{\KL(Q,P)^{L+\tfrac12} \!\lor 1}
    \end{align*}
where $a \lor b=\max(a,b)$.
\end{theorem}
In the special case when $L=1$, our bound on the mean has the simpler form
\begin{equation}\label{eqn:1hl-good-constant}
    \!\!\!\norm{\E_Q[\vf(\vx)] \!-\! \E_P[\vf(\vx)]}_2 \!\le\!  \frac{2}{3}\smash{\parens*{\!\frac{1\!+\!\tfrac{1}{D\ss{i}}\|\vx\|_2^2}{M}\!}^{\!\!\smash{\frac12}} \!\!\KL(Q,\!P)}.\!\!
\end{equation}

While a similar result to  \cref{eqn:1hl-good-constant} can be derived as a special case of \cref{thm:convergence-of-moments}, we derive this result specifically for the case $L=1$ to improve the constant factors; see \cref{app:1hl-good-constants}.
    
Given the bounds in \cref{thm:convergence-of-moments}, we can immediately obtain convergence of the variational predictive mean and variance to the prior as $M\to\infty$ by bounding $\KL(Q^*,P)$ by a constant. 

\begin{lemma}[\emph{Bounds on the KL, simplified}]\label{lem:kl-bound}
For Gaussian, Student's $t$, and logistic likelihood functions, and for an optimal mean-field variational posterior $\Q^*$, $\KL(Q^*,P)$ is bounded by a constant that does not depend on the network width $M$.
\end{lemma}

\cref{fig:convergence} illustrates the upper bound given by \cref{eqn:1hl-good-constant} and \cref{lem:kl-bound} for the optimal posterior, $Q^*$. Empirically, the observed distance of the optimal posterior predictive mean to the prior predictive mean is well below the upper bound, which may be due to our bound of $\KL(Q^*, P)$. See Step 2 of \cref{sec:proof-sketch} for further discussion of this bound. For example, above a width of $10^3$, we observe the distance to the prior predictive within approximately $10^{-2}$, which is well below scale of the $y$ observations ($-1$ and $+1$) and the corresponding distance for the NNGP. 

\cref{fig:many_datasets_mean_tanh} confirms that convergence to the prior leads to a poor fit of the data. We see that across datasets, the RMSE between the posterior mean and the test data increases with the network width (right panel). For comparison, we show the RMSE between the posterior and the prior mean (left panel), which decreases as expected. The datasets ``concrete'' and ``slump'' are from the UCI Machine Learning Repository and the rest are synthetic. The ``2 points'' dataset is the same as in Figures \ref{fig:posteriors} and \cref{fig:convergence}. See \cref{app:exp-setup} for details and an analogous plot of the posterior variance.

\begin{figure}[t]
	\centering
	\includegraphics[width=.47\textwidth]{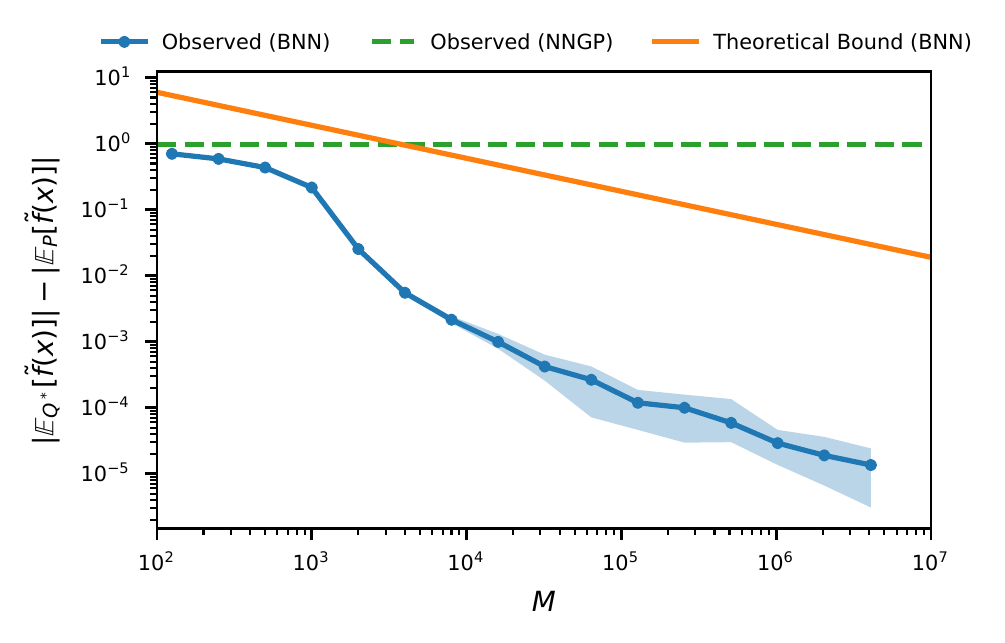}
	\vspace*{-1em}
	\caption{Maximum observed distance of the optimal posterior predictive mean to the prior predictive mean over a grid of points in $[-1,1]$ compared to the theoretical $O(M^{-1/2})$ upper bound given by \cref{thm:convergence-of-moments}. For each $M$ we train 10 single-layer networks on the same two observations shown in \cref{fig:posteriors}. The shaded region shows the range of estimates over the 10 random initializations. We also show the analogous distance for the NNGP.}
	\label{fig:convergence}
\end{figure}

\begin{figure}[t]
	\centering
	\includegraphics[width=.47\textwidth]{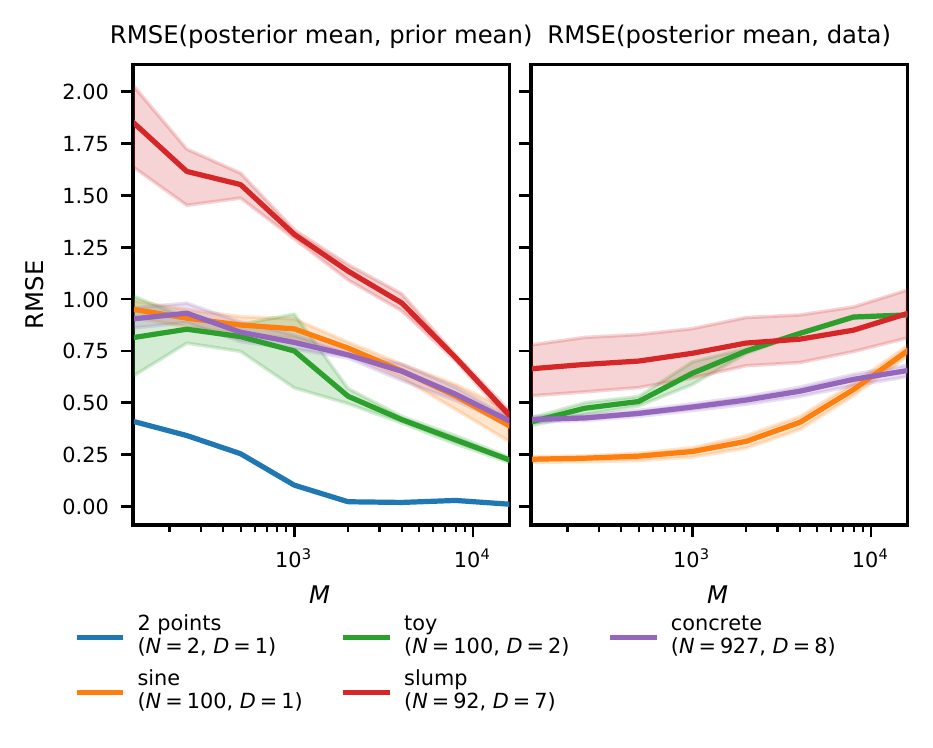}
	\caption{Root mean squared error (RMSE) of the posterior mean to the prior mean (i.e., $\E_P[f(x)]=0$), and the data, $y$, for a few real and synthetic datasets. We use a tanh activation function. The shaded regions are 95\% confidence intervals that reflect 5 train/test splits. The datasets ``concrete'' and ``slump'' are from the UCI Machine Learning Repository, while ``2 points'' is the same dataset from the previous figures. The posterior approaches the prior as the width ($M$) increases, as expected by \cref{thm:main}, resulting in a poor fit of the data. }
	\label{fig:many_datasets_mean_tanh}
\end{figure}

\subsection{Influence of Depth and Number of Observations}

In practice, the upper bounds given by \cref{thm:convergence-of-moments} can be large, limiting their immediate use to practitioners. Furthermore, we see that greater depth increases the bound, since the dependence on $\KL(Q,P)$ grows with $L$ and $c_2, c_4 >1$. It is therefore important to investigate whether a faster rate of convergence can be achieved. 

The case of linear networks (i.e., $\phi(z) = z$) provides for a relevant discussion. We show in \cref{app:convergence-linear-nets} that
\begin{equation}
    \norm{\E_Q[\vf(\vx)] \!-\! \E_\Q[\vf(\vnull)]}_2
    = \Theta(
        M^{-\frac{L}{2}}
        \KL(Q,P)^{\frac{L+1}{2}}
    ).
\end{equation}
Thus, \cref{thm:convergence-of-moments} correctly captures the dependence on the KL divergence, but not the dependence on the width $M$, which is much faster for the linear case: $M^{-\frac{L}{2}}$ versus $M^{-\frac12}$. This raises the question of whether the dependence on $M$ can be improved in case of nonlinear activations. Unfortunately, the answer in general is no. \Cref{app:convergence-lower-bound} shows an example where the dependence is $M^{-\frac12}$.
Although \cref{thm:convergence-of-moments} cannot be generally improved for a generic $Q$, $Q^*$ maximizes the ELBO, which introduces additional structure.
In particular, in the Gaussian case, we know that $\KL(Q^*,P)$ tends to $0$ with $M$ (see Step 3 in \cref{sec:proof-sketch}), which could potentially be used to derive faster rates.
%

It is also important to consider the dependence of \cref{thm:convergence-of-moments} on the number of observations, $N$, which influences the bound through $\KL(Q^*,P)$. If $\E[\norm{\vy}_2^2] = O(N)$, we show in \cref{app:kl-bounds} that $\KL(Q^*,P) \le CN$ for a constant $C>0$. Therefore, the first two moments of the optimal variational posterior predictive approach their respective values under the prior if $\lim_{N,M \to \infty} \frac{N^{L + 1}}{M} = 0$.
Hence, for our results to be non-vacuous for deep networks, $M$ needs to be larger than for shallow networks.

\subsection{Proof Sketch}\label{sec:proof-sketch}
The proof of \cref{thm:main} proceeds in three steps:
\begin{enumerate}[noitemsep,nolistsep]
    \item[]
        \textbf{Step 1.} Establish \cref{thm:convergence-of-moments}, which bounds the posterior predictive mean and variance at any $\vx$ in terms of $\KL(Q, P)$. \\[-0.5em]
    \item[]
        \textbf{Step 2.} Establish \cref{lem:kl-bound}.
        Combined with \cref{thm:convergence-of-moments}, it follows that, in the limit $M \to \infty$, the first and second moments of the approximate posterior predictive and the prior predictive agree. \\[-0.5em]
    \item[]
        \textbf{Step 3.} For a Gaussian likelihood, observe that the ELBO depends only on the first and second moments of the variational posterior predictive distribution at each datapoint and $\KL(Q, P)$. Since (i) $\operatorname{ELBO}(Q) \ge \operatorname{ELBO}(P)$ and (ii) the first and second variational predictive moment converge to the prior predictive moments, it follows that $\KL(\Q,\P) \to 0$.
\end{enumerate}
The complete proof of step 1 can be found in \cref{app:mean-convergence} for the first moment and \cref{app:convergence-of-variance} for the second moment. A more complete version of step 2 that can be made quantitative is given in \cref{app:kl-bounds}. A version of step 3 incorporating the final bias can be found in \cref{app:convergence-in-dist}. Below we expand on each step to give insight into how it is achieved. 
\paragraph{Step 1: Bounding the Moments.}
Here we prove the result for convergence of the mean for a network with $L=1$ and sub-optimal constants. The variance argument follows a generally similar --- though more involved --- argument, and the $L>1$ case is achieved by inductively applying a variant of the argument used in the $L=1$ case.

We have
\begin{align}
    &\|\E[\vf(\vx)]\|_2 \overset{\text{(i)}}{=} \tfrac{1}{\sqrt{M}}\|\E[\mW_2]\E[\phi(\tfrac{1}{\sqrt{D\ss{i}}}\mW_1\vx+\vb_1)]\|_2 \\
    & \qquad \overset{\smash{\text{(ii)}}}{\le} \tfrac{1}{\sqrt{M}}\|\E[\mW_2]\|\ss{F}\|\E[\phi(\tfrac{1}{\sqrt{D\ss{i}}}\mW_1x+\vb_1)]\|_2 \label{eqn:pf-sketch-eqn-top}
\end{align}
where in (i) we use independence and in (ii) we use that $\norm{\vardot}_2 \le \norm{\vardot}\ss{F}$.

Define $\mW' = \mW_1 - \E[\mW_1]$ and $\vb' = \vb_1 - \E[\vb_1]$. Note that $(\mW', \vb') \smash{\disteq} (-\mW, -\vb)$ as these random variables are mean-centered and jointly Gaussian, hence symmetric about $0$. Then, 
\begin{align}
    \E[\phi(\mW'\vx+\vb')] &\overset{\text{(i)}}{=} \E[\phi(-\mW'\vx-\vb')]\\ &\overset{\smash{\text{(ii)}}}{=}-\E[\phi(\mW'\vx+\vb')], 
\end{align}
where (i) follows from the equality in distribution and (ii) by oddness of $\phi$.
From this, we conclude $\E[\phi(\mW'\vx+\vb')]=0$.
We then make the following calculation:
\begin{align}
    &\|\E[\phi(\tfrac{1}{\sqrt{D\ss{i}}}\mW_1\vx+\vb_1)]\|_2  \\ & \; = \!\|\E[\phi(\tfrac{1}{\sqrt{D\ss{i}}}\mW_1\vx+\vb_1) \!-\!\phi(\tfrac{1}{\sqrt{D\ss{i}}}\mW'\vx+\vb')] \|_2 \\
    &\; \overset{\smash{\text{(i)}}}{\le} \E\|\phi(\tfrac{1}{\sqrt{D\ss{i}}}\mW_1\vx+\vb_1)] -\phi(\tfrac{1}{\sqrt{D\ss{i}}}\mW'\vx+\vb')\|_2 \\
    & \; \overset{\smash{\text{(ii)}}}{\le} \E\|\tfrac{1}{\sqrt{D\ss{i}}}(\mW_1-\mW')\vx+(\vb_1-\vb')\|_2\\
    & \; \overset{\smash{\text{(iii)}}}{\le} \|\E[\mW_1]\|\ss{F}\tfrac{1}{\sqrt{D\ss{i}}}\|\vx\|_2 + \|\E[\vb_1]\|_2 \label{eqn:pf-sketch-eqn-bottom}
\end{align}
where (i) uses convexity of norm and Jensen's inequality, (ii) uses that $\phi$ is $1$-Lipschitz, and (iii) combines the triangle inequality and $\norm{\vardot}_2 \le \norm{\vardot}\ss{F}$.

Combining \cref{eqn:pf-sketch-eqn-top} and \cref{eqn:pf-sketch-eqn-bottom} gives
\begin{align}
&\|\E[\vf(\vx)]\|_2 \nonumber \\ & \,\leq\! \tfrac{1}{\sqrt{M}}\|\E[\mW_2]\|\ss{F}\parens*{\!\|\E[\mW_1]\|\ss{F}\tfrac{\|\vx\|_2}{\sqrt{D\ss{i}}}+ \|\E[\vb]\|_2\!}.
\end{align}
We now note that the Frobenius norm $\|\E[\mW_2]\|\ss{F}$ is the $\ell^2$-norm of the mean parameters of weights in the second layer, and similar conditions apply to $\|\E[\mW_1]\|\ss{F}, \|\E[\vb]\|_2$. Recalling \cref{eq:kl},
\begin{align*}
    \|\E[\mW_1]\|\ss{F},\|\E[\mW_2]\|\ss{F}, \|\E[\vb]\|_2 \leq \sqrt{2\KL(\Q,\P)},
\end{align*}
so
\begin{align}
&\|\E[\vf(\vx)]\|_2 \le \tfrac{1}{\sqrt{M}}2(1+\tfrac{1}{\sqrt{D\ss{i}}} \|\vx\|_2)\KL(\Q,\P).
\end{align}
This is of the same form as the bound in \cref{thm:convergence-of-moments}. 

\paragraph{Step 2: Bounding $\KL(Q^*, P)$.}

In order for \cref{thm:convergence-of-moments} to be useful, we need to understand how large $\KL(\Q^*,\P)$ could be. We make the following three assumptions when doing this:
\begin{enumerate}[label=(\roman*)]\itemsep0pt
   \item The likelihood factorizes over data points, i.e.~$\log \mathcal{L}(\vtheta) = \sum_{n=1}^N \log p(\vy_n | \vf_{\vtheta}(\vx_n))$, for some function $p$;
    \item there exists a $C$ such that $\log p(\vy_n | \vf_{\vtheta}(\vx_n)) \leq C$;
    \item for any fixed $\vy_n$, $\log p(\vy_n | \vf_{\vtheta}(\vx_n))$ can be lower bounded by a quadratic function in $\vf_{\vtheta}(\vx_n)$.
\end{enumerate}

By the optimality of $\Q^*$, we have
\begin{align}
	0 &\leq \text{ELBO}(Q^*) -  \text{ELBO}(P) \\
	&= \mathbb{E}_{\vtheta\sim Q^*}[\log \mathcal{L}(\vtheta)]- \KL(Q^*, P) \nonumber\\
	&\qquad\qquad- \mathbb{E}_{\vtheta\sim P}[\log \mathcal{L}(\vtheta)].     
\end{align}
Rearranging and using the assumptions on $\log \mathcal{L}(\vtheta)$,
\begin{align}
    \!\!\!\KL(Q^*, P) &\leq \mathbb{E}_{\vtheta\sim Q^*}[\log \mathcal{L}(\vtheta)] -\mathbb{E}_{\vtheta\sim P}[\log \mathcal{L}(\vtheta)]\\
    & \leq \!CN \!-\! \mathbb{E}_{\vtheta\sim P}[\textstyle\sum_{n=1}^N \!\log p(\vy_n | \vf_{\vtheta}(\vx_n))] \\
    & \leq \! CN \!-\! \E_{\vtheta \sim \P}[\textstyle\sum_{n=1}^N h_n(\vf_{\vtheta}(\vx_n))]
\end{align} 
where $h_n$ is quadratic. Since $h_n$ is quadratic, $\E_{\P}[h_n(\vf_{\vtheta}(\vx_n))]$ is a linear combination of the first and second moments of $\vf_{\vtheta}(\vx_n)$. As we know the moments of $\vf_{\vtheta}(\vx_n)$ converge to those of the corresponding NNGP \citep{matthews_2018}, and since any convergent sequence is bounded, this gives an upper bound on $\KL(\Q^*,\P)$ that is independent of width. 

\paragraph{Step 3: Convergence in Distribution.}
For Gaussian likelihoods, we can go one step further and prove \cref{thm:main} using the optimality of $\Q^*$ yet again. In particular, by the same argument as in the previous paragraph, we have 
\begin{align}
\!\KL(Q^*, P) &\leq\! \mathbb{E}_{\vtheta\sim Q^*}[\log \mathcal{L}(\vtheta)] -\mathbb{E}_{\vtheta\sim P}[\log \mathcal{L}(\vtheta)].
\end{align}
For simplicity, assume $\vy_n=y_n \in \R$ and a homoscedastic likelihood with variance parameter $\sigma^2$ is used. Then, using
\begin{equation}
   \log \mathcal{L}(\vtheta)= -\frac{N}{2}\log 2 \pi\sigma^2-  \frac{1}{2\sigma^2}\sum_{n=1}^N (y_n-f(\vx_n))^2
\end{equation} 
in combination with $|(a - b)^2 - (a - c)^2| \le 2|a||b - c| + |a^2 - b^2|$, we find that
\begin{align}
\KL(Q^*, P) &\leq \sum_{n=1}^N \Big[ 2|y_n| |\E_{\Q^*}[f(\vx_n)] - \E_{\P}[f(\vx_n)]| \nonumber \\ &\qquad + |\E_{\Q^*}[f(\vx_n)^2] - \E_{\P}[f(\vx_n)^2]| \Big].
\end{align}
By \cref{thm:convergence-of-moments}, we conclude $\lim_{M \to \infty} \KL(\Q^*,\P) = 0$. Since the KL divergence between any finite dimensional distribution of the predictive of $\Q^*$ and $\P$ is upper bounded by this KL divergence, we conclude that a similar statement holds for finite-dimensional distributions. Finally, convergence in this sense implies weak convergence, so convergence of finite dimensional distributions of the posterior predictive of $\Q^*$ to the NNGP follows.

\section{NON-ODD ACTIVATIONS} \label{sec:counterexample}

In \cref{sec:main_analysis}, our theorems  assume odd activation functions. The following theorem shows that this assumption is necessary.
\begin{theorem}[\emph{non-odd counterexample, simplified}] \label{thm:counterexample}
Given any non-odd, $1$-Lipschitz activation function $\phi$ (e.g., ReLU), we can construct a homoscedastic Gaussian likelihood and a dataset where the optimal mean-field variational mean is bounded away from the prior mean as the width tends to infinity\footnote{The theorem has additional technical conditions, but applies to all non-odd activation functions used in practice.}.
\end{theorem}

\cref{fig:counterexample} illustrates this counterexample dataset along with the resulting mean-field posterior predictive distributions of networks with ReLU and erf activations (left panel). We also train on the same two observations as in Figures \ref{fig:posteriors} and \ref{fig:convergence} (right panel). In the erf activation case, the posterior predictive converges to the prior predictive on both datasets, as expected by \cref{thm:main}, and in the ReLU activation case the posterior predictive does not converge to the prior predictive on the counterexample dataset, as expected by \cref{thm:counterexample} (recall that we mean without the output bias, which can generally differ from the prior). Interestingly, in the ReLU case, on the dataset that is not constructed as a counterexample, the approximate posterior closely resembles the prior. However, an examination of additional datasets in \cref{fig:many_datasets_mean_relu} reveals the story is generally less clear. For some datasets, the wider networks are closer to the prior than the narrower networks, while for other datasets the opposite is true. See \cref{sec:discussion} for further discussion.

Our proof strategy starts by finding a sequence of variational distributions (indexed by $M$) with a mean function that does not tend to a constant. We then show that there exists a dataset for which the sequence of ELBOs defined by this sequence of variational distribution converges a number that exceeds the ELBOs of any sequence of variational distributions that have a mean function that does tend to a constant. 

The key observation to the counterexample is that in the odd activation case, the expected value of the final layer post-activations, $\E_{\P}[\phi(\vz_L)]$, is zero, whereas in the non-odd activation case this expectation will generally depend on $\vx$. 

To construct the counterexample, we define $Q_M$ as the mean-field variational distribution that is equivalent to the prior except in the last layer, where $\vw_{L+1}\sim \Normal(\tfrac{1}{\sqrt{M}}\mathbf{1},\mI)$. Notice that under $Q_M$ the predictive mean is equal to $\E_{\P}[\phi(\vz_L)]$. $Q_M$ will serve as a candidate set of distributions with non-constant means and ``good'' ELBOs. 
We select the $Y$ values in our dataset to fall very near the mean predictor for $Q_M$, or more precisely, to coincide with the mean predictor as $M \to \infty$. This will ensure $Q_M$ has small error. On the other hand, we can ensure that the error of any predictor that gives a constant prediction is large. The left panel of \cref{fig:counterexample} confirms that the posterior predictive under the odd activation (erf) converges to the prior, whereas the posterior predictive under the non-odd activation (ReLU) is able to model the data. These networks are very wide ($M \approx 4 \times 10^6$).

\begin{figure}[h]
	\centering
	\includegraphics[width=.47\textwidth]{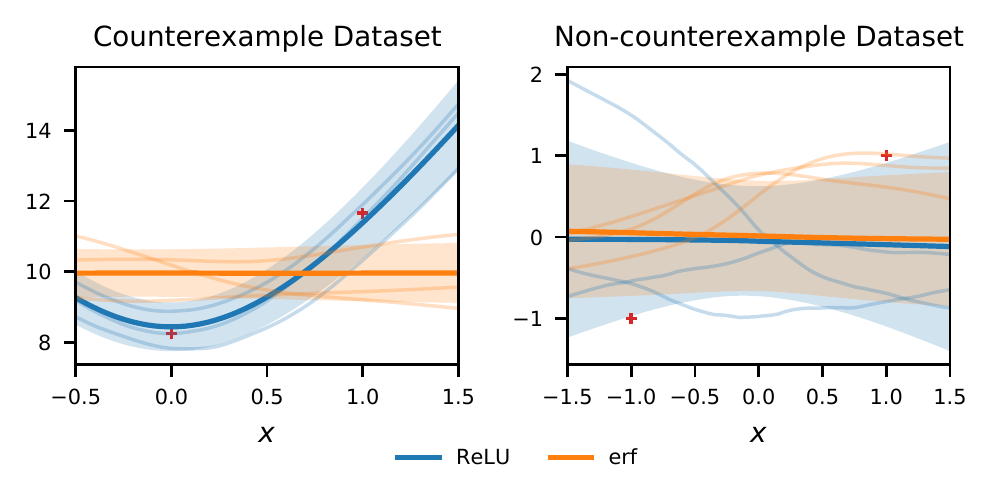}
	\vspace*{-1em}
	\caption{Mean-field posterior predictive distributions for very wide networks with odd and non-odd activations (erf and ReLU, respectively) trained on one of two datasets --- the counterexample we construct (left panel) and the same dataset as in Figures \ref{fig:posteriors} and \ref{fig:convergence}, which does not meet the conditions of the counterexample (right panel). We observe convergence of the posterior predictive to the prior predictive in all cases except for the ReLU network trained on the counterexample dataset.}
	\label{fig:counterexample}
\end{figure}

\begin{figure}[t]
	\centering
	\includegraphics[width=.47\textwidth]{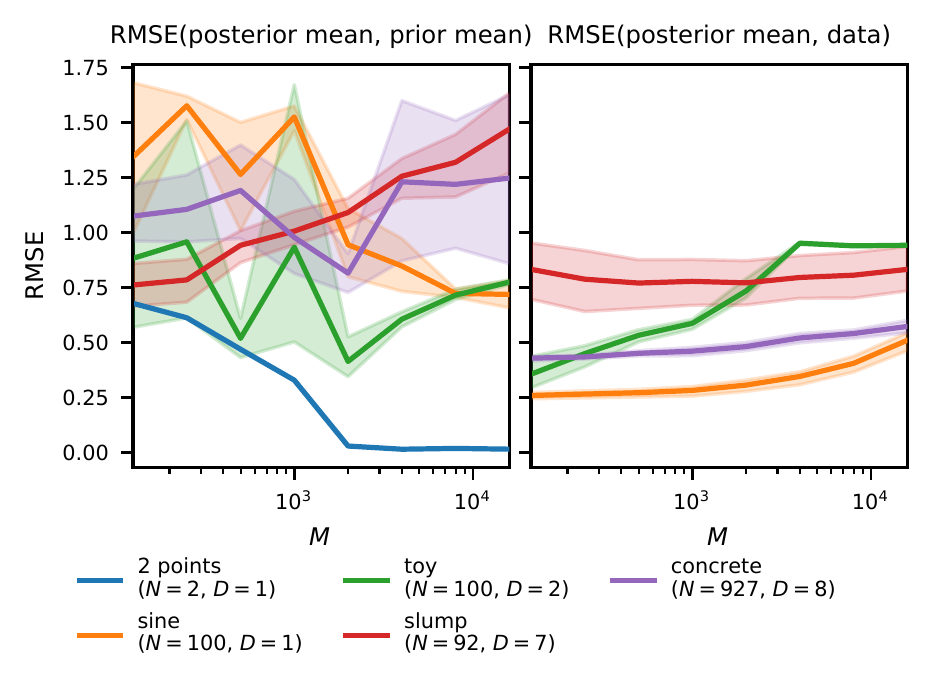}
	\vspace*{-1.2em}
	\caption{Analogous to \cref{fig:many_datasets_mean_tanh} but using a ReLU activation. Wider networks tend not to fit the data as well but it is unclear if this is due to convergence to the prior.}
	\label{fig:many_datasets_mean_relu}
\end{figure}

\section{DISCUSSION}\label{sec:discussion}

\paragraph{What Do Our Results Mean for Bayesian Methods in Over-Parameterized Models?}
The successes of modern deep learning have given strong evidence to the claim that over-parameterized models can lead to better empirical performance.  As the flexibility of the model relative to the amount of data increases, it has been suggested that Bayesian methods have more to offer in terms of promoting generalization \citep{NEURIPS2020_322f6246}. However, in cases when inaccurate inference is combined with very large models, our results prove a crippling and previously unknown limitation to this approach. This highlights the need for accurate inference methods in Bayesian neural networks, as well as robust techniques for monitoring and diagnosing inference quality.


\paragraph{Does Using ReLU Solve All of the Problems with MFVI?}
Another question raised by our results concerns the use of odd activation functions, which are necessary for our results to hold. Can the issues we raised be avoided by simply using a non-odd activation such as ReLU? 
While our counterexample shows there exists a dataset for which the approximate posterior mean under a non-odd activation does not converge to the prior, it is still possible the approximate posterior will converge to the prior on other datasets. Moreover, even without exact convergence to the prior, the approximate posterior could still be a poor model of the data. For the dataset in \cref{fig:counterexample} that does not meet the conditions of the counterexample, this is exactly what we see: a poor model of the data and a close resemblance of the approximate posterior to the prior. Figures \ref{fig:many_datasets_mean_tanh} and \ref{fig:many_datasets_mean_relu} investigate these two attributes --- distance to the data and distance to the prior --- across a variety of datasets. In the tanh case, we see convergence to the prior as expected. However, in the ReLU case the behavior is unclear. For some datasets it is possible the approximate posterior is converging to the prior, whereas for other datasets there is little indication of this. Yet, we emphasize that in all cases (datasets and activations), we see an increasingly poor fit of the data as the network width increases.
Our work frames the characterization of the optimal mean-field posterior under ReLU activations as an important area of future work.

\paragraph{Should MFVI be Abandoned Entirely in BNNs?}
Our results raise an important question for practitioners --- should mean-field posteriors be thrown out, since asymptotically the optimal one converges to something degenerate, or should practitioners merely be careful about the relative scaling of the width, depth, and dataset size? By providing non-asymptotic upper bounds on how much the predictive mean and variance can differ from the prior, our results provide a regime where mean-field variational inference is guaranteed to fail. For example, for a given depth we can provide a width above which the optimal posterior predictive mean and variance are within a given threshold of their values under the prior. Yet, we emphasize our results are upper bounds, with constants that could likely be further optimized. There is possibly a smaller width that would give the same behavior, and this is what we observed in our experiments. These results suggest serious shortcomings of mean-field variational inference, and we would recommend practitioners take great care in applying MFVI, even with networks with narrower widths where our bounds do not provably show the approximate posterior will revert to the prior.

%
%
\section{OPEN PROBLEMS}\label{sec:open-problems}

We believe our analysis leads to several interesting generalizations, which we leave as open problems. The first question we pose is whether \cref{thm:main} can be generalized to all  likelihoods where \cref{thm:convergence-of-moments} applies:

\begin{conjecture}\label{conj:general-conv-in-dist}
For any likelihood satisfying the assumptions needed for \cref{lem:kl-bound}, the optimal variational posterior predictive (excluding the final bias) converges in distribution to the corresponding NNGP. 
\end{conjecture}

A potential avenue for proving \cref{conj:general-conv-in-dist} would be to establish a central limit theorem for any one-dimensional predictive distribution under $\Q$, using that $\KL(\Q,\P)$ is bounded. While the post-activations in the final hidden layer are not exchangeable, they are in some sense close to an exchangeable sequence. Given a central limit theorem \cref{conj:general-conv-in-dist} can be established following the same argument sketched in \cref{sec:proof-sketch}, step 3. Establishing \cref{conj:general-conv-in-dist} would make the consequences of the behavior we analyze to classification much clearer.

Another fascinating open question is the precise non-asymptotic dependence of \cref{thm:convergence-of-moments} on the depth of the network. Our current bounds suggest that deeper networks may need to be wider before the optimal MFVI posterior converges to the prior. However, it is difficult to determine how much of this effect is due to the analysis becoming more complicated, leading to sub-optimal constants. We therefore pose the following, somewhat imprecise open problem,
\begin{openproblem}
Can the dependence of \cref{thm:convergence-of-moments} on $L$ be improved? In particular, should we expect the optimal MFVI posterior in deeper networks to converge more or less quickly to the prior as width increases?
\end{openproblem}

We believe both of these questions are interesting theoretical questions with concrete ramifications for practitioners that may be challenging, but seem to be approachable problems for future research.
\paragraph{Acknowledgements}
The authors would like to thank Andrew Y.K.~Foong for useful discussion and comments. Additionally, the authors would like to thank Richard E.~Turner for helping to facilitate this collaboration. DRB acknowledges funding from the Qualcomm Innovation Fellowship.
Wessel P. Bruinsma was supported by the Engineering and Physical Research Council (studentship number 10436152). 

\bibliography{aistats}
\clearpage
\appendix
\onecolumn \makesupplementtitle

\section*{}
\addcontentsline{toc}{section}{Appendix} 
\part{} 
\parttoc 

\newpage
\glsxtrnewsymbol[description={Matrices are bold, capital letter}]{mats}{\ensuremath{\mathbf{W}, \mathbf{V}, \mathbf{U}}}
\glsxtrnewsymbol[description={Vectors are bold, lower case letter}]{vecs}{\ensuremath{\mathbf{w}, \mathbf{v}, \mathbf{u}}}
\glsxtrnewsymbol[description={Number of neurons per hidden layer (width)}]{M}{\ensuremath{M}}
\glsxtrnewsymbol[description={Number of hidden layers (depth)}]{L}{\ensuremath{L}}
\glsxtrnewsymbol[description={Number of observations}]{N}{\ensuremath{N}}
\glsxtrnewsymbol[description={Shorthand for \ensuremath{\sqrt{\KL(Q,P)}}}]{K}{\ensuremath{K}}
\glsxtrnewsymbol[description={All parameters in neural network}]{theta}{\ensuremath{\theta}}
\glsxtrnewsymbol[description={An arbitrary input}]{x}{\ensuremath{\mathbf{x}}}
\glsxtrnewsymbol[description={An arbitrary output}]{y}{\ensuremath{\mathbf{y}}}
\glsxtrnewsymbol[description={Preactivation for neuron $m$ in hidden layer $\ell$}]{zhm}{\ensuremath{z_{\ell, m}}}
\glsxtrnewsymbol[description={Vector of preactivations at layer h}]{zh}{\ensuremath{\mathbf{z_{h}}}}
\glsxtrnewsymbol[description={Dimensionality of input, i.e.~$\gls{x} \in \mathbb{R}^{D\ss{i}}$}]{Din}{\ensuremath{D\ss{i}}}
\glsxtrnewsymbol[description={Dimensionality of output, i.e.~$\gls{y} \in \mathbb{R}^{D\ss{o}}$}]{Dout}{\ensuremath{D\ss{o}}}
\glsxtrnewsymbol[description={Activation function in neural network (non-linearity)}]{phi}{\ensuremath{\phi}}
\glsxtrnewsymbol[description={Even part of activation function, i.e.~$\phi\ss{e}(a)=\frac{\gls{phi}(a)+\gls{phi}(-a)}{2}$}]{phie}{\ensuremath{\phi\ss{e}}}
\glsxtrnewsymbol[description={Odd part of activation function, i.e.~$\phi\ss{o}(a)=\frac{\gls{phi}(a)-\gls{phi}(-a)}{2}$}]{phio}{\ensuremath{\phi\ss{o}}}
\glsxtrnewsymbol[description={Prior distribution, usually $\Normal(\vnull,\mI)$}]{prior}{\ensuremath{P}}
\glsxtrnewsymbol[description={A variational posterior distribution, $\Normal(\vmu_Q, \mathrm{diag}(\vsigma^2_Q))$}]{varposterior}{\ensuremath{Q}}
\glsxtrnewsymbol[description={Expectation, optionally with subscript to clarify the measure to be integrated over}]{exp}{\ensuremath{\mathbb{E}}}
\glsxtrnewsymbol[description={Variance, optionally with subscript to clarify the measure to be integrated over}]{var}{\ensuremath{\mathbb{V}}}
\glsxtrnewsymbol[description={Diagonal of $\V$}]{vardiag}{\ensuremath{\mathbb{V}\ss{d}}}
\glsxtrnewsymbol[description={Kullback-Leibler divergence}]{kl}{\ensuremath{\mathrm{KL}}}
\glsxtrnewsymbol[description={$\mathbf{f_{\theta}}: \mathbb{R}^{\gls{Din}} \to \mathbb{R}^{\gls{Dout}}$ represents the output of the network with parameters $\theta$}]{networkfn}{\ensuremath{\mathbf{f_{\theta}}}}
\glsxtrnewsymbol[description={$\widetilde{\vf}_{\theta}: \mathbb{R}^{\gls{Din}} \to \mathbb{R}^{\gls{Dout}}$ represents the output of the network with parameters $\theta$, excluding the contribution from the final bias}]{networkfnnobias}{\ensuremath{\widetilde{\vf}_{\theta}}}
\glsxtrnewsymbol[description={$a \lor b= \mathrm{max}(a,b)$}]{max}{\ensuremath{\lor}}
\glsxtrnewsymbol[description={Frobenius Norm of a matrix, equal to the $\ell^2$ norm of the singular values, also equal to the sum of squared entries}]{frob}{\ensuremath{\|\cdot\|\ss{F}}}
\glsxtrnewsymbol[description={Spectral Norm of a matrix, equal to the $\ell^\infty$ norm of the singular values, also the matrix norm induced by the $\ell^2$ norm on vectors}]{spec}{\ensuremath{\|\cdot\|\ss{2}}}
\glsxtrnewsymbol[description={$f(x) \lesssim g(x)$ $\Longleftrightarrow$ there exists an irrelevant proportionality constant $C$ such that   $f(x) \leq Cg(x)$. }]{tildelt}{\ensuremath{\lesssim}}
\printunsrtglossary[type=symbols]
\newpage

\section{Map of the Appendix and Sketches of the Results}\label{app:appendix-intro}


Our main results show that for fully-connected networks with odd, Lipschitz continuous activation functions,

\begin{itemize}
    \item  For Gaussian likelihoods, the variational posterior converges in distribution to the prior as the width of the network tends to infinity.
    \item For a wide class of other likelihoods, including the Student's t likelihood, the first two moments of one-dimensional marginals of the variational posterior converge to the prior.
\end{itemize}

 The key idea in both cases is to upper bound the difference between the first two moments of one-dimensional marginals of any mean-field posterior and the corresponding moments of the prior in terms of its KL divergence to the prior. Crucially, we show that we can derive a bound of this form, \emph{that goes to $0$ as $M$ goes to $\infty$}. While we often make asymptotic statements about the width as these have the simplest form, all of the bounds are non-asymptotic and can be explicitly computed for finite widths.

In the following sketch, we ignore the final bias as it must be handled separately and leads to notational clutter and slightly more unwieldy statements. Hence, the following statements will all be true for a network without a final output bias, and some minor modifications of them is true for a normal fully-connected neural network.
\paragraph{Bounding the mean in terms of $\KL(\Q,\P)$}
We use $\widetilde{\vf}_{\theta}$ to denote the network output excluding the final output bias and even part of the activation, i.e., $\widetilde{\vf}_{\theta} = \vf_{\theta} - \vb_{L+1} - \alpha \mW_{L+1}\vone$. Then if we consider the mean, and use the independence structure of $\Q$,
\begin{align}
    \|\E_{\Q}[\widetilde{\vf}_{\theta}(\vx)]\|_2 = \tfrac1{\sqrt{M}} \norm{\E_{Q}[\mW_{L+1}]\E[\phi\ss{o}(\vz_L(\vx))]}_2 \leq \tfrac1{\sqrt{M}} \norm{\E_{Q}[\mW_{L+1}]}_2\norm{\E_{\Q}[\phi\ss{o}(\vz_L(\vx))]}_2. \label{eqn:mean-top-layer}
\end{align}
The first term, $\norm{\E_{Q}[\mW_{L+1}]}_2$, can directly be upper bounded in terms of $\KL(\Q,\P)$. The second term is more difficult. The simplest thing to do would be to push the norm inside the expectation using Jensen's inequality and use that $\phi$ is assumed to be Lipschitz continuous. However, this prevents us from taking advantage of any cancellation due to the oddness of $\phi$, which we will see is essential to the proof (cf.~\cref{thm:app-counterexample}), and the resulting bound need not tend to $0$ with $M$.

We instead setup a recursion to show that for each $\ell$, $\norm{\E[\phi(\vz_\ell(\vx))]}_2$ is upper bounded in terms of an expression that is independent of $M$. This will be the main work done in \cref{app:main-recursion}, and crucially relies on the oddness and Lipschitz continuity of $\phi$.

\paragraph{Bounding the variance in terms of $\KL(\Q,\P)$}
We work with the un-centered second moment, as in combination with a bound on the mean this implies a bound on the marginal variance.  The proof will proceed by splitting the variance into two terms.  The first term, which we term the diagonal, arises from the product of the variance of the weights with the second moment of each activation.  We show that under the optimal variational posterior, this term is close to the same term under the prior. The second term, which we term the off-diagonal arises from the product of the mean of the weights with the second moment of each activation. Under the prior, this term vanishes.

\paragraph{Convergence for Gaussian likelihoods}
Having established that the mean and variance of the variational posterior both converge to the prior, the proof diverges based on the likelihood. In the case of a homoscedastic, Gaussian likelihood, the analysis becomes particularly nice. By noting that the evidence lower bound essentially has three terms, one depending on the mean at each data point, one depending on the variance at each data point and one depending on the KL divergence between the variational posterior and the prior, we can upper bound the KL divergence between \emph{any variational posterior with an ELBO at least as good as the prior} and the prior, by something that will tend to $0$ with width. Combining this with standard inequalities between divergences on probability measures and the results of \citet{matthews_2018} suffices to show weak convergence of finite marginals of the the optimal posterior to the NNGP prior.  

\paragraph{Convergence for more general likelihoods}
In the case of more general likelihoods, the evidence lower bound may depend on quantities besides the first and second moment of outputs of the variational posterior, so the above argument breaks down. However, so long as we can derive upper bounds on the KL divergence between any optimal posterior and the prior that are independent of the width of the network, we can use our earlier results to conclude that the predictive mean and variance of any variational posterior with a better ELBO than the prior converges to the corresponding values under the prior. In order to upper bound the KL divergence, we assume that the log likelihood is bounded above and has a quadratic lower bound.

\paragraph{Counterexample for non-odd activations}
\cref{app:counterexample} examines whether it was essential for the results that the activation function is odd, or whether these results may be extended to other common activation functions such as ReLU. We answer this question in the negative, by exhibiting a dataset and Gaussian likelihood such that the mean of the optimal posterior does not converge to the prior mean. The key distinction in this case is that, where in the odd case we had,  $\norm{\E_{\P}[\phi_L(\vz(\vx))]}_2 =0$, in the non-odd case we have, 
\begin{align}
    \norm{\E_{\P}[\phi(\vz_L(\vx))]}_2 = \Theta(\sqrt{M}). 
\end{align}
This means that our earlier proof technique cannot possibly work. We use this observation to construct a counterexample.

\section{Existence of an Optimal Mean-Field Solution}
\label{app:existence-mf-solution}

\begin{proposition} \label{prop:existence-mf-solution}
    Let $\mathcal{Q}$ be the family of factorized Gaussian distributions.
    Assume the following:
    \begin{enumerate}[label=(\roman*)]
        \item The activation $\phi\colon\R \to\R$ is Lipschitz continuous.
        \item The likelihood factorizes over data points: $\log \mathcal{L}(\vtheta) = \sum_{n=1}^N \log p(\vy_n \cond \vf(\vx_n))$ for some function $p$.
        \item The likelihood is continuous: for all $\vy$, the function $\vf \mapsto p(\vy \cond \vf)$ is continuous.
        \item The likelihood is upper bounded: there exists a $C \in \R$ such that, for all $\vy$ and $\vf$, $\log p(\vy \cond \vf) \leq C$.
        \item The likelihood admits a quadratic lower bound:
        for all $\vy$, the function $\vf \mapsto \log p(\vy \cond \vf)$ can be lower bounded by a quadratic function in $\vf$.
    \end{enumerate}
    Then $\argmax_{Q \in \mathcal{Q}} \operatorname{ELBO}(Q)$ is non-empty:
    an optimal mean-field solution $Q \in \argmax_{Q \in \mathcal{Q}} \operatorname{ELBO}(Q)$ exists.
\end{proposition}
\begin{proof}
Let $I$ denote the total number of parameters in the network.
Then the variational optimization problem can be phrased as 
\newcommand{\ELBO}{\operatorname{ELBO}}
\begin{align}
    \sup_{\Q \in \mathcal{Q}} \operatorname{ELBO}(\Q) \quad \text{where}
    \quad \ELBO \colon \mathcal{Q} \to \R,
    \quad \mathcal{Q} = \set{
        \Normal(\vmu, \diag(\vsigma^2)) : \vmu \in \R^I, \, \vsigma^2 \in (0, \infty)^I
    }.
\end{align}
Call a sequence $(\Normal(\vmu_k, \diag(\vsigma_k^2)))_{k \ge 1} \sub \mathcal{Q}$ \emph{parameter convergent} [to $\Normal(\vmu, \diag(\vsigma^2)) \in \mathcal{Q}$] if $(\vmu_k, \vsigma_i^2)_{k \ge 1} \sub \R^I \times (0, \infty)^I$ is convergent [to $(\vmu, \vsigma) \in \R^I \times (0, \infty)^I$].
Using the definition of the supremum, extract a sequence $(Q_k)_{k \ge 1} \sub \mathcal{Q}$ such that $\ELBO(Q_k) \to \sup_{\Q \in \mathcal{Q}} \operatorname{ELBO}(\Q)$.
The argument now consists of two parts.
First, we show that there exists a subsequence $(Q_{n_k})_{k \ge 1} \sub (Q_k)_{k \ge 1}$ which is parameter convergent to some limit $Q^* \in \mathcal{Q}$ (compactness).
Second, we show that $\ELBO$ is upper semi-continuous with respect to parameter convergence (continuity).
Assuming the two parts,
\begin{align}
    \sup_{\Q \in \mathcal{Q}} \operatorname{ELBO}(\Q)
    = \lim_{k \to \infty} \operatorname{ELBO}(\Q_k)
    = \limsup_{k \to \infty}\, \operatorname{ELBO}(\Q_{n_k})
    \overset{\text{(i)}}{\le}  \operatorname{ELBO}(\Q^*),
\end{align}
where we use in (i) that $\ELBO$ is upper semi-continuous with respect to parameter convergence. 
Therefore, $\Q^* \in \argmax_{Q \in \mathcal{Q}} \operatorname{ELBO}(Q)$, which concludes the proof.

\paragraph{Compactness.}
We show that there exists a subsequence $(Q_{n_k})_{k \ge 1} \sub (Q_k)_{k \ge 1}$ which is parameter convergent to some limit $Q^* \in \mathcal{Q}$.
Let $P = \Normal(\vnull, \mI) \in \mathcal{Q}$ be the prior.
Assume that $\ELBO(P) < \sup_{\Q \in \mathcal{Q}} \operatorname{ELBO}(\Q)$; for if equality holds, an optimal mean-field solution certainly exists.
Since $\ELBO(Q_k) \to \sup_{\Q \in \mathcal{Q}} \operatorname{ELBO}(\Q)$, it follows that, for large enough $k$, $\ELBO(P) < \operatorname{ELBO}(\Q_k)$.
Therefore, by step 2 from \cref{sec:proof-sketch}, it follows that there exists a $C > 0$ such that, for large enough $k$, $\KL(Q_k, P) < C$.
Consequently, denoting $Q_k = \Normal(\vmu_k, \vsigma^2_k)$, by \cref{app:parameter-bounds} and the observation that $r(a) < c$ for $c \ge 0$ implies that $a \in [R^{-1}, R]$ for some $R \in [1, \infty)$, it follows that there exists an $R \in [1, \infty)$ such that, for large enough $k$ and all $i \in [I]$, $\abs{\mu_{k,i}} \le R$ and $R^{-1} \le \abs{\sigma^2_{k,i}} \le R$.
Hence, by Bolzano--Weierstrass, there exists a subsequence $(Q_{n_k})_{k \ge 1} \sub (Q_k)_{k \ge 1}$ which is parameter convergent to some limit $Q^* \in \mathcal{Q}$.



\paragraph{Continuity.}
Let $(Q_k)_{k \ge 1} \sub \mathcal{Q}$ be parameter convergent to some $Q \in \mathcal{Q}$.
To conclude the proof, we show that
$
    \limsup_{k \to \infty}\, \ELBO(Q_k) \le \ELBO(Q).
$
Decompose the ELBO as follows:
\begin{equation} \textstyle
    \operatorname{ELBO}(Q_k)
    = \sum_{n=1}^{N}\E_{Q_k}[\log p(\vy_n \cond \vf(\vx_n))] - \KL(Q_k, P).
\end{equation}
From \cref{app:parameter-bounds}, $Q \mapsto {-\KL}(Q, P)$ is clearly continuous with respect to parameter convergence, so it is also upper semi-continuous with respect to parameter convergence.
Hence, it remains to show that
\begin{equation}
    \limsup_{k \to \infty}\,\E_{Q_k}[\log p(\vy_n \cond \vf(\vx_n))]
    \le \E_{Q}[\log p(\vy_n \cond \vf(\vx_n))].
\end{equation}
To show this, we use the reparametrization trick:
rewrite $\E_{Q_k}[\log p(\vy_n \cond \vf(\vx_n))] = \E[\log p(\vy_n \cond \vf^{Q_k}(\vx_n))]$ where
\begin{align}
    \vf^{Q_k}(\vx) &= \tfrac1{\sqrt{D\ss{o}}}(\mS^{Q_k}_{L+1} \circ \mathbfcal{E}_{L+1} + \mM^{Q_k}_{L+1}) \phi(\vz^{Q_k}_{L}) + (\vs^{Q_k}_{L+1} \circ \vep_{L+1} + \vm^{Q_k}_{L+1}), \\
    \vz^{Q_k}_\ell
    &= \tfrac1{\sqrt{M}}(\mS^{Q_k}_\ell \circ \mathbfcal{E}_\ell + \mM^{Q_k}_\ell) \phi(\vz^{Q_k}_{\ell-1}) + (\vs^{Q_k}_\ell \circ \vep_\ell + \vm^{Q_k}_\ell),
    \qquad\qquad \ell = L, \ldots, 2, \\
    \vz^{Q_k}_1
    &= \tfrac1{\sqrt{D\ss{i}}}(\mS^{Q_k}_1 \circ \mathbfcal{E}_1 + \mM^{Q_k}_1) \vx + (\vs^{Q_k}_1 \circ \vep_1 + \vm^{Q_k}_1).
\end{align}
where $(\mathbfcal{E}_\ell)_{\ell=1}^{L+1}$ are matrices of i.i.d.~standard Gaussian random variables, $(\vep_\ell)_{\ell=1}^{L+1}$ are vectors of i.i.d.~standard Gaussian random variables,
$(\mM^{Q_k}_\ell)_{\ell=1}^{L+1}$ are matrices consisting of the means of each weight in each layer under $\Q_k$,
$(\vm^{Q_k}_\ell)_{\ell=1}^{L+1}$ are vectors consisting of the means of each bias in each layer under $\Q_k$, 
$(\mS^{Q_k}_\ell)_{\ell=1}^{L+1}$ are matrices consisting of the standard deviations of each weight in each layer under $\Q_k$, and
$(\vs^{Q_k}_\ell)_{\ell=1}^{L+1}$ are vectors consisting of the standard deviations of each bias in each layer under $\Q_k$.
Since $\phi$ is Lipschitz, it is continuous, so clearly $\vf^{Q_k}(\vx) \to \vf^{Q}(\vx)$.
Let $C$ be the upper bound on the likelihood.
Then
\begin{align*}
    \limsup_{k \to \infty}\,\E_{Q_k}[\log p(\vy_n \cond \vf(\vx_n))]
    &= C + \limsup_{k \to \infty}\,\E[-C + \log p(\vy_n \cond \vf(\vx_n))] \\
    &= C - \liminf_{k \to \infty}\,\E[C - \log p(\vy_n \cond \vf^{Q_k}(\vx_n))] \\
    &\overset{\smash{\text{(i)}}}{\le} C - \E[\liminf_{k \to \infty}\,(C - \log p(\vy_n \cond \vf^{Q_k}(\vx_n)))] \\
    &\overset{\smash{\text{(ii)}}}{=} \vphantom{\liminf_{k \to \infty}} C - \E[C - \log p(\vy_n \cond \vf^{Q}(\vx_n))] \\
    &=\vphantom{\liminf_{k \to \infty}} \E[\log p(\vy_n \cond \vf^{Q}(\vx_n))] \\
    &= \E_{Q}[\log p(\vy_n \cond \vf(\vx_n))],
\end{align*}
where in (i) we use Fatou's lemma in combination with that $C - \log p(\vy_n \cond \vf^{Q_k}(\vx_n)) \ge 0$ by definition of $C$ and in (ii) we use (ii.a) continuity of $\vf \mapsto p(\vy_n \cond \vf)$ and (ii.b) $\vf^{Q_k}(\vx_n) \to \vf^{Q}(\vx_n)$.
%
%
%
\end{proof}
\section{Bounds on Parameters in Terms of the Kullback--Leibler Divergence}\label{app:parameter-bounds}
Fundamentally, if $\KL(\Q,\P)$ is small, we know that $\Q$ and $\P$ are `close' in some sense. We want to translate this notion of `close' to a notion directly related to the moments of the predictive distributions implied by the networks. In order to do this, we desire statements about how close the parameters of $\Q$ and $\P$ are, according to some norm. In this section, we show how to upper bound various norms of the parameters of $\Q$ in terms of $\KL(\Q,\P)$. These bounds will be a key ingredient in proofs of \cref{thm:app-convergence-of-mean,thm:app-convergence-of-variance}.

\begin{lemma}[Kullback-Leibler Divergence between diagonal multivariate Gaussian distributions]\label{lem:gaussian-kl}
If $P=\Normal(\vnull,\mI)$ and $Q = \Normal(\vmu_Q, \mathrm{diag}(\vsigma^2_Q))$ It holds that
    \begin{equation}
        \KL(\Q, \P)
        =
            \tfrac{1}{2}
            (\norm{\vmu_Q}_2^2 + \norm{r(\vsigma^2_Q)}_1)
    \end{equation}
    where $r\colon (0, \infty) \to [0, \infty)$, $r(a) = a - 1 - \log(a)$ is applied element-wise.
\end{lemma}
We note that $r$ is a convex function, with a minimum at $r(1)=0$. We now turn to proving bounds on the parameters of $\Q$, which amounts to various methods of rearranging \cref{lem:gaussian-kl}.

\begin{lemma}[Bounds on parameters in term of KL divergence] \label{lem:parameter_bound}
    The following inequalities are true:
    \begin{enumerate}
        \item[(i)]
            $\norm{\vmu_\Q}^2_2 \le 2 \KL(\Q, \P)$,
        \item[(ii)]
            $\norm{\vsigma_\Q - \vone}_2^2 \le 2 \KL(\Q, \P)$,
        \item[(iii)]
            $\sigma\ss{max}
             \le 1 + \sqrt{2 \KL(\Q, \P)}
             \le 2 \sqrt{2\KL(\Q, \P) \lor 1}$, and
        \item[(iv)]
            $\norm{\vsigma_\Q^2 - \vone}_2^2
             \le (\sigma\ss{max} + 1)^2 \norm{\vsigma_\Q - 1}_2^2
             \le (2 + \sqrt{2 \KL(\Q, \P)})^2 (2 \KL(\Q, \P))$.
    \end{enumerate}
\end{lemma}
\begin{proof}
We prove each bound in turn.
\paragraph{(i):}
    This follows directly from \cref{lem:gaussian-kl} and non-negativity of norms.
    
\paragraph{(ii):}
    To prove (ii), we recall the identity $\log (a') \leq a'-1$ for all $a'>0.$ Define $a'=\sqrt{a}$. Then $\log(a) = 2(\sqrt{a} - 1)$. 
Rearranging, $\log(a) \le 2 \sqrt{a} - 2$ for all $a > 0$. Then compute
    \begin{equation}
        r(a) = a - 1 - \log(a)
        \ge a - 1 - (2 \sqrt{a} - 2)
        =  a - 2 \sqrt{a} + 1
        = (\sqrt{a} - 1)^2.
    \end{equation}
    
\paragraph{(iii):}    
    For (iii), estimate
    \begin{equation}
        \sigma\ss{max} - 1
        \le \abs{\sigma\ss{max} - 1}
        \le \norm{\vsigma_\Q - \vone}_2
        \le \sqrt{2 \KL(\Q, \P)},
    \end{equation}
where the last inequality uses (ii).

\paragraph{(iv):}     
    For (iv), factoring the difference of two squares, 
	%
	\begin{align}
		\norm{\vsigma_\Q^2 - \vone}_2^2
		&\textstyle= \sum_{i=1}^I (\sigma^2_i -1)^2 
		\\
		&\textstyle =  \sum_{i=1}^I (\sigma_i +1)^2 (\sigma_i -1)^2 \\
 		&\textstyle\le \sum_{i=1}^I (\sigma\ss{max} + 1)^2 (\sigma_i-1)^2
		\\
		&\textstyle= (\sigma\ss{max} + 1)^2 \sum_{i=1}^I (\sigma_i-1)^2
		\\
		&\textstyle= (\sigma\ss{max} + 1)^2 \norm{\vsigma_\Q - \vone}_2^2
		\\
		&\le (2 + \sqrt{2 \KL(\Q, \P)})^2 (2 \KL(\Q, \P))
	\end{align}
where in the final inequality we have used (ii) and (iii).
\end{proof}

\begin{proposition} \label{prop:optimised_param_bound}
    Let $\mW$ be an arbitrary weight matrix and let $\vb$ be an arbitrary bias vector.
    Then
    \begin{equation}
        \norm{\V\ss{d}[\operatorname{vec}(\mW)]}_\infty
        + \norm{\V\ss{d}[\vb]}_\infty
        + \norm{\E[\vb^2]}_\infty
        \le (\sqrt{2} + \sqrt{2\KL(Q, P)})^2.
    \end{equation}
\end{proposition}
\begin{proof}
    Denote $K = \sqrt{2\KL(Q, P)}$.
    By \cref{lem:gaussian-kl}, we have the constraint
    \begin{equation}
        \norm{r(\V\ss{d}[\operatorname{vec}(\mW)])}_1
        + \norm{\E[\vb]}_2^2
        + \norm{r(\V\ss{d}[\vb])}_1
        \le K^2.
    \end{equation}
    We argue that this constraint implies that $
        \norm{\V\ss{d}[\operatorname{vec}(\mW)]}_\infty
        + \norm{\E[\vb^2]}_\infty
    \le (\sqrt{2} + K)^2$.
    To argue this, consider optimising $\norm{\vsigma^2_1}_\infty + \norm{\vmu_2^2}_\infty + \norm{\vsigma_2^2}_\infty$ over $(\vsigma_1^2, \vmu_2, \vsigma_2^2)$ such that $\norm{r(\vsigma_1^2)}_1 + \norm{\vmu_2}_2^2 + \norm{r(\vsigma_2^2)}_1 \le K^2$.
    Without loss of generality, assume that $\vsigma_1^2 \succeq 1$ and $\vsigma_2^2 \succeq 1$.
    Then, without loss of generality, by the observation that the objective comprises $\infty$-norms, for all $i \ge 2$, assume that $\sigma_{1,i}^2=\sigma_{2,i}^2=1$ and $\mu_{i} = 0$.
    Since, on $[1, \infty)$, $r'(a) = 1 - \smash{\tfrac1a} < 1$ and $r'$ is strictly increasing, it is clear that, at the maximum, $\sigma_{1,1}^2 = \sigma_{2,1}^2 = \sigma^2$ and $\mu_1 = 0$.
    From \cref{lem:parameter_bound} and the constraint, we have that $2 (\sigma - 1)^2 \le 2 r(\sigma^2) \le K^2$, so $\sigma \le 1 + \smash{\tfrac1{\sqrt{2}}} K$.
    Therefore, at the maximum, $\norm{\vsigma^2_1}_\infty + \norm{\vmu_2^2}_\infty + \norm{\vsigma_2^2}_\infty = 2 \sigma^2 \le 2 (1 + \smash{\tfrac1{\sqrt{2}}} K)^2 = (\sqrt{2} + K)^2$.
\end{proof}

\section{Proof of Convergence of the Mean of the Variational Posterior for Odd Activation Functions}\label{app:mean-convergence}

The main result in this section we prove will be the following,
\begin{theorem}[Convergence of mean prediction] \label{thm:app-convergence-of-mean}
    Let $Q$ be a mean-field variational posterior and $P=\Normal(\vnull,\mI)$ denote the prior over a neural network with $L$ hidden layers and $M$ neurons per hidden layer.
    Suppose $\phi\ss{e}=\alpha$ for some $\alpha \in \R$ and $\phi\colon \R \to \R$ is $1$-Lipschitz.
    Let $\vx \in \R^{D\ss{i}}$ and let $\widetilde{\vf}_{\theta}(\vx) = \vf_\theta(\vx) - \tfrac\alpha{\sqrt{M}}\mW_{L+1}\vone - \vb_{L+1}$ denote the network output excluding final bias and even part of the final activation.
    Then there exist universal constants $c_1 \leq 4$ and $c_2 \leq 6$ such that 
    \begin{equation}
         \norm{\E_Q[\widetilde{\vf}_{\theta}(\vx)]}_2
        \le c_1c_2^{L-1} L \frac{\abs{\alpha} + 1+ \|\vx\|_2/\sqrt{D\ss{i}}}{\sqrt{M}} \KL(Q, P) \parens{(2\KL(\Q, \P))^\frac{L-1}{2} \lor 1}.
    \end{equation}
\end{theorem}
For the proof we assume the Lipschitz constant of the activation function is 1, but we note that any Lipschitz function can be scaled to have a Lipschitz constant of 1. 

\begin{corollary}\label{cor:difference-of-means}
    With the same notation and assumptions as in \cref{thm:app-convergence-of-mean}, for $\vx, \vx' \in \R^{D\ss{i}}$ we have 
    \begin{align}
        \norm{\E_Q[\vf_{\theta}(\vx)] -\E_Q[\vf_{\theta}(\vx')]}_2 \leq c_1c_2^{L-1} L \frac{2\abs{\alpha} + 2+( \|\vx\|_2+\|\vx'\|_2)/\sqrt{D\ss{i}}}{\sqrt{M}} \KL(Q, P) \parens{(2\KL(\Q, \P))^\frac{L-1}{2} \lor 1}
    \end{align}
\end{corollary}
\Cref{cor:difference-of-means} follows from \cref{thm:app-convergence-of-mean} by noting that $\E_Q[\vf_{\theta}(\vx)] -\E_Q[\vf_{\theta}(\vx')] = \E_Q[\widetilde{\vf}_{\theta}(\vx)] -\E_Q[\widetilde{\vf}_{\theta}(\vx')]$, then applying triangle inequality.

 
The crucial technical result for the proof of \cref{thm:app-convergence-of-mean} will be the following lemma, which upper bounds the norm of the expected value of the final layer of hidden units. We defer the proof of this lemma, which essentially inducts on the number of hidden layers, to \cref{app:main-recursion}.
\begin{lemma}\label{lem:last-layer-bound}
 Suppose $\Q$ is mean-field Gaussian, $\P=\Normal(\vnull, \mI)$, $\phi_{e}=\alpha$, with $\alpha \in \R$ and $\phi$ is $1$-Lipschitz. Then,
\begin{align}
    \norm{\E[\phi\ss{o}(\vz_L(\vx))]}_2
    &\le 2L(2 + \abs{\alpha} + \tfrac{1}{\sqrt{D\ss{i}}}\norm{\vx}_2)(2 + 2c)^{L - 1}\sqrt{2\KL(Q,P)}(\sqrt{2\KL(Q,P)} \lor 1)^{L - 1},
\end{align}
where $c>0$ is a universal constant.
\end{lemma}

We now turn to the proof of \cref{thm:app-convergence-of-mean}, which is relatively direct once \cref{lem:last-layer-bound} has been established.
\begin{proof}[Proof of \cref{thm:app-convergence-of-mean}]
Let every expectation be under $\Q$. Define $K=\sqrt{2\KL(Q,P)}$. By a slight abuse of notation, let $\vz_L = \vz_L(\vx)$.

To begin with, note that
\begin{align}
    \norm{\E[\widetilde{\vf}_{\theta}(\vx)]}_2 
    &= \tfrac1{\sqrt{M}} \norm{\E[\mW_{L+1}]\E[(\phi\ss{o}(\vz_L)]}_2.
\end{align}
Using \cref{lem:parameter_bound}, $\norm{\E[\mW_{L+1}]}\ss{F} \le K$. We then apply \cref{lem:last-layer-bound},
\begin{align}
    \norm{\E[\widetilde{\vf}_{\theta}(\vx)]}_2 
    &\le \frac{K}{\sqrt{M}} 
        2L\parens*{2 + \abs{\alpha}+\frac{1}{\sqrt{D\ss{i}}}\|\vx\|_2}(2 + 2c)^{L - 1} K(K \lor 1)^{L - 1} \\
    &= \frac{2L}{\sqrt{M}} \parens*{2 + \abs{\alpha}+\frac{1}{\sqrt{D\ss{i}}}\|\vx\|_2} (2 + 2c)^{L-1}\KL(\Q, \P)\parens{(2\KL(\Q, \P))^\frac{L-1}{2} \lor 1}.
\end{align}
\end{proof}

\begin{figure}
	\begin{centering}
		\begin{tikzpicture}
			
			\node[state] (main) at (0,0) {Theorem 23};
			
			\node[state] (lem13) [above  = of main, yshift=.75cm] {Lemma 13};
			
			\path (lem13) edge[] (main);
			
			\node[state] (lem20) [above  = of lem13, yshift=.75cm] {Proposition 20};
			
			\node[state] (lem21) [left  = of lem20, xshift=.75cm] {Lemma 21};
			\node[state] (lem22) [left  = of lem21, xshift=.75cm] {Lemma 22};
			
			\node[state] (lem19) [right  = of lem20, xshift=-.75cm] {Lemma 19};
			\node[state] (lem18) [right  = of lem19, xshift=-.75cm] {Lemma 18};
			
			\path (lem18) edge[bend left=5] (lem13);
			\path (lem21) edge[bend right=5] (lem13);
			\path (lem22) edge[bend right=5] (lem13);
			
			\path (lem18) edge[bend right=25] (lem22);
			\path (lem19) edge[bend right=35] (lem21);
			\path (lem20) edge[bend right=15] (lem21);
			
			\path (lem21) edge[bend right=15] (lem22);
			
			
			\node[state] (lem15) [above left  = of lem20, xshift=1.45cm, yshift=.75cm] {Lemma 15};
			\node[state] (lem16) [above right  = of lem20, xshift=-1.45cm, yshift=.75cm] {Remark 16};
			
			\node[state] (lem14) [left  = of lem15, xshift=.75cm] {Lemma 14};
			\node[state] (lem17) [right  = of lem16, xshift=-.75cm] {Lemma 17};

			\path (lem14) edge[bend left=45] (lem15);
			\path (lem15) edge[bend left=45] (lem16);
			
			\path (lem16) edge[bend left=5] (lem18);
			\path (lem17) edge[bend left=5] (lem18);
			\path (lem17) edge[bend left=5] (lem19);
			
		\end{tikzpicture}
		\caption{Dependency structure of the results in \cref{app:mean-convergence}.}
		\label{fig:dag-mean}
	\end{centering}
\end{figure}
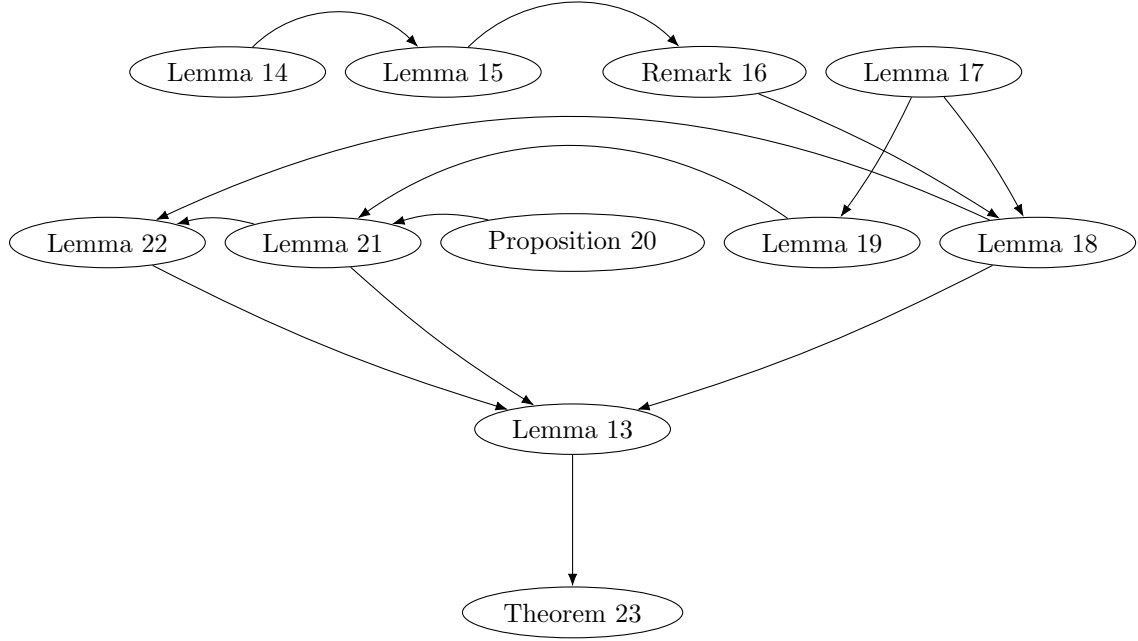

\subsection{Main Recursion: Proof of \texorpdfstring{\cref{lem:last-layer-bound}}{TEXT}}\label{app:main-recursion}

The main purpose of this section will be the proof of \cref{lem:last-layer-bound}. We begin by proving several results that build up to bounds on the norm of the expected value and the expected value of the norm of one layer in terms of the previous layer (\cref{lem:bound,lem:bound_norm}). Once we have established these bounds, the proof of \Cref{lem:last-layer-bound} follows by recursive application of these bounds. See \cref{fig:dag-mean} for a diagram of the dependencies between the results. 

For a matrix $\mW$, recall that $\norm{\mW}_2 \le \norm{\mW}\ss{F}$.
Call a random variable $a$ \emph{symmetric around its mean} if $a - \E[a] \disteq -(a - \E[a])$

\begin{lemma} \label{lem:lipschitz}
    Let $\phi_1, \phi_2 \colon \R \to \R$ be $1$-Lipschitz and odd.
    Then
    \begin{align}
        &\norm{\phi_1(\mW \phi_2(\vz) + \vb) - \phi_1(\mW' \phi_2(\vz') + \vb')}_2 \nonumber \\
        &\qquad\le
            \norm{\vz}_2 \norm{\mW - \mW'}\ss{F}
            + \norm{\mW'}_2 \norm{\vz - \vz'}_2
            + \norm{\vb - \vb'}_2.
    \end{align}
\end{lemma}
\begin{proof}
    When applied element-wise, $\phi_1$ and $\phi_2$ are also $1$-Lipschitz as functions $(\R^n, \norm{\vardot}_2) \to (\R^n, \norm{\vardot}_2)$.
    The result then follows from an application of $1$-Lipschitzness and the triangle inequality:
    \begin{align}
        &\norm{\phi_1(\mW \phi_2(\vz) + \vb) - \phi_1(\mW' \phi_2(\vz') + \vb')}_2 \nonumber \\
        &\qquad\le \norm{\mW \phi_2(\vz) + \vb - (\mW' \phi_2(\vz') + \vb')}_2 \\
        &\qquad= \norm{(\mW - \mW')\phi_2(\vz) + \mW'(\phi_2(\vz) - \phi_2(\vz')) + (\vb - \vb')}_2 \\
        &\qquad\le  \norm{\mW - \mW'}_2\norm{\phi_2(\vz)}_2 + \norm{\mW'}_2 \norm{\phi_2(\vz) - \phi_2(\vz')}_2 + \norm{\vb - \vb'}_2 \\
        &\qquad\le \norm{\mW - \mW'}\ss{F}\norm{\vz}_2  + \norm{\mW'}_2 \norm{\vz - \vz'}_2 + \norm{\vb - \vb'}_2
    \end{align}
    where in the last inequality we use that since $\phi_2$ is odd $\phi_2(\vnull) = \vnull$ so $\norm{\phi_2(\vz)}_2 = \norm{\phi_2(\vz) - \phi_2(\vnull)}_2 \le \norm{\vz}_2$.
\end{proof} 

\begin{lemma} \label{lem:contraction}
    Let $\phi_1,\phi_2\colon \R \to \R$ be $1$-Lipschitz and odd.
    Let the triple $(\mW, \vz, \vb)$ be  (possibly dependent) random variables such that
    \begin{equation}
        (\mW - \E[\mW], \vz, \vb - \E[\vb])
        \disteq
        (-(\mW - \E[\mW]), \vz, -(\vb - \E[\vb]))
    \end{equation}
    Then
    \begin{align}
        &\norm{\E[\phi_1(\mW \phi_2(\vz) + \vb)]}_2 
    \le \norm{\E[\mW]}\ss{F}\E[\norm{\vz}_2] + \E[\norm{\mW - \E[\mW]}_2] \norm{\E[\vz]}_2 + \norm{\E[\vb]}_2
    \end{align}
\end{lemma}
\begin{proof}
    Consider $\mW' = \mW - \E[\mW]$, $\vz' = \vz - \E[\vz]$, and $\vb' = \vb - \E[\vb]$.
    By assumption, $(\mW', \vz', \vb') \disteq (-\mW', \vz', -\vb')$.
    Therefore, using that $\phi_1$ is odd
    \begin{equation}
        \E[\phi_1(\mW' \phi_2(\vz') + \vb')]
        = \E[\phi_1(-\mW' \phi_2(\vz')-\vb')]
        = -\E[\phi_1(\mW' \phi_2(\vz')+\vb')],
    \end{equation}
    which means that $\E[\phi(\mW' \phi(\vz')+\vb')] = \vnull$.
    We now apply \cref{lem:lipschitz}:
    \begin{align}
        &\norm{\E[\phi_1(\mW \phi_2(\vz) + \vb)]}_2 \nonumber \\
        &\qquad = \norm{\E[\phi_1(\mW \phi_2(\vz) + \vb) - \phi_1(\mW' \phi_2(\vz') + \vb')]}_2 \\
        &\qquad \le \E[\norm{\phi_1(\mW \phi_2(\vz) + \vb) - \phi_1(\mW' \phi_2(\vz') + \vb')}_2] \\
        &\qquad\le \E[ \norm{\mW - \mW'}\ss{F}\norm{\vz}_2] + \E[\norm{\mW'}_2 \norm{\vz - \vz'}_2] + \E[\norm{\vb - \vb'}_2] \\
        &\qquad\le  \norm{\E[\mW]}\ss{F}\E[\norm{\vz}_2] + \E[\norm{\mW - \E[\mW]}_2] \norm{\E[\vz]}_2 + \norm{\E[\vb]}_2,
    \end{align}
    which proves the result.
\end{proof}

\begin{remark} \label{rem:symmetry_offset}
    The symmetry condition for $(\mW, \vz, \vb)$ is satisfied if $\mW$, $\vz$, and $\vb$ are independent and $\mW$ and $\vb$, not $\vz$, are symmetric around their means.
    For such $(\mW, \vz, \vb)$, the symmetry condition is also satisfied for the triple $(\mW, \vz, \alpha \mW \vone + \vb)$
    where $\vone$ is the vector of all ones and $\alpha \in \R$.
    In that case, a similar conclusion holds:
    \begin{align}
        &\norm{\E[\phi_1(\mW (\phi_2 + \alpha)(\vz) + \vb)]}_2 \\
        &\quad= \norm{\E[\phi_1(\mW \phi_2(\vz) + \alpha \mW\vone + \vb)]}_2 \\
        &\quad\le  \norm{\E[\mW]}\ss{F}\E[\norm{\vz}_2] + \E[\norm{\mW - \E[\mW]}_2] \norm{\E[\vz]}_2 + \norm{\E[\alpha \mW\vone + \vb]}_2 \\
        &\quad\le \norm{\E[\mW]}\ss{F}\E[\norm{\vz}_2] + \E[\norm{\mW - \E[\mW]}_2] \norm{\E[\vz]}_2 + \abs{\alpha} \norm{\E[\mW]}_2 \norm{\vone}_2 + \norm{\E[\vb]}_2 \\
        &\quad= (\E[\norm{\vz}_2 + \abs{\alpha}\sqrt{M})\norm{\E[\mW]}\ss{F} + \E[\norm{\mW - \E[\mW]}_2] \norm{\E[\vz]}_2 + \norm{\E[\vb]}_2.
    \end{align}
\end{remark}

\begin{lemma} \label{lem:expectation_bound}
    Let $\mA \in \R^{I \times J}$ be a zero-mean random matrix with independent Gaussian entries with variances bounded by $\sigma^2$.
    Then
    \begin{equation}
        \E[\norm{\mA}_2] \le 2 \sigma \sqrt{I \lor J}.
    \end{equation}
    Moreover, if the variances of the entries of $\mA$ are all equal to one, then
    \begin{equation}
        \mathbb{P}(\norm{\mA}_2 \ge 2 \sqrt{I \lor J} + \delta) \le 2 \exp(-\tfrac1{2} \delta^2).
    \end{equation}
\end{lemma}
\begin{proof}
    We slightly generalise Exercise 5.14 from \citet{Wainwright:2019:High-Dimensional_Statistics_A_Non-Asymptotic_Viewpoint}.
    To begin with, rewrite the operator norm as
    \begin{equation}
        \norm{\mA}_2
        = \sup_{(\vu, \vv) \in \mathbb{S}^{I - 1} \times \mathbb{S}^{J - 1}} \lra{\vu, \mA \vv}.
    \end{equation}
    Define the zero-mean Gaussian process $Z_{\vu, \vv} = \lra{\vu, \mA \vv}$ indexed on $\mathbb{S}^{I - 1} \times \mathbb{S}^{J - 1}$ and
    define $\mS$ by $S_{i,j} = \V[A_{i,j}] \le \sigma^2$.
    Note that, by independence of the entries of $\mA$,
    \begin{align}
        \E[(Z_{\vu,\vv} - Z_{\vw, \vx})^2]
        &= \E[Z^2_{\vu, \vv} - 2 Z_{\vu,\vv} Z_{\vw, \vx} + Z^2_{\vw, \vx}] \\
        &= \textstyle
            \sum_{i=1,j=1}^{I,J} S_{ij} u_i^2 v_j^2
            - 2 \sum_{i=1,j=1}^{I,J} S_{ij} u_i v_j w_i x_j
            + \sum_{i=1,j=1}^{I,J} S_{ij} w_i^2 x_j^2 \\
        &= \textstyle
            \sum_{i=1,j=1}^{I,J} S_{ij} (u_i v_j - w_i x_j)^2 \\
        &\le \textstyle
            \sigma^2 \sum_{i=1,j=1}^{I,J} (u_i v_j - w_i x_j)^2 \\
        &=
            \sigma^2 \norm{\vu \vv^\T - \vw \vx^\T}\ss{F}^2.
    \end{align}
    Also consider the zero-mean Gaussian process $Y_{\vu, \vv} = \sigma \lra{\vu, \vep_1} + \sigma  \lra{\vv, \vep_2}$ again indexed on $\mathbb{S}^{I - 1} \times \mathbb{S}^{J - 1}$, where $\vep_1 \in \R^I$ and $\vep_2 \in \R^J$ are standard Gaussian vectors.
    Then
    \begin{align}
        \E[(Y_{\vu, \vv} - Y_{\vw, \vx})^2]
        &= \sigma^2 \E[(\lra{\vu - \vw, \vep_1} +  \lra{\vv - \vx, \vep_2})^2] \\
        &= \sigma^2 \E[\lra{\vu - \vw, \vep_1}^2 + \lra{\vv - \vx, \vep_2}^2] \\
        &= \sigma^2 (\norm{\vu - \vw}_2^2 + \norm{\vv - \vx}_2^2).
    \end{align}
    Using that $\norm{\vu}_2 = \norm{\vv}_2 =\norm{\vx}_2 =\norm{\vw}_2 = 1$, careful algebra \citep[see, e.g., page 164 from][]{Wainwright:2019:High-Dimensional_Statistics_A_Non-Asymptotic_Viewpoint} shows that $\norm{\vu \vv^\T - \vw \vx^\T}\ss{F}^2 \le \norm{\vu - \vw}_2^2 + \norm{\vv - \vx}_2^2$:
    \begin{align}
        (u_i v_j - w_i x_j)^2
        &= (u_i v_j - w_i v_j + w_i v_j - w_i x_j)^2 \\
        &= (u_i - w_i)^2 v_j^2 + w_i^2 (v_j - x_j)^2
            + 2 (u_i v_j - w_i v_j) (w_i v_j - w_i x_j) \\
        &= (u_i - w_i)^2 v_j^2 + w_i^2 (v_j - x_j)^2
            + 2 (u_i w_i - w_i^2) (v^2_j - v_j x_j).
    \end{align}
    Therefore, summing over $i \in [I]$ and $j \in [J]$ and using that $\norm{\vv}_2 = \norm{\vw}_2 = 1$,
    \begin{equation}
        \norm{\vu \vv^\T - \vw \vx^\T}^2\ss{F}
        = \norm{\vu - \vw}_2^2 + \norm{\vv - \vx}_2^2 + 2(\lra{\vu, \vw} - 1)(1 - \lra{\vv, \vx})
        \le \norm{\vu - \vw}_2^2 + \norm{\vv - \vx}_2^2
    \end{equation}
    where the inequality follows from additionally using that $\norm{\vu}_2 = \norm{\vx}_2 = 1$.
    Using this result, we find
    \begin{equation}
        \E[(Z_{\vu,\vv} - Z_{\vw, \vx})^2] \le \E[(Y_{\vu, \vv} - Y_{\vw, \vx})^2].
    \end{equation}
    Hence, denoting
    $Z^* = \sup_{(\vu, \vv) \in \mathbb{S}^{I - 1} \times \mathbb{S}^{J - 1}} Z_{\vu,\vv}$
    and $Y^* = \sup_{(\vu, \vv) \in \mathbb{S}^{I - 1} \times \mathbb{S}^{J - 1}} Y_{\vu,\vv}$, by the Sudakov--Fernique comparison theorem \citep[Theorem 5.27,][]{Wainwright:2019:High-Dimensional_Statistics_A_Non-Asymptotic_Viewpoint},
    we conclude that
    \begin{equation}
        \E[\norm{\mA}_2] = \E[Z^*]
        \le \E[Y^*]
        \overset{\smash{\text{(i)}}}{=} \sigma\,\E[\norm{\vep_1}_2 + \norm{\vep_2}_2]
        \overset{\smash{\text{(ii)}}}{\le} \sigma(\sqrt{I} + \sqrt{J})
    \end{equation}
    where (i) follows from that the suprema are achieved by $\vep_1/\norm{\vep_1}_2$ and $\vep_2/\norm{\vep_2}_2$ for $\vu$ and $\vv$, respectively, and  (ii) follows from Jensen's inequality applied to the square root. 
    
    For the second statement, assume that all entries of $\mA$ have variance one, so $\sigma^2=1$.
    We apply concentration of Gaussian suprema \citep[e.g., Exercise 5.10 by][]{Wainwright:2019:High-Dimensional_Statistics_A_Non-Asymptotic_Viewpoint} to $Z_{\vu,\vv}$:
    \begin{equation}
        \mathbb{P}(
            \abs{
                Z^* - \E[Z^*]
            } \ge \delta
        ) \le 2 \exp(-\tfrac1{2 v} \delta^2)
    \end{equation}
    where
    \begin{equation}
        v = \sup_{(\vu, \vv) \in \mathbb{S}^{I - 1} \times \mathbb{S}^{J - 1}} \V[Z_{\vu,\vv}] = \sup_{(\vu, \vv) \in \mathbb{S}^{I - 1} \times \mathbb{S}^{J - 1}}\norm{\vu}_2^2 \norm{\vv}_2^2 = 1.
    \end{equation}
    Therefore, using the bound on $\E[Z^*]$ from the first statement,
    \begin{equation}
        \mathbb{P}(\norm{\mA}_2 \ge 2 \sqrt{I \lor J} + \delta) \le 2 \exp(-\tfrac1{2} \delta^2),
    \end{equation}
    which concludes the proof.
\end{proof}


Having established the preliminaries, we can now bound the norm of the expectation of one layer in terms of the previous layer.
\begin{lemma} \label{lem:bound}
    Let $\phi_1, \phi_2 \colon \R \to \R$ be $1$-Lipschitz and odd, and let $\alpha \in \R$.
    Let $(\mW, \vz, \vb)$ be independent random variables with $\mW$ and $\vb$, not $\vz$, symmetric around their means.
    Moreover, assume that $\mW$ has independent sub-Gaussian entries with sub-Gaussian parameters bounded by $\sigma$.
    Then there exists a universal constant $c > 0$ such that
    \begin{align}
        &\norm{\E[\phi_1(\tfrac1{\sqrt{M}}\mW (\phi_2 + \alpha)(\vz) + \vb)]}_2 \nonumber \\
        &\qquad\le \parens{\tfrac1{\sqrt{M}}\E[\norm{\vz}_2] + \abs{\alpha}}\norm{\E[\mW]}\ss{F} + c \sigma \norm{\E[\vz]}_2 + \norm{\E[\vb]}_2.
    \end{align}
    with $c$ the same constant as in \cref{lem:expectation_bound}.
\end{lemma}
\begin{proof}
    By \cref{lem:expectation_bound}, there exists a universal constant $c > 0$ such that
    \begin{equation}
        \E[\norm{\mW - \E[\mW]}_2] \le c \sigma \sqrt{M}
        \implies
        \E[\norm{\tfrac{1}{\sqrt{M}}\mW - \E[\tfrac{1}{\sqrt{M}}\mW]}_2] 
        \le c \sigma.
    \end{equation}
    The result then follows from \cref{rem:symmetry_offset} with $\tfrac{1}{\sqrt{M}}\mW$ instead of $\mW$.
\end{proof}

We will also need upper bounds on the expected value of the norm of a hidden layer in terms of the previous layer. 
\begin{lemma} \label{lem:bound_norm}
    Assume the conditions of \cref{lem:bound}.
    Then 
    \begin{align}
        &\E[\norm{\tfrac1{\sqrt{M}}\mW (\phi_2 + \alpha)(\vz) + \vb}_2] \nonumber \\
        &\qquad\le (\tfrac1{\sqrt{M}}\E[\norm{\vz}_2] + \abs{\alpha}) (\norm{\E[\mW]}_2 + c\sigma \sqrt{M})  + \norm{\E[\vb]}_2+ c \sigma \sqrt{M}.
    \end{align}
\end{lemma}
\begin{proof}
    Use the triangle inequality and recall that $\norm{\phi_2(\vz)}_2 \le \norm{\vz}_2$ (proof of \cref{lem:lipschitz}):
    \begin{align}
        &\E[\norm{\tfrac1{\sqrt{M}}\mW \phi_2(\vz) + \tfrac{\alpha}{\sqrt{M}}\mW \vone + \vb}_2] \nonumber \\
        &\qquad\le
            (\tfrac1{\sqrt{M}}\E[\norm{\vz}_2] + \abs{\alpha}) \E[\norm{\mW}_2]  + \E[\norm{\vb}_2]\\
        &\qquad\le
            (\tfrac1{\sqrt{M}}\E[\norm{\vz}_2] + \abs{\alpha}) (\norm{\E[\mW]}_2 + c\sigma \sqrt{M})  + \norm{\E[\vb]}_2 + c \sigma \sqrt{M}
    \end{align}
    where in the second inequality we use the triangle inequality in combination with \cref{lem:expectation_bound}.
\end{proof}

We now turn to the proof of \cref{lem:last-layer-bound}, which relies of an argument with the following form.

\begin{proposition}\label{prop:recursion}
    Let $a, b\in \R$ and $\{c_\ell\}_{\ell=1}^L$, with $c_\ell \in \R$. Suppose, $c_\ell \le a c_{\ell-1} + b$ with $a, b \ge 0$ for $h > 1$.
    Then for $2 \leq \ell \leq L$
    \begin{equation}
        c_\ell \le a^{\ell - 1} c_1 + (1 + a)^{\ell - 2} b.
    \end{equation}
\end{proposition}
\begin{proof}
    The proof is a standard induction. In the base case $h=2$, we have to prove,
    \begin{align}
        c_2 \leq a c_1 + b,
    \end{align}
which holds because this is simply our assumption on the $c_\ell$ with $\ell=2$. For the inductive step, we now assume that $c_\ell \leq a^{\ell - 1} c_1 + (1 + a)^{\ell - 2} b$. Under this assumption, we have,
\begin{align}
    c_{\ell+1} &\leq  a c_{\ell} + b 
     \leq a^{\ell} c_1 + a(1 + a)^{\ell - 2} b + b 
     = a^{\ell} c_1 + (1 + a)^{\ell - 1} b.
\end{align}
\end{proof}

\begin{lemma}\label{lem:expected-norm-bound}
Define $K = \sqrt{2\KL(\Q,\P)}$. Then for $1 \le \ell \leq L$
\begin{align}
     \tfrac1{\sqrt{M}}\E_Q[\norm{\vz_\ell}_2] &\le (2 + \abs{\alpha} + \tfrac1{\sqrt{D\ss{i}}}\norm{\vx}_2)(2 + 2c)^\ell(K \lor 1)^\ell,
\end{align}
where $c$ is the same absolute constant from \cref{lem:contraction}.
\end{lemma}
\begin{proof}
By \cref{lem:parameter_bound}, all mean parameters are bounded by $K$ and $\sigma\ss{max} \le 2 (K \lor 1)$. Applying, \cref{lem:bound_norm} using these estimates, ,
\begin{align}
    \tfrac1{\sqrt{M}}\E_Q[\norm{\vz_\ell}_2]
    &\le (2c(K \lor 1) \!+\! \tfrac1{\sqrt{M}}K) \tfrac1{\sqrt{M}}\E[\norm{\vz_{\ell - 1}}_2] +
    (1 + \abs{\alpha})(2c(K \lor 1) \!+\! \tfrac1{\sqrt{M}}K), \\
    \tfrac1{\sqrt{M}}\E_Q[\norm{\vz_1}_2]
    &\le (1 + \tfrac1{\sqrt{D\ss{i}}}\norm{\vx}_2)(2c(K \lor 1) + \tfrac1{\sqrt{M}}K).
\end{align}

Bound $2c(K \lor 1) + \tfrac1{\sqrt{M}} K \le (1 + 2c)(K \lor 1)$ and apply \cref{prop:recursion}:
\begin{align}
    \tfrac1{\sqrt{M}}\E[\norm{\vz_\ell}_2]
    &\le (1 + \tfrac1{\sqrt{D\ss{i}}}\norm{\vx}_2)(1 + 2c)^\ell(K \lor 1)^\ell + (1 + \abs{\alpha})(2 + 2c)^{\ell - 1}(K \lor 1)^{\ell - 1}, \\
    &\le (2 + \abs{\alpha} + \tfrac1{\sqrt{D\ss{i}}}\norm{\vx}_2)(2 + 2c)^\ell(K \lor 1)^\ell. \qedhere
\end{align}
\end{proof}
\begin{lemma}\label{lem:norm-expecation}
    Define $K = \sqrt{2\KL(\Q,\P)}$. Then for $1 \le \ell \leq L$
\begin{align}
     \norm{\E_Q[\vz_\ell]}_2 &\le 2\ell (2 + \abs{\alpha} + \tfrac1{\sqrt{D\ss{i}}}\norm{\vx}_2)(2 + 2c)^{\ell-1}K(K \lor 1)^{\ell-1},
\end{align}
where $c$ is the same absolute constant from \cref{lem:contraction}
\end{lemma}
\begin{proof}
The proof proceeds by induction on $\ell$.
\paragraph{Base Case}
For the first layer, $\ell = 1$, by linearity of expectation, triangle inequality and the definition of the spectral norm,
\begin{equation}
    \norm{\E[\vz_1]}_2 \le \tfrac1{\sqrt{D\ss{i}}}\norm{\E[\mW_1]}\ss{2} \norm{\vx}_2 + \norm{\E[\vb_1]}_2 \le \tfrac1{\sqrt{D\ss{i}}}\norm{\E[\mW_1]}\ss{F} \norm{\vx}_2 + \norm{\E[\vb_1]}_2.
\end{equation}
We then apply
\cref{lem:parameter_bound} (i.), to conclude $\norm{\E[\mW_1]},\norm{\E[\vb_1]}_2 \leq K$,
\begin{equation}
    \norm{\E[\vz_1]}_2 \le (1+ \tfrac{1}{\sqrt{D\ss{i}}}\|\vx\|_2)K \leq 2(2+|\alpha|+\tfrac{1}{\sqrt{D\ss{i}}}\|\vx\|_2)K .
\end{equation}
\paragraph{Inductive step}
We take the inductive hypothesis, \begin{align}
    \norm{\E_Q[\vz_\ell]}_2 &\le 2\ell (2 + \abs{\alpha} + \tfrac1{\sqrt{D\ss{i}}}\norm{\vx}_2)(2 + 2c)^{\ell-1}K(K \lor 1)^{\ell-1}.
\end{align}
By an application of \cref{lem:bound}, 
\begin{equation}
    \norm{\E[\vz_{\ell+1}]}_2 \le (\tfrac1{\sqrt{M}}\E[\norm{\vz_{\ell}}_2] + \abs{\alpha})\norm{\E[\mW_{\ell+1}]}\ss{F} + c \sigma\ss{max} \norm{\E[\vz_{\ell}]}_2 + \norm{\E[\vb_{\ell+1}]}_2
\end{equation}
Applying
\cref{lem:parameter_bound} (i., iii.), 
\begin{equation}
   \norm{\E[\vz_{\ell+1}]}_2 \le (\tfrac1{\sqrt{M}}\E[\norm{\vz_{\ell}}_2] + \abs{\alpha} + 1)K + 2c (K\lor 1) \norm{\E[\vz_{\ell}]}_2 
\end{equation}
Using the inductive hypothesis, 
\begin{align}
    2c(K \lor 1)\norm{\E_Q[\vz_\ell]}_2] &\le 4c\ell (2 + \abs{\alpha} + \tfrac1{\sqrt{D\ss{i}}}\norm{\vx}_2)(2 + 2c)^{\ell-1}K(K \lor 1)^{\ell} \\
    & \le 2\ell (2 + \abs{\alpha} + \tfrac1{\sqrt{D\ss{i}}}\norm{\vx}_2)(2 + 2c)^{\ell}K(K \lor 1)^{\ell}.
\end{align}
We then make use of \cref{lem:expected-norm-bound},
\begin{align}
  K(1 + |\alpha| + \tfrac1{\sqrt{M}}\E_Q[\norm{\vz_\ell}_2]) &\le K\left(1 + |\alpha| + (2 + \abs{\alpha} + \tfrac1{\sqrt{D\ss{i}}}\norm{\vx}_2)(2 + 2c)^\ell(K \lor 1)^\ell\right) \\
  & \le  2(2 + \abs{\alpha} + \tfrac1{\sqrt{D\ss{i}}}\norm{\vx}_2)(2 + 2c)^\ell K(K \lor 1)^\ell
\end{align}
Hence,
\begin{equation*}
   \norm{\E[\vz_{\ell+1}]}_2 \le (2\ell+2)(2 + \abs{\alpha} + \tfrac1{\sqrt{D\ss{i}}}\norm{\vx}_2)(2 + 2c)^\ell K(K \lor 1)^\ell. \qedhere
\end{equation*}
\end{proof}

We are now ready to prove \cref{lem:last-layer-bound}.
\begin{proof}[Proof of \cref{lem:last-layer-bound}]

The proof is essentially identical to the inductive step in \cref{lem:norm-expecation}. 

To begin with, we apply \cref{lem:bound} to the deep architecture.
For the last hidden layer,%
\begin{equation}
    \norm{\E[\phi\ss{o}(\vz_L)]}_2 \le (\tfrac1{\sqrt{M}}\E[\norm{\vz_{L-1}}_2] + \abs{\alpha})\norm{\E[\mW_L]}\ss{F} + c \sigma\ss{max} \norm{\E[\vz_{L - 1}]}_2 + \norm{\E[\vb_{L}]}_2.
\end{equation}
We apply \cref{lem:parameter_bound} (i., iii.), yielding,
\begin{equation}
    \norm{\E[\phi\ss{o}(\vz_L)]}_2 \le (\tfrac1{\sqrt{M}}\E[\norm{\vz_{L-1}}_2] + \abs{\alpha})K+ 2c (K\lor 1) \norm{\E[\vz_{L - 1}]}_2 + K. \label{eqn:bound-w-preactivations}
\end{equation}

We then make use of \cref{lem:expected-norm-bound},
\begin{align}
  K(1 + |\alpha| + \tfrac1{\sqrt{M}}\E_Q[\norm{\vz_{L-1}}_2]) &\le K\left(1 + |\alpha| + (2 + \abs{\alpha} + \tfrac1{\sqrt{D\ss{i}}}\norm{\vx}_2)(2 + 2c)^{L-1}(K \lor 1)^{L-1}\right) \\
  & \le  2(2 + \abs{\alpha} + \tfrac1{\sqrt{D\ss{i}}}\norm{\vx}_2)(2 + 2c)^{L-1} K(K \lor 1)^{L-1}. \label{eqn:bound-norm}
\end{align}
From \cref{lem:norm-expecation},
\begin{align}
    2c(K \lor 1)\norm{\E_Q[\vz_{L-1}]}_2] &\le 4c(L-1) (2 + \abs{\alpha} + \tfrac1{\sqrt{D\ss{i}}}\norm{\vx}_2)(2 + 2c)^{L-2}K(K \lor 1)^{L-1} \\
    & \le 2(L-1) (2 + \abs{\alpha} + \tfrac1{\sqrt{D\ss{i}}}\norm{\vx}_2)(2 + 2c)^{L-1}K(K \lor 1)^{L-1}. \label{eqn:bound-exp}
\end{align}
Combining \cref{eqn:bound-w-preactivations,eqn:bound-norm,eqn:bound-exp} gives the result. \qedhere

\end{proof}

\section{Proof of Convergence of the Marginal Variance for Deep Networks}\label{app:convergence-of-variance}
 
We now turn to the problem of bounding the marginal variance of the predictive distribution for $\Q$. We work with the uncentered second moment, as in combination with the previous section this implies the marginal variance of $\P$ and $\Q$ agree.
The main result of this section is as follows:

\begin{theorem}[Convergence of second moment prediction]\label{thm:app-convergence-of-variance}
    Let $Q$ be a mean-field variational posterior and $P=\Normal(\vnull,\mI)$ denote the prior over a neural network with $L$ hidden layers and $M$ neurons per hidden layer.
    Suppose $\phi\ss{e}=\alpha$ for some $\alpha \in \R$ and $\phi\colon \R \to \R$ is $1$-Lipschitz.
    Let $\vx \in \R^{D\ss{i}}$ and let $\widetilde{\vf}_{\theta}(\vx) = \vf_\theta(\vx) - \tfrac\alpha{\sqrt{M}}\mW_{L+1}\vone - \vb_{L+1}$ denote the network output excluding final bias and even part of the final activation. Then 
    \begin{equation}
        \norm{\E_Q[\tilde{\vf}_{\vtheta}^2(\vx)] - \E_P[\tilde{\vf}_{\vtheta}^2(\vx)]}_\infty
        \le
            c_1L^{1/2} \rho_{\alpha, M}^L
            \frac
                {\alpha^2 + 1 + \tfrac1{D\ss{i}}\norm{\vx}_2^2}
                {\sqrt{M}}  
            \sqrt{\KL(\Q, \P)} (2\KL(\Q, \P) \lor 1)^{L+\tfrac12},
    \end{equation}
    where $c_1 = 16 + 25\sqrt{2} \in (51, 52)$ and where $\rho_{\alpha, M} \in (17, 97)$ is defined in \cref{lem:diag-frob-norm}.
\end{theorem}

The proof of \cref{thm:app-convergence-of-variance} will proceed by splitting the variance into two terms. The first term, which we name the \textit{diagonal}, arises from the product of the variance of the weights with the second moment of each activation. We show that under the optimal variational posterior, this term is close to the same term under the prior. Precisely, we have the following lemma:
\begin{lemma}[Diagonal Variance terms]\label{lem:app-diagonal-variance}
Define $\mSigma_{Q} =\E_{\Q}[\vw_{L+1}\vw_{L+1}^\T] -  \E_{\Q}[\vw_{L+1}]\E_{\Q}[\vw_{L+1}]^\T$, which is a diagonal matrix with the variances of $\vw_{L+1}$ as entries. Let 
\begin{align}
    D^i_\P(\vx) &= \tfrac{1}{M}\tr(\E_{\P}[\vw_{L+1}\vw_{L+1}^\T]\E_{\P}[\phi\ss{o}(\vz_{L})\phi\ss{o}(\vz_{L})^\T])\\
        D^i_\Q(\vx) &=\tfrac{1}{M}\tr(\mSigma_{Q}\E_{\Q}[\phi\ss{o}(\vz_{L})\phi\ss{o}(\vz_{L})^\T])
\end{align}
Then,
\begin{align}
    |D^{i}_{\Q}(\vx) - D^{i}_{\P}(\vx)| 
    &\leq  
    \left(16 + \sqrt{2}\right) L^{1/2}
    \rho_{\alpha, M}^{L}\frac{\alpha^2 + 1\lor \tfrac1{D\ss{i}}\|\vx\|^2_2}{\sqrt{M}}  \sqrt{\KL(\Q,\P)} (2 \KL(\Q,\P) \lor 1)^{L+\tfrac12}
\end{align}
where $\rho_{\alpha, M} \in (17, 97)$ is defined in \cref{lem:diag-frob-norm}.
\end{lemma}
The proof of \cref{lem:app-diagonal-variance} will be the main topic of \cref{app:variance-diagonal}.

The second term, which we term the  \textit{off-diagonal} arises from the product of the mean of the weights with the second moment of each activation. Under the prior, this term vanishes. Hence, we show that this term is small for the optimal approximate posterior. This leads to the following lemma,
\begin{lemma}[Off-Diagonal Variance terms]\label{lem:app-off-diagonal-variance}
Define $
O^i_{\Q}(\vx) = \frac{1}{M}\E_{\Q}[\vw_{L+1}]^\T\E_{\Q}[\phi\ss{o}(\vz_{L})\phi\ss{o}(\vz_{L})^\T]\E_{\Q}[\vw_{L+1}]$.
Then,
\begin{equation}
    |O^i_{\Q}(\vx)| \leq 48\gamma_\alpha^{L}\frac{\alpha^2 + 1\lor \tfrac1{D\ss{i}}\|\vx\|^2_2}{\sqrt{M}}
    \KL(\Q,\P)(2\KL(\Q,\P) \lor 1)^{L},
\end{equation}
where $\gamma_\alpha = 9 + \sqrt{83} \in (18,19)$ if $\alpha=0$ and $\gamma_\alpha = 55$ if $\alpha\neq0$.
\end{lemma}
The proof of \cref{lem:app-off-diagonal-variance} will be the main topic of \cref{app:variance-off-diagonal}.

These two lemmas taken together allow us to prove \cref{thm:app-convergence-of-variance}. The entire structure of the proof is depicted in \cref{fig:dag-variance}.

\begin{proof}[Proof of \cref{thm:app-convergence-of-variance}]
We will repeatedly use cyclic property and linearity of trace. Hence, we briefly recall that for conformable matrix $\mA, \mB$, we have $\tr(\mA\mB)=\tr(\mB\mA)$ and $\E[\tr(\mA\mB)] = \tr(\E[\mA\mB])$. We also observe that the trace of a $1\times 1$ matrix is simply the corresponding scalar.

It suffices to consider an arbitrary output component. Let $\widetilde{f}(\vx) = \widetilde{\vf}(\vx)_i$ denote an arbitrary output component $i$ and let $\vw_{L+1}$ denote row $i$ of $\mW_{L+1}$, so that $\widetilde{f}(\vx) = \tfrac{1}{M} \vw_{L+1}^T\phi\ss{o}(\vz_L)$. 
\begin{align}
M\widetilde{f}(\vx)^2 &= \vw_{L+1}^\T\phi\ss{o}(\vz_{L})\phi\ss{o}(\vz_{L})^\T \vw_{L+1} \\
& =\tr\left(\vw_{L+1}^\T\phi\ss{o}(\vz_{L})\phi\ss{o}(\vz_{L})^\T \vw_{L+1}\right) \\
& = \tr\left(\vw_{L+1}\vw_{L+1}^\T\phi\ss{o}(\vz_{L})\phi\ss{o}(\vz_{L})^\T \right).
\end{align}
Recall $\mSigma_{Q} =\E_{\Q}[\vw_{L+1}\vw_{L+1}^\T] -  \E_{\Q}[\vw_{L+1}]\E_{\Q}[\vw_{L+1}]^\T$, which is a diagonal matrix (by indpendence) with the variances of $\vw_{L+1}$ as entries. Then, adding and subtracting $\E_{\Q}[\vw_{L+1}]\E_{\Q}[\vw_{L+1}]^\T$, and using that $\vw_{L+1}$ is independent of $\vz_{L}$ by the independence assumption
\begin{align}
    \E_{\Q}[\widetilde{f}(\vx)^2] &=
    \tfrac{1}{M}\tr((\E_{\Q}[\vw_{L+1}]\E_{\Q}[\vw_{L+1}]^\T + \mSigma_{Q}) \E_{\Q}[\phi\ss{o}(\vz_{L})\phi\ss{o}(\vz_{L})^\T] ) \\
    & = \tfrac{1}{M}\tr(\E_{\Q}[\vw_{L+1}]\E_{\Q}[\vw_{L+1}]^\T\E_{\Q}[\phi\ss{o}(\vz_{L})\phi\ss{o}(\vz_{L})^\T] ) + \tfrac{1}{M}\tr(\mSigma_{Q}\E_{\Q}[\phi\ss{o}(\vz_{L})\phi\ss{o}(\vz_{L})^\T]) \\
    & = \tfrac{1}{M}\E_{\Q}[\vw_{L+1}]^\T\E_{\Q}[\phi\ss{o}(\vz_{L})\phi\ss{o}(\vz_{L})^\T]\E_{\Q}[\vw_{L+1}] + D^{i}_{\Q} \\
    &= O^{i}_{\Q}(\vx) + D^{i}_{\Q}(\vx).
\end{align}
Also, $\E_{\P}[\widetilde{f}(\vx)^2] =D^{i}_{\P}(\vx)$ since $\E_{\P}[\vw_{L+1}]=0$. So,
\begin{align}
    |\E_{\Q}[\widetilde{f}(\vx)^2] - \E_{\P}[\widetilde{f}(\vx)^2]| &= |O^{i}_{\Q}(\vx) + D^{i}_{\Q}(\vx) - D^{i}_{\P}(\vx)| \\
    & \leq |O^{i}_{\Q}(\vx)|  + |D^{i}_{\Q}(\vx) - D^{i}_{\P}(\vx)|.
\end{align}
The final result is obtained by plugging in the expressions from Lemmas \ref{lem:app-diagonal-variance} and \ref{lem:app-off-diagonal-variance}, noting that $\gamma_{\alpha} \le \gamma_{\alpha, M}$.
\end{proof}

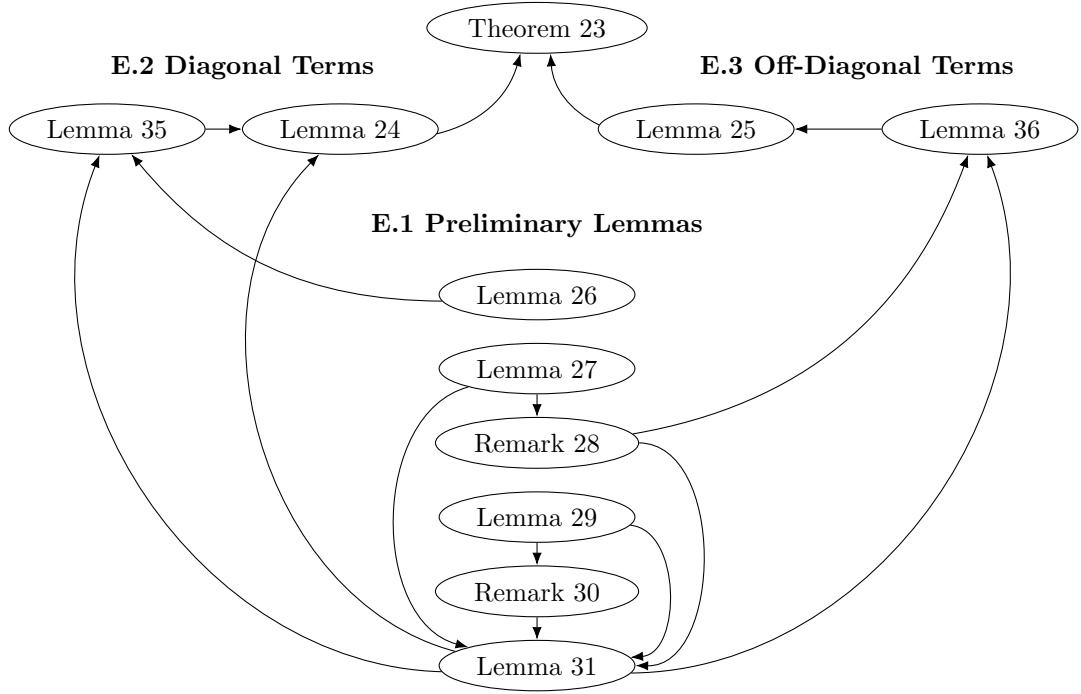
\begin{figure}[t]
\centering
	\begin{tikzpicture}
		\node[state] (main) at (0,0) {\cref{thm:app-convergence-of-variance}};
		
		
		\node[] (E2) [below  left = of main, yshift=1.0cm] {\textbf{E.2   Diagonal Terms}};
		\node[state] (E2lem1) [below  left = of E2, xshift=2.0cm, yshift=0.7cm] {\cref{lem:diag-frob-norm}}; 
		\node[state] (E2lem2) [below right  = of E2, xshift=-2.5cm, yshift=0.7cm] {\cref{lem:app-diagonal-variance}}; 
		
		\path (E2lem2) edge[bend right=30] (main);
		\path (E2lem1) edge (E2lem2);
		
		
		\node[] (E3) [below  right = of main, yshift=1.0cm] {\textbf{E.3   Off-Diagonal Terms}};
		\node[state] (E3lem1) [below  left = of E3, xshift=2.0cm, yshift=0.7cm] {\cref{lem:app-off-diagonal-variance}}; 
		\node[state] (E3lem2) [below right  = of E3, xshift=-2.5cm, yshift=0.7cm] {\cref{lem:offdiag-operator-norm}}; 

		\path (E3lem1) edge[bend left=30] (main);
		\path (E3lem2) edge (E3lem1);
		
		\node[] (E1) [below = of main, yshift=-1.0cm] {\textbf{E.1   Preliminary Lemmas}};
		\node[state] (E1lem1) [below  = of E1, yshift=0.7cm] {\cref{lem:expectation_squared_bound_update}}; 
		\node[state] (E1lem2) [below  = of E1lem1, yshift=0.7cm] {\cref{lem:outer}}; 
		\node[state] (E1lem3) [below  = of E1lem2, yshift=0.7cm] {\cref{rem:outer_offset}}; 
		\node[state] (E1lem4) [below  = of E1lem3, yshift=0.7cm] {\cref{lem:inf}}; 
		\node[state] (E1lem5) [below  = of E1lem4, yshift=0.7cm] {\cref{rem:inf_offset}}; 
		\node[state] (E1lem6) [below  = of E1lem5, yshift=0.7cm] {\cref{lem:z-squared-recursion}}; 
		
		\path (E1lem2) edge (E1lem3);
		\path (E1lem4) edge (E1lem5);
		\path (E1lem5) edge (E1lem6);
		
		\path (E1lem2) edge[bend right=75] (E1lem6);
		\path (E1lem4) edge[bend left=85] (E1lem6);
		
		\path (E1lem3) edge[bend left=90] (E1lem6);
		
		\path (E1lem3) edge[bend right=30] (E3lem2);
		\path (E1lem6) edge[bend right=55] (E3lem2);
=
		\path (E1lem1) edge[bend left=25] (E2lem1);
		\path (E1lem6) edge[bend left=60] (E2lem2);
		\path (E1lem6) edge[bend left=55] (E2lem1);
	\end{tikzpicture}
	\caption{Dependency structure of the results leading to  \cref{thm:app-convergence-of-variance}. Lemmas \ref{lem:cov}, \ref{lem:second_moment_conditional_prior}, and \ref{lem:inner_prod} are proved in this section but only used in \cref{app:convergence-in-dist}.}
	\label{fig:dag-variance}
\end{figure}

\subsection{Preliminary Lemmas}

In order to prove both \cref{lem:app-diagonal-variance} and \cref{lem:app-off-diagonal-variance} we first establish several preliminary results. The first lemma will be useful in bounding the diagonal terms in \cref{lem:diag-frob-norm}. We construct the bound $\eta_{I \lor J}$ to be convenient when collecting terms in \cref{lem:diag-frob-norm}.

\begin{lemma} \label{lem:expectation_squared_bound_update}
    Let $\mA \in \R^{I \times J}$ have independent standard Gaussian entries.
    Then
    \begin{equation}
        \E[\norm{\mA}^2_2]
        \le \eta_{I \lor J} (I \lor J),
    \end{equation}
    where $\eta_{I \lor J}$ takes the value $\tfrac19(37+6\sqrt{2\pi}) \in (5,6)$ if $I\lor J\ge 36$ and $4(2+\sqrt{2\pi}) \in (18,19)$ otherwise.
\end{lemma}
\begin{proof}
    We integrate the tail bound from \cref{lem:expectation_bound}:
    \begin{equation}
        \mathbb{P}(\norm{\mA}_2 \ge 2 \sqrt{I \lor J} + \delta) \le 2 \exp(-\tfrac1{2} \delta^2).
    \end{equation}
     To integrate the tail bound, use the layer cake trick:
    \begin{align}
        \E[\norm{\mA}^2_2]
        &= \int_0^\infty \P(\norm{\mA}^2_2 > u) \isd u \\
        &\le 4(I \lor J) + \int_0^\infty \P(\norm{\mA}^2_2 > 4(I \lor J) + u) \isd u.
    \end{align}
    Consider the change of variables defined by
    \begin{equation}
        4(I \lor J) + u = (2 \sqrt{I \lor J} + v)^2.
    \end{equation}
    Then
    $\sd u = 2( 2\sqrt{I \lor J} + v) \isd v$,
    so
    \begin{align}
        \E[\norm{\mA}^2_2]
        &\le
            4(I \lor J)
            + 2\int_0^\infty \P(\norm{\mA}_2 > 2 \sqrt{I \lor J} + v) (2 \sqrt{I \lor J} + v) \isd v \\
        &\le\vphantom{\int_0^\infty}
            4(I \lor J)
            + 4\int_0^\infty e^{-\tfrac12 v^2} (2 \sqrt{I \lor J} + v) \isd v \\
        &=\vphantom{\int_0^\infty}
            4(I \lor J)
            + 4 \sqrt{2\pi} \sqrt{I \lor J} + 4.
    \end{align}
The claimed bound then follows by checking cases.
\end{proof}

\begin{lemma} \label{lem:outer}
    Let $\phi_1,\phi_2\colon\R \to \R$ be $1$-Lipschitz and odd.
    Let the triple $(\mW, \vz_2, \vb) \in \R^{M_1\times M_2} \times \R^{M_2} \times \R^{M_1}$ be (possibly dependent) random variables such that, for all Rademacher vectors $\vep \in \set{-1, 1}^{M_1}$,
    \begin{equation}
        (\mW - \E[\mW], \vz_2, \vb - \E[\vb])
        \disteq
        (\diag(\vep)(\mW - \E[\mW]), \vz_2, \vep \had (\vb - \E[\vb])).
    \end{equation}
    Consider $\vz_1 = \tfrac1{\sqrt{M_2}}\mW \phi_2(\vz_2) + \vb$.
    Then
    \begin{align}
        \norm{\E[\phi_1(\vz_1) \phi_1(\vz_1)^\T]}_2
        &\le
            4\sqrt{M_1}(
                \norm{\E[\vz_1^2]}_\infty
                + \norm{\E[\mW]}_2^2\, \norm{\E[\vz_{2}^2]}_\infty
                + \norm{\E[\vb]}_2^2
            )\\
        &\le
            4\sqrt{M_1}(
                \norm{\E[\vz_1^2]}_\infty
                + K^2\, (\norm{\E[\vz_{2}^2]}_\infty \lor 1)
            )
    \end{align}
    provided that $M_1 \ge 3$. 
    If $M_1 < 3$, then the inequality holds with the slightly worse constant $6$.
\end{lemma}
\begin{proof}
    Set
    \begin{align}
        \vz_1' &= \vz_1 - \tfrac1{\sqrt{M_2}}\E[\mW] \phi_2(\vz_2) - \E[\vb] = \tfrac1{\sqrt{M_2}}(\mW - \E[\mW])\phi_2(\vz_2) + (\vb - \E[\vb]). \label{eqn:z1'}
    \end{align}
    Consider
    \begin{align}
        \norm{\E[\phi_1(\vz_1) \phi_1(\vz_1)^\T]}_2
        &\le \norm{\E[(\phi_1(\vz_1) - \phi_1(\vz'_1)) \phi_1(\vz_1)^\T]}_2
        + \norm{\E[\phi_1(\vz'_1) \phi_1(\vz_1)^\T]}_2 \\
        &\overset{\smash{\text{(i)}}}{\le} \E[\norm{(\phi_1(\vz_1) - \phi_1(\vz'_1)) \phi_1(\vz_1)^\T}_2]
        + \norm{\E[\phi_1(\vz'_1) \phi_1(\vz_1)^\T]}_2 \\
        &\overset{\smash{\text{(ii)}}}{=} \E[\norm{\phi_1(\vz_1) - \phi_1(\vz'_1)}_2 \norm{\phi_1(\vz_1)}_2]
        + \norm{\E[\phi_1(\vz'_1) \phi_1(\vz_1)^\T]}_2 \\
        &\overset{\smash{\text{(iii)}}}{\le} \E[\norm{\vz_1 - \vz'_1}_2 \norm{\vz_1}_2] 
        + \norm{\E[\phi_1(\vz'_1) \phi_1(\vz_1)^\T]}_2 \\
        &\overset{\smash{\text{(iv)}}}{\le} (\E[\norm{\vz_1 - \vz'_1}^2_2] \,\E[\norm{\vz_1}^2_2])^{1/2}
        + \norm{\E[\phi_1(\vz'_1) \phi_1(\vz_1)^\T]}_2
    \end{align}
    where in (i) we use Jensen's Inequality, in (ii) we compute the $2$-norm, in (iii) we use 1-Lipschitzness and oddness of $\phi_1$, and in (iv) we use Cauchy--Schwarz.
    In a similar way, we can simplify the last term from above:%
    \begin{equation}
        \norm{\E[\phi_1(\vz'_1) \phi_1(\vz_1)^\T]}_2
        \le (\E[\norm{\vz'_1}^2_2]\,\E[\norm{\vz_1 - \vz'_1}^2_2])^{1/2} + \norm{\E[\phi_1(\vz'_1) \phi_1(\vz'_1)^\T]}_2.
    \end{equation}
    Therefore,
    \begin{align}
        \norm{\E[\phi_1(\vz_1) \phi_1(\vz_1)^\T]}_2 
        \le
            \E[\norm{\vz_1 - \vz'_1}^2_2]^{1/2}
            (\E[\norm{\vz_1}^2_2]^{1/2} + \E[\norm{\vz'_1}^2_2]^{1/2})
            + \norm{\E[\phi_1(\vz'_1) \phi_1(\vz'_1)^\T]}_2.\label{eqn:phi1_z1_phi1_z1}
    \end{align}
    Unfortunately, the square roots cannot be pushed inside: Jensen's inequality is the other way around. Instead, we bound each of the three terms involving $\vz'$ in \cref{eqn:phi1_z1_phi1_z1} separately, starting with the first term. Applying the inequality $(a + b)^2 \le 2a^2 + 2b^2$ to \cref{eqn:z1'}, we have
    \begin{align}
        \E[\norm{\vz_1 - \vz'_1}^2_2]
        &\le 2\norm{\E[\mW]}_2^2 \tfrac1{M_2} \E[\norm{\phi_2(\vz_{2})}^2_2] + 2\norm{\E[\vb]}_2^2 \\
        &\le 2\norm{\E[\mW]}_2^2 \tfrac1{M_2} \E[\norm{\vz_{2}}^2_2] + 2\norm{\E[\vb]}_2^2 \label{eqn:use_lipschitz} \\
        &\le 2\norm{\E[\mW]}_2^2 \norm{\E[\vz_2^2]}_\infty + 2\norm{\E[\vb]}_2^2.
    \end{align}
    Above, the second inequality follows from the oddness and 1-Lipschitzness of $\phi_2$ and that $\norm{\vone}_2^2 = M_2$
    while the third inequality follows from converting the 2-norm to an $\infty$-norm:
    \begin{equation}
    \E[\norm{\vz_2}_2^2] \le \E\left[\sum_{i=1}^{M_2} z_{2,i}^2\right] = \sum_{i=1}^{M_2} \E[z_{2,i}^2] \le M_2 \norm{\E[\vz_2^2]}_\infty. \label{eqn:convert_2_infty}
    \end{equation}
    Similarly, 
    \begin{equation}
        \E[\norm{\vz'_1}^2_2]
        \le 3\E[\norm{\vz_1}^2_2] + 3\norm{\E[\mW]}_2^2\, \norm{\E[\vz_{2}^2]}_\infty + 3\norm{\E[\vb]}_2^2.
    \end{equation}
    For convenience, define the expression
    \begin{equation}
    c = (\norm{\E[\mW]}_2^2 \norm{\E[\vz_{2}^2]}_\infty + \norm{\E[\vb]}_2^2)^{1/2}.
    \end{equation}
    Using the subadditivity of the square root, we can write the previous two expressions as
    \begin{equation}
        \E[\norm{\vz_1 - \vz'_1}^2_2]^{1/2} \le \sqrt{2}c \label{eqn:z1_minutes_z1'}
    \end{equation}
    and
    \begin{equation}
        \E[\norm{\vz'_1}^2_2]^{1/2} \le \sqrt{3}\E[\norm{\vz_1}^2_2]^{1/2} + \sqrt{3}c, \label{eqn:z1-2}
    \end{equation}
    respectively. To bound the last term in \cref{eqn:phi1_z1_phi1_z1}, 
    let $\vw_m$ denote the $m$\textsuperscript{th} row of $\mW$.
    Also let $\vw_m' = \vw_m - \E[\vw_m]$ and $\vb' = \vb - \E[\vb]$.
    Let $m, m' \in [M_1]$, $m \neq m'$.
    By the assumed symmetry condition,
    \begin{equation}
        ((\vw_m, \vw_{m'}), \vz_2, (b_m, b_{m'}))
        \disteq
        ((-\vw_m, \vw_{m'}), \vz_2, (-b_m, b_{m'})).
    \end{equation}
    Therefore, using oddness,
    \begin{align}
        -\phi_1(z_{1,m}') \phi_1(z_{1,m'}')
        &=
            -\phi_1(\lra{\tfrac1{M_2}\vw_m', \phi_2(\vz_2)} + b_m') 
            \phi_1(\lra{\tfrac1{M_2}\vw_{m'}', \phi_2(\vz_2)} + b_{m'}') \\
        &=
            \phi_1(\lra{-\tfrac1{M_2}\vw_m', \phi_2(\vz_2)} - b_m') 
            \phi_1(\lra{\tfrac1{M_2}\vw_{m'}', \phi_2(\vz_2)} + b_{m'}') \\
        &\disteq
            \phi_1(\lra{\tfrac1{M_2}\vw_m', \phi_2(\vz_2)} + b_m') 
            \phi_1(\lra{\tfrac1{M_2}\vw_{m'}', \phi_2(\vz_2)} + b_{m'}') \\
        &=\phi_1(z_{1,m}') \phi_1(z_{1,m'}'),
    \end{align}
    which means that $\E[\phi_1(z_{1,m}') \phi_1(z_{1,m'}')] = 0$. Consequently, the matrix $\E[\phi_1(\vz_1') \phi_1(\vz_1')^\T]$ is diagonal, so
    \begin{equation}
        \norm{\E[\phi_1(\vz'_1) \phi_1(\vz'_1)^\T]}_2
        = \max_{m \in [M_1]}\, \E[\phi_1^2(z'_{1,m})]
        = \norm{\E[\phi_1^2(\vz'_{1})]}_\infty.
    \end{equation}
    To bound $\norm{\E[\phi_1^2(\vz'_{1})]}_\infty$, consider that
    \begin{align}
        \phi_1^2(\vz'_{1})
        &\preceq (\vz'_{1})^2 \\
        &\preceq 3 \vz^2_{1} + 3 \tfrac1{M_2}(\E[\mW] \phi_2(\vz_{2}))^2 + 3\,\E[\vb]^2 \\
        &\preceq 3 \vz^2_{1} + 3 \norm{\E[\mW]}_2^2 \tfrac1{M_2}\norm{\phi_2(\vz_{2})}_2^2 \vone + 3\,\E[\vb]^2 \\
        &\preceq 3 \vz^2_{1} + 3 \norm{\E[\mW]}_2^2 \tfrac1{M_2}\norm{\vz_{2}}_2^2 \vone + 3\,\E[\vb]^2
    \end{align}
    where the squares and inequalities are element-wise. We use the oddness and 1-Lipschitzness of $\phi_1$ and $\phi_2$ in the first and last inequalities, respectively, and we use the manipulation $\mA\vv \preceq \norm{\mA\vv}_\infty \vone \preceq \norm{\mA\vv}_2 \vone \preceq \norm{\mA}_2\norm{\vv}_2 \vone$ in the third inequality. Therefore, using that $\norm{\E[\vb]}_\infty \le \norm{\E[\vb]}_2$, we have
    \begin{align}
        \norm{\E[\phi_1^2(\vz'_{1})]}_\infty
        &\le
            3 \norm{\E[\vz_1^2]}_\infty
            + 3 \norm{\E[\mW]}_2^2 \tfrac1{M_2}\E[\norm{\vz_{2}}_2^2]
            + 3 \norm{\E[\vb]}_\infty^2 \\
        &=
            3 \norm{\E[\vz_1^2]}_\infty + 3 c^2 \label{eqn:phi2z1'}
    \end{align}
    We can now plug the bounds on the $\vz_1'$ terms, given by Equations (\ref{eqn:z1_minutes_z1'}), (\ref{eqn:z1'}), and (\ref{eqn:phi2z1'}), into \cref{eqn:phi1_z1_phi1_z1}:
    \begin{align}
        \norm{\E[\phi_1(\vz_1) \phi_1(\vz_1)^\T]}_2
        &\le
            \sqrt{2}c \left((1+\sqrt{3})\E[\norm{\vz_1}^2_2]^{1/2} + \sqrt{3}c\right) + 3 \norm{\E[\vz_1^2]}_\infty + 3 c^2 \\
        &\le
            \sqrt{2}(1+\sqrt{3})\E[\norm{\vz_1}^2_2]^{1/2} c
            + (3+\sqrt{6})( \norm{\E[\vz_1^2]}_\infty + c^2) \\
        &\le
            \sqrt{2}(1+\sqrt{3})\sqrt{M_1}\norm{\E[\vz_1^2]}_\infty^{1/2} c + (3+\sqrt{6})( \norm{\E[\vz_1^2]}_\infty + c^2),
    \end{align}
    where we add an extra $2\sqrt{3}\norm{\E[\vz_1^2]}_\infty$ term to more simply group the terms in second inequality and we covert the 2-norm of $\vz_2$ to an $\infty$-norm in the third inequality, as in \cref{eqn:convert_2_infty}. Next, since $2ab + a^2 + b^2 = (a + b)^2 \le 2a^2 + 2b^2 \implies ab \le \tfrac12(a^2 + b^2)$, we can simplify the following expression in the first term above as
    \begin{equation}
        \norm{\E[\vz_1^2]}_\infty^{1/2} c \le \tfrac{1}{2}( \norm{\E[\vz_1^2]}_\infty + c^2).
    \end{equation}
    Therefore, 
    \begin{align}
        \norm{\E[\phi_1(\vz_1) \phi_1(\vz_1)^\T]}_2
        &\le
            \tfrac12\sqrt{2}(1+\sqrt{3})\sqrt{M_1}( \norm{\E[\vz_1^2]}_\infty + c^2)
            + (3+\sqrt{6})(\norm{\E[\vz_1^2]}_\infty + c^2) \\
        &\le
            \left(\tfrac12\sqrt{2}(1+\sqrt{3})\sqrt{M_1} + (3+\sqrt{6}) \right)
            (\tfrac12\norm{\E[\vz_1^2]}_\infty + c^2) 
    \end{align}
    If $M_1 \ge 3$, then $(3+\sqrt{6}) \le \tfrac12\sqrt{2}(1+\sqrt{3}) \sqrt{M_1}$, so we have
    \begin{align}
        \norm{\E[\phi_1(\vz_1) \phi_1(\vz_1)^\T]}_2
        &\le
            \sqrt{2}(1+\sqrt{3})\sqrt{M_1}( \norm{\E[\vz_1^2]}_\infty + c^2 \\
        &\le
            4\sqrt{M_1} ( \norm{\E[\vz_1^2]}_\infty^{1/2} + \norm{\E[\mW]}_2^2 \norm{\E[\vz_{2}^2]}_\infty + \norm{\E[\vb]}_2^2).
    \end{align}
    Otherwise, if $M_1 < 3$, then the inequality holds with the slightly worse constant $\sqrt{2}(1+\sqrt{3}) +  (3+\sqrt{6}) \approx 7.4 \le 8$:
    \begin{align}
        \norm{\E[\phi_1(\vz_1) \phi_1(\vz_1)^\T]}_2
        &\le
            8\sqrt{M_1} ( \norm{\E[\vz_1^2]}_\infty^{1/2} + \norm{\E[\mW]}_2^2 \norm{\E[\vz_{2}^2]}_\infty + \norm{\E[\vb]}_2^2).
    \end{align}
\end{proof}

\begin{remark} \label{rem:outer_offset}
    The symmetry condition for $(\mW, \vz_2, \vb)$ is satisfied if $\mW$ and $\vb$, not $\vz_2$, are element-wise independent and symmetric around their means.
    For such $(\mW, \vz_2, \vb)$, the symmetry condition is also satisfied for the triple $(\mW, \vz_2, \tfrac{\alpha}{\sqrt{M_2}} \mW \vone + \vb)$
    where $\vone$ is the vector of all ones and $\alpha \in \R$.
    In that case, a similar conclusion holds:
    if instead
    \begin{equation}
        \vz_1
            = \tfrac1{\sqrt{M_2}}\mW (\phi_2 + \alpha)(\vz_2) + \vb
            = \tfrac1{\sqrt{M_2}}\mW \phi_2 (\vz_2) + (\tfrac{\alpha}{\sqrt{M_2}} \mW \vone + \vb),
    \end{equation}
    then, if $M_1\ge 3$, by \cref{lem:outer} 
    \begin{align}
        \norm{\E[\phi_1(\vz_1) \phi_1(\vz_1)^\T]}_2
        &\le
            4\sqrt{M_1}(
                \norm{\E[\vz_1^2]}_\infty
                + \norm{\E[\mW]}_2^2\, \norm{\E[\vz_{2}^2]}_\infty
                + \norm{\E[\tfrac{\alpha}{\sqrt{M_2}} \mW \vone + \vb]}_2^2
            )\\
        &\le
            4\sqrt{M_1}(
                \norm{\E[\vz_1^2]}_\infty
                + \norm{\E[\mW]}_2^2\, \norm{\E[\vz_{2}^2]}_\infty
                + \tfrac{2\alpha^2}{M_2} \norm{\E[\mW]}_2^2 \norm{\vone}_2^2 + 2\norm{\E[\vb]}_2^2
            )\\
        &\overset{\smash{\text{(i)}}}{\le}
            4\sqrt{M_1}(
                \norm{\E[\vz_1^2]}_\infty
                + \norm{\E[\mW]}_2^2 (2 \alpha^2 + \norm{\E[\vz_{2}^2]}_\infty) + 2\norm{\E[\vb]}_2^2
            )\\
        &\le
            8\sqrt{M_1}(
                \norm{\E[\vz_1^2]}_\infty
                + \norm{\E[\mW]}_2^2 (\alpha^2 + \norm{\E[\vz_{2}^2]}_\infty) + \norm{\E[\vb]}_2^2
            )\\
        &\overset{\smash{\text{(ii)}}}{\le}
            8\sqrt{M_1}(
                \norm{\E[\vz_1^2]}_\infty
                + K^2\, ((\alpha^2 +  \norm{\E[\vz_{2}^2]}_\infty) \lor 1)
            ) \\
        &\le
            8\sqrt{M_1}(
                \norm{\E[\vz_1^2]}_\infty
                + K^2\, (\alpha^2 + 1 \lor \norm{\E[\vz_{2}^2]}_\infty)
            )
    \end{align}
    where in (i) we use that $\norm{\vone}_2^2 = M_2$ and in (ii) we use that $\norm{\E[\mW]}_2^2 + \norm{\E[\vb]}_\infty^2 \le K^2$. If $M_1<3$, then the inequality holds with the slightly worse constant of 12. 
\end{remark}

Towards developing a recursion, we now express $\norm{\E[\vz_{1}^2]}_\infty$ in terms of $\norm{\E[\vz_{2}^2]}_\infty$.

\begin{lemma} \label{lem:inf}     Assume the conditions of the \cref{lem:outer}.
    Then
    \begin{align}
        \norm{\E[\vz_1^2]}_\infty 
        &\le \tfrac2{M_2} \norm{\V\ss{d}[\operatorname{vec}(\mW)]}_\infty \norm{\E[\vz_{2}^2]}_1
        + \tfrac2{M_2} \norm{\E[\mW]}\ss{F}^2 \norm{\E[\phi_2(\vz_{2})\phi_2(\vz_{2})^\T]}_2
        + 2 \norm{\E[\vb^2]}_\infty \label{eqn:lem-outer-part1}\\
        &\le
            2(\sqrt{2}+K)^2 (1 \lor \tfrac1{M_2} \norm{\E[\vz_{2}^2]}_1)
            + \tfrac2{M_2} K^2 \norm{\E[\phi_2(\vz_{2})\phi_2(\vz_{2})^\T]}_2
    \end{align}
    and
    \begin{align}
        \norm{\E[\vz_1^2]}_\infty 
        &\le 2 \norm{\V\ss{d}[\operatorname{vec}(\mW)]}_\infty \norm{\E[\vz_{2}^2]}_\infty
        + \tfrac2{M_2} \norm{\E[\mW]}\ss{F}^2 \norm{\E[\phi_2(\vz_{2})\phi_2(\vz_{2})^\T]}_2
        + 2 \norm{\E[\vb^2]}_\infty \label{eqn:lem-outer-part2} \\
        &\le
            2(\sqrt{2}+K)^2 (1 \lor \norm{\E[\vz_{2}^2]}_\infty)
            + \tfrac2{M_2} K^2 \norm{\E[\phi_2(\vz_{2})\phi_2(\vz_{2})^\T]}_2.
    \end{align}
\end{lemma}
\begin{proof}
    Let $m \in [M_2]$, and let $\vw_m$ denote the $m$\textsuperscript{th} row of $\mW$, so that $z_{1,m}=\tfrac{1}{\sqrt{M_2}}\vw_m^T\phi_2(\vz_2) + b_m$. Using $(a + b)^2\le 2a^2 + 2b^2$,
    \begin{align}
        \E[z_{1,m}^2]
        &\le \tfrac2{M_2}\,\E[(\vw_m^T \phi_2(\vz_2))^2] + 2\,\E[b_{m}^2] \\
        &\le \tfrac2{M_2} \lra{\E[\vw_m^2], \E[\phi_2^2(\vz_{2})]} + C + 2\,\E[b_{m}^2] \\
        &= \tfrac2{M_2} \lra{\V\ss{d}[\vw_m], \E[\phi_2^2(\vz_{2})]} + C + \tfrac2{M_2} \lra{\E^2[\vw_m], \E[\phi_2^2(\vz_{2})]}  + 2\,\E[b_{m}^2],
    \end{align}
    where $C$ is the sum of the off-diagonal terms of $\tfrac2{M_2}\E[(\vw_m^T \phi_2(\vz_2))^2]$. 
    Hence, by the triangle inequality,
    \begin{align}
        \norm{\E[\vz_1^2]}_\infty
        &\le \tfrac2{M_2} \max_m \{\abs{\lra{\V\ss{d}[\vw_m], \E[\phi_2^2(\vz_{2})]}} \} + \max_m \{\abs{C + \tfrac2{M_2} \lra{\E^2[\vw_m], \E[\phi_2^2(\vz_{2})]}}\} + 2\, \norm{\E[\vb^2]}_\infty. \label{eqn:z1_infty}
    \end{align}
    Consider the middle term in \cref{eqn:z1_infty}. We have
    \begin{align}
        C + \tfrac2{M_2} \lra{\E^2[\vw_m], \E[\phi_2^2(\vz_{2})]}
        &= \tfrac2{M_2} \lra{\E[\vw_m], \E[\phi_2(\vz_{2})\phi_2(\vz_{2})^\T] \E[\vw_m]} \\
        &\le \tfrac2{M_2} \norm{\E[\vw_m]}_2^2 \norm{\E[\phi_2(\vz_{2})\phi_2(\vz_{2})^\T]}_2,
    \end{align}
    where the second inequality follows from the Cauchy-Schwarz inequality and the definition of the operator norm. Therefore,
    \begin{align}
        \max_m \{ \abs{C + \tfrac2{M_2} \lra{\E^2[\vw_m], \E[\phi_2^2(\vz_{2})]} } \}
        &\le \tfrac2{M_2} \max_m \{ \norm{\E[\vw_m]}_2^2 \} \norm{\E[\phi_2(\vz_{2})\phi_2(\vz_{2})^\T]}_2 \\
        &\le \tfrac2{M_2} \norm{\E[\mW]}\ss{F}^2 \norm{\E[\phi_2(\vz_{2})\phi_2(\vz_{2})^\T]}_2.
    \end{align}
    To bound the first term in \cref{eqn:z1_infty}, use the H{\"o}lder inequality with the conjugate pair $(\infty, 1)$:
    \begin{align}
        \tfrac2{M_2} \max_m \{\abs{\lra{\V\ss{d}[\vw_m], \E[\phi_2^2(\vz_{2})]}} \} 
        &\le \tfrac2{M_2} \max_m \{ \norm{\V\ss{d}[\vw_m]}_\infty \norm{\E[\phi_2^2(\vz_{2})]}_1 \} \\
        &= \tfrac2{M_2} \max_m \{ \norm{\V\ss{d}[\vw_m]}_\infty \} \norm{\E[\phi_2^2(\vz_{2})]}_1 \\
        &= \tfrac2{M_2} \norm{\V\ss{d}[\operatorname{vec}(\mW)]}_\infty \norm{\E[\phi_2^2(\vz_{2})]}_1.
    \end{align}
    Notice that since $\phi_2$ is 1-Lipschitz and odd, we can write
    \begin{align}
        \norm{\E[\phi_2^2(\vz_{2})]}_1
        = \norm{\E[\abs{\phi_2(\vz_{2}) - \phi_2(\vnull) }^2]}_1
        \le \norm{\E[\abs{\vz_{2} - \vnull}^2]}_1
        = \norm{\E[\vz_{2}^2]}_1.
    \end{align}
    Plugging into \cref{eqn:z1_infty},  we have
    \begin{align}
        \norm{\E[\vz_1^2]}_\infty 
        &\le \tfrac2{M_2} \norm{\V\ss{d}[\operatorname{vec}(\mW)]}_\infty \norm{\E[\vz_{2}^2]}_1
        + \tfrac2{M_2} \norm{\E[\mW]}\ss{F}^2 \norm{\E[\phi_2(\vz_{2})\phi_2(\vz_{2})^\T]}_2
        + 2 \norm{\E[\vb^2]}_\infty \\
        &\le
            2(
            \norm{\V\ss{d}[\operatorname{vec}(\mW)]}_\infty
                + \norm{\E[\vb^2]}_\infty
            ) (1 \lor \tfrac1{M_2} \norm{\E[\vz_{2}^2]}_1)
            + \tfrac2{M_2} \norm{\E[\mW]}\ss{F}^2 \norm{\E[\phi_2(\vz_{2})\phi_2(\vz_{2})^\T]}_2.
    \end{align}
    The result follows by applying \cref{lem:parameter_bound} and \cref{prop:optimised_param_bound}.
    Similarly, using $\norm{\E[\vz_{2}^2]}_1 \le M_2 \norm{\E[\vz_{2}^2]}_\infty$, we can again plug into \cref{eqn:z1_infty} to obtain
    \begin{align}
        \norm{\E[\vz_1^2]}_\infty 
        &\le 2 \norm{\V\ss{d}[\operatorname{vec}(\mW)]}_\infty \norm{\E[\vz_{2}^2]}_\infty
        + \tfrac2{M_2} \norm{\E[\mW]}\ss{F}^2 \norm{\E[\phi_2(\vz_{2})\phi_2(\vz_{2})^\T]}_2
        + 2 \norm{\E[\vb^2]}_\infty \\
        &\le
            2(
            \norm{\V\ss{d}[\operatorname{vec}(\mW)]}_\infty
                + \norm{\E[\vb^2]}_\infty
            ) (1 \lor \norm{\E[\vz_{2}^2]}_\infty)
            + \tfrac2{M_2} \norm{\E[\mW]}\ss{F}^2 \norm{\E[\phi_2(\vz_{2})\phi_2(\vz_{2})^\T]}_2.
    \end{align}
\end{proof}

\begin{remark} \label{rem:inf_offset}
    Like in \cref{rem:outer_offset}, apply \cref{lem:inf} to the triple $(\mW, \vz_2, \tfrac{\alpha}{\sqrt{M_2}} \mW \vone + \vb)$ where $\mW$ and $\vb$, not $\vz_2$, are element-wise independent and symmetric around their means.
    To write down the result, we need to estimate $\norm{\E[(\tfrac{\alpha}{\sqrt{M_2}} \mW \vone + \vb)^2]}_\infty$.
    To begin with, note that
    \begin{align}
        (\tfrac{\alpha}{\sqrt{M_2}} \mW \vone + \vb)^2
        &= (\tfrac{\alpha}{\sqrt{M_2}} (\mW - \E[\mW]) \vone + \tfrac{\alpha}{\sqrt{M_2}} \E[\mW] \vone + \vb)^2 \\
        &\preceq 3\tfrac{\alpha^2}{M_2} ((\mW - \E[\mW]) \vone)^2 + 3\tfrac{\alpha^2}{M_2} (\E[\mW] \vone)^2 + 3 \vb^2 \\
        &\preceq 3\tfrac{\alpha^2}{M_2} ((\mW - \E[\mW]) \vone)^2 + 3\tfrac{\alpha^2}{M_2} \norm{\E[\mW]}_\infty^2 \norm{\vone}_\infty^2 \vone + 3 \vb^2 \\
        &\overset{\smash{\text{(i)}}}{\preceq} 3\tfrac{\alpha^2}{M_2} ((\mW - \E[\mW]) \vone)^2 + 3\alpha^2 \norm{\E[\mW]}\ss{F}^2 \vone + 3 \vb^2
    \end{align}
    where the squares and inequalities are element-wise and in (i) we use that $\norm{\E[\mW]}_\infty \le \sqrt{M_2} \norm{\E[\mW]}\ss{F}$.
    For the first term, consider the $m$\textsuperscript{th} element and use independence of the elements of $\mW$:
    \begin{align}
        \E[\tfrac{1}{M_2}((\mW - \E[\mW]) \vone)^2]_m
        &= \textstyle
            \tfrac{1}{M_2} \E\big[
                \big(\sum_{m'=1}^{M_2}(W_{m,m'} - \E[W_{m,m'}])\big)^2
            \big] \\
        &= \textstyle
            \tfrac{1}{M_2}
                \sum_{m'=1}^{M_2} \E[(W_{m,m'} - \E[W_{m,m'}])^2] \\
        &\le \norm{\V\ss{d}[\operatorname{vec}(\mW)]}_{\infty}.
    \end{align}
    Therefore, 
    \begin{equation}
        \norm{\E[(\tfrac{\alpha}{\sqrt{M_2}} \mW \vone + \vb)^2]}_\infty
        \le 3 \alpha^2 \norm{\V\ss{d}[\operatorname{vec}(\mW)]}_{\infty} + 3 \alpha^2 \norm{\E[\mW]}\ss{F}^2 + 3 \norm{\E[\vb^2]}_\infty,
    \end{equation}
    so, by \cref{lem:outer} (specifically \cref{eqn:lem-outer-part1}),
    \begin{align}
        &\norm{\E[\vz_{1}^2]}_\infty \nonumber \\
        &\quad\le
            2 \norm{\V\ss{d}[\operatorname{vec}(\mW)]}_\infty 
            \norm{\E[\vz_2^2]}_\infty 
            + \tfrac2{M_2} \norm{\E[\mW]}\ss{F}^2 
            \norm{\E[\phi_2(\vz_{2})\phi_2(\vz_{2})^\T]}_2
            + 2\norm{\E[(\tfrac{\alpha}{\sqrt{M_2}} \mW \vone + \vb)^2]}_\infty, \\
        &\quad\le
            2 \norm{\V\ss{d}[\operatorname{vec}(\mW)]}_\infty 
            (3 \alpha^2 + \norm{\E[\vz_2^2]}_\infty) 
            + 2\norm{\E[\mW]}\ss{F}^2 
            (3\alpha^2 + \tfrac1{M_2} \norm{\E[\phi_2(\vz_{2})\phi_2(\vz_{2})^\T]}_2)
            + 6 \norm{\E[\vb^2]}_\infty \\
        &\quad\le
            6 (\sqrt{2} + K)^2(\alpha^2 + 1\lor \norm{\E[\vz_2^2]}_\infty)
            + 6K^2 
            (\alpha^2 + \tfrac1{M_2} \norm{\E[\phi_2(\vz_{2})\phi_2(\vz_{2})^\T]}_2).
    \end{align}
    Similarly, by \cref{eqn:lem-outer-part2}, we can write this inequality in terms of $\tfrac1{M_2}\norm{\E[\vz_2^2]}_1$ instead of $\norm{\E[\vz_2^2]}_\infty$:
    \begin{align}
        &\norm{\E[\vz_{1}^2]}_\infty
            \le
            6 (\sqrt{2} + K)^2(\alpha^2 + 1\lor \tfrac1{M_2}\norm{\E[\vz_2^2]}_1)
            + 6K^2 
            (\alpha^2 + \tfrac1{M_2} \norm{\E[\phi_2(\vz_{2})\phi_2(\vz_{2})^\T]}_2).
    \end{align}
    
\end{remark}

\begin{lemma}\label{lem:z-squared-recursion}
    For $\alpha=0$ and $1 \leq \ell \leq L$, we have
    \begin{equation}
        \norm{\E[\vz_\ell^2]}_\infty \leq (\gamma_{\alpha, M} (1 \lor K^2))^{\ell} (1\lor \tfrac1{D\ss{i}}\|\vx\|^2_2),
    \end{equation}
    with $\gamma_{\alpha, M}$ taking the value $\tfrac23(13+2\sqrt{43} \in (17,18)$ if $M\ge 36$ and $2(6+\sqrt{38}) \in (24,25)$ if $1\le M < 36$.
    On the other hand, for $\alpha \neq 0$ and $1 \leq \ell \le L$, we have
    \begin{equation}
        \alpha^2 + \norm{\E[\vz_\ell^2]}_\infty \leq (\gamma_{\alpha, M} (1 \lor K^2))^{\ell} (\alpha^2 + 1\lor \tfrac1{D\ss{i}} \|\vx\|^2_2),
    \end{equation}
    with $\gamma_{\alpha, M}$ taking the value $28 + \sqrt{793} \in (56,57)$ if $M\ge 36$ and $48 + \sqrt{2353} \in (96,97)$ if $1\le M < 36$.
\end{lemma}
\begin{proof}[Proof of case $\alpha = 0$]
    For notational convenience, define $a_\ell = 1 \lor  \norm{\E[\vz_\ell^2]}_\infty$ and $b_\ell = \norm{\E[\phi(\vz_{\ell})\phi(\vz_{\ell})^\T]}_2$.
    Also take $a_0 = 1\lor \|\vx\|^2_2$. 
    Apply \cref{lem:inf} with $\phi_2$ given by the identity function and $\vz_2 = \vx$.
    Then
    \begin{align}
        a_1 &\leq 2\left((\sqrt{2}+K)^2
                (1 \lor \tfrac1{D\ss{i}}\norm{\vx^2}_1) 
                + \tfrac1{D\ss{i}} K^2 \norm{\vx\vx^\T}_2 \right) \\
            &\le 2\left(8(1 \lor K)^2
                (1 \lor \tfrac1{D\ss{i}}\norm{\vx}_2^2) 
                + \tfrac1{D\ss{i}} K^2 \|\vx\|_2^2 \right) \\
            & \leq 18(1 \lor K^2)(1 \lor \tfrac1{D\ss{i}}\norm{\vx}^2_2) 
    \end{align}
    where the second inequality follows from $\sqrt{2} + K \le \sqrt{2}(1 + K) \le 2\sqrt{2}(1 \lor K)$ and $\norm{\vx\vx^\T}_2 = \norm{\vx}_2^2$.

    By \cref{lem:inf} we have
    \begin{align}
        a_\ell \leq  2\left((\sqrt{2}+K)^2
                a_{\ell-1} 
                + \tfrac1M K^2 b_{\ell-1} \right).
    \end{align}
    Further, by \cref{lem:outer}, we have
    \begin{align}
        b_{\ell-1} \leq 4\sqrt{M}(a_{\ell-1}
                    + K^2\, a_{\ell-2}) .
    \end{align}
    Combining these estimates yields
    \begin{align}
         a_\ell &\leq  2\left((\sqrt{2}+K)^2
                a_{\ell-1}
                + \tfrac{4}{\sqrt{M}} K^2 \left(a_{\ell-1}
                    + K^2 a_{\ell-2}\right. \right) \\
            & \leq       2\left(8(1\lor K^2)
                 a_{\ell-1}
                + \tfrac{4}{\sqrt{M}} K^2 \left(a_{\ell-1}
                    + K^2 a_{\ell-2}\right) \right) \\
             & \le
                (16+\tfrac8{\sqrt{M}}) (1 \lor K^2) a_{\ell-1}
                + \tfrac8{\sqrt{M}} (1 \lor K^4) a_{\ell-2}.
    \end{align}
    %
    Assuming $M \ge 36$, we further simplify
    \begin{align}
         a_\ell
         &\leq
            (16+\tfrac43) (1 \lor K^2) a_{\ell-1}
            + \tfrac43 (1 \lor K^4) a_{\ell-2}. \label{eqn:second_order_recursion}
    \end{align}
    We choose $M \ge 36$ for convenience when subsequently applying this lemma. We will return to the general case of $M \ge 1$ later. First, notice 
    \cref{eqn:second_order_recursion} is a homogeneous second-order linear recurrence relation.
    By finding the roots of the characteristic polynomial associated to this recurrence relation,  
    \begin{align}
        a_{\ell} \leq c_0\Big(\gamma (1 \lor K^2)\Big)^{\ell} 
    \end{align}
    for some $c_0$ and $\gamma = \tfrac23(13+2\sqrt{43} \in (17,18)$, which can be proved by induction. Checking initial conditions, we see that we can take $c_0 = 1\lor \|\vx\|^2 $, yielding the bound
    \begin{align}
          a_{\ell}
          &\leq \Big(\gamma(1 \lor K^2)\Big)^{\ell} (1\lor \tfrac1{D\ss{i}} \|\vx\|^2_2).
    \end{align}
    On the other hand, if $1 \le M < 36$, by the same reasoning the bound holds with constant $\gamma = 2(6+\sqrt{38}) \in (24,25)$.
\end{proof}
\begin{proof}[Proof of case $\alpha \neq 0$]
    The proof for the case $\alpha \neq 0$ proceeds like the case $\alpha = 0$, but uses \cref{rem:outer_offset,rem:inf_offset} instead of \cref{lem:outer,lem:inf}.
    Analogously define $a_\ell = \alpha^2 + 1 \lor \norm{\E[\vz_\ell^2]}_\infty$,
    $a_0 = \alpha^2 + 1 \lor \norm{\vx}^2_2$,
    and $b_\ell = \norm{\E[\phi(\vz_\ell)\phi(\vz_\ell)^\T]}_2$.
    To begin with, apply \cref{rem:inf_offset} with $\phi_2$ given by the identity function and $\vz_2 = \vx$:
    \begin{align}
        a_1
        &\le \alpha^2 + 6(\sqrt{2} + K)^2(\alpha^2 + 1 \lor \tfrac1{D\ss{i}}\norm{\vx^2}_1)+ 6K^2(\alpha^2 + \tfrac1{D\ss{i}} \norm{\vx \vx^\T}_2) \\
        &\le 48(1 \lor K^2)(\alpha^2 + 1 \lor \tfrac1{D\ss{i}}\norm{\vx}_2^2)+ 7K^2(\alpha^2 + \norm{\vx}^2_2) \\
        &\le 55(1 \lor K^2)(\alpha^2 + 1 \lor \tfrac1{D\ss{i}}\norm{\vx}^2_2).
    \end{align}
    For $1 < \ell \le L$,
    \begin{align}
        a_\ell
        &\le
            \alpha^2
            + 48(1 \lor K^2)a_{\ell - 1}
            + 6K^2(\alpha^2 + \tfrac1{M}b_{\ell-1}).
    \end{align}
    By \cref{rem:outer_offset},
    \begin{align}
        b_{\ell - 1} \le 8 \sqrt{M} (\norm{\E[\vz_{\ell-1}^2]}_\infty + K^2 a_{\ell - 2}).
    \end{align}
    Plugging the expression for $b_{\ell-1}$ into the expression for $a_\ell$,
    \begin{align}
        a_\ell
        &\le
            \alpha^2
            + 48(1 \lor K^2)a_{\ell - 1}
            + 6K^2(\alpha^2 + \tfrac{8}{\sqrt{M}}(\norm{\E[\vz_{\ell-1}^2]}_\infty + K^2 a_{\ell - 2})) \\
        &\le
            \alpha^2
            + 48(1 \lor K^2)a_{\ell - 1}
            + 6(1 \lor 8M^{-1/2})K^2(\alpha^2 + \norm{\E[\vz_{\ell-1}^2]}_\infty + K^2 a_{\ell - 2}) \\
        &\le
            \alpha^2
            + 48(1 \lor K^2)a_{\ell - 1}
            + 6(1 \lor 8M^{-1/2})K^2(a_{\ell - 1} + K^2 a_{\ell - 2}) \\
        &\le
            \alpha^2
            + (48 + 6(1 \lor 8M^{-1/2}))(1 \lor K^2)a_{\ell - 1}
            + 6(1 \lor 8M^{-1/2}) (1 \lor K^4) a_{\ell - 2} \\
        &\le
            (48 + 6(1 \lor 8M^{-1/2}) )(1 \lor K^2)a_{\ell - 1}
            + (1 + 6(1 \lor 8M^{-1/2})) (1 \lor K^4) a_{\ell - 2} .
    \end{align}
    As in the $\alpha=0$ case, we first consider the case $M\ge36$, giving
    \begin{equation}
        a_\ell
        \le
            56(1 \lor K^2)a_{\ell - 1}
            + 9 (1 \lor K^4) a_{\ell - 2} .
    \end{equation}
    
    This is again a homogeneous second-order linear recurrence relation.
    Solving for the roots of the characteristic polynomial, we find that
    \begin{equation}
        a_{\ell} \le c_0 \Big(\gamma' (1 \lor K^2)\Big)^\ell,
    \end{equation}
    for some $c_0$ and $\gamma' = 28 + \sqrt{793} \in (56,57)$.
    Comparing with the bound on $a_1$, we see that $c_0 = \alpha^2 + 1 \lor \tfrac1{D\ss{i}}\norm{\vx}_2^2$. On the other hand, if $1\le M < 36$, by the same reasoning the bound holds with $\gamma' = 48 + \sqrt{2353} \in (96,97)$.
\end{proof}

Finally, we end with a lemma which can be used to bound the covariance between two post-activations. The following lemmas in this section are only applied in \cref{app:convergence-in-dist} but we include them here because of their similarity. 

\begin{lemma} \label{lem:cov}
    Let $\phi_1, \phi_2 \colon\R \to \R$ be $1$-Lipschitz and odd and $\alpha \in \R$.
    Let the triple $(\mW, \vz_2, \vb) \in \R^{M_1\times M_2} \times \R^{M_2} \times \R^{M_1}$ be (possibly dependent) random variables such that, for all Rademacher vectors $\vep \in \set{-1, 1}^{M_1}$,
    \begin{equation}
        (\mW - \E[\mW], \vz_2, \vb - \E[\vb])
        \disteq
        (\diag(\vep)(\mW - \E[\mW]), \vz_2, \vep \had (\vb - \E[\vb])).
    \end{equation}
    Consider $\vz_1 = \tfrac1{\sqrt{M_2}}\mW (\phi_2 + \alpha)(\vz_2) + \vb$.
    Let $m, m' \in [M_1]$, $m \neq m'$.
    Then
    \begin{equation}
        \abs{\E[\phi_1(z_{1,m}) \phi_1(z_{1,m'})]}
        \le
            5(
                \norm{\E[\vz_1^2]}_\infty^{1/2}
                + 
            )(\alpha^2 + 1 \lor \norm{ \E[\vz_2^2] }_\infty)^2(K \lor K^2).
    \end{equation}
\end{lemma}
\begin{proof}
    The proof of this lemma uses exactly the same approach as \cref{lem:outer}.
    Set
    \begin{align}
        \vz_1' &= \vz_1 - \tfrac1{\sqrt{M_2}}\E[\mW] (\phi_2 + \alpha)(\vz_2) - \E[\vb].
    \end{align}
    Consider
    \begin{align}
        \abs{\E[\phi_1(z_{1,m}) \phi_1(z_{1,m'})]}
        &\le
            \sqrt{\E[(\phi_1(z_{1,m}) - \phi_1(z'_{1,m}))^2] \E[\phi^2_1(z_{1,m'})]}
            + \abs{\E[\phi_1(z'_{1,m}) \phi_1(z_{1,m'})]} \\
        &\le \sqrt{ab} +\abs{\E[\phi_1(z'_{1,m}) \phi_1(z_{1,m'})]}
    \end{align}
    where we denote
    \begin{equation}
        a = \norm{\E[\vz^2_{1}]}_\infty,
        \quad
        b = \norm{\E[(\vz_{1} - \vz'_{1})^2]}_\infty.
    \end{equation}
    Similarly,
    \begin{align}
        \abs{\E[\phi_1(z'_{1,m}) \phi_1(z_{1,m'})]}
        &\le
            \sqrt{
                \E[\phi^2_1(z'_{1,m})]
                \E[(\phi_1(z_{1,m'}) - \phi_1(z'_{1,m'}))^2] 
            }
            + \abs{\E[\phi_1(z'_{1,m}) \phi_1(z'_{1,m'})]}.
    \end{align}
    Now bound
    \begin{equation}
        \phi_1^2(\vz'_{1})
        \preceq (\vz'_{1})^2
        \preceq 2 \vz^2_{1} + 2 (\vz_1 - \vz_1')^2,
    \end{equation}
    so
    \begin{align}
        \sqrt{
            \E[\phi^2_1(z'_{1,m})]
            \E[(\phi_1(z_{1,m'}) - \phi_1(z'_{1,m'}))^2] 
        }
        \le \sqrt{(2a + 2b) b}.
    \end{align}
    Therefore,
    \begin{equation}
        \abs{\E[\phi_1(z_{1,m}) \phi_1(z_{1,m'})]}
        \le \sqrt{ab} + \sqrt{(2a + 2b)b} + \cancel{\abs{\E[\phi_1(z'_{1,m}) \phi_1(z'_{1,m'})]}}
    \end{equation}
    where the expectation on the RHS is zero by the assumed symmetry condition and oddness of $\phi$.
    Breaking up the square root, we find that
    \begin{equation}
        \abs{\E[\phi_1(z_{1,m}) \phi_1(z_{1,m'})]}
        \le (\sqrt{2} + 1)\sqrt{ab} + \sqrt{2} b.
    \end{equation}
    We finally estimate $b$:
    \begin{equation}
        \E[(\vz_1 - \vz_1')^2]
        \preceq \norm{\E[\mW]}_2^2 \tfrac{1}{M_2}\E[ \norm{(\phi_2 + \alpha)^2(\vz_2)}_2 ] \vone + \norm{\E[\vb]}_2^2 \vone.
    \end{equation}
    Therefore
    \begin{align}
        b
        &\le \norm{\E[\mW]}\ss{F}^2 \norm{\E[(\phi_2 + \alpha)^2(\vz_2)]}_\infty + \norm{\E[\vb]}_2^2 \\
        &\le 2\norm{\E[\mW]}\ss{F}^2 (\alpha^2 + \norm{\E[\vz_2^2]}_\infty) + \norm{\E[\vb]}_2^2 \\
        &\le 2 K^2 (\alpha^2 + 1 \lor \norm{\E[\vz_2^2]}_\infty),
    \end{align}
    which implies that
    \begin{equation}
        \abs{\E[\phi_1(z_{1,m}) \phi_1(z_{1,m'})]}
        \le
            (
                (\sqrt{2}+1) \norm{\E[\vz_1^2]}_\infty^{1/2}
                + \sqrt{2}
            )(\sqrt{b} \lor b),
        \quad
        b = 2 K^2 (\alpha^2 + 1 \lor \norm{ \E[\vz_2^2] }_\infty).
    \end{equation}
    Simplify the estimate to get the result:
    \begin{equation}
        \abs{\E[\phi_1(z_{1,m}) \phi_1(z_{1,m'})]}
        \le
            5(
                \norm{\E[\vz_1^2]}_\infty^{1/2}
                + 
            )(\alpha^2 + 1 \lor \norm{ \E[\vz_2^2] }_\infty)^2(K \lor K^2).
    \end{equation}
\end{proof}

\begin{lemma} \label{lem:second_moment_conditional_prior}
Let $\phi$ be 1-Lipschitz odd and $\alpha\in \mathbb{R}$. Then, 
\begin{align}
    &\tfrac1M\tr \left( \E_{Q}[\vw_{L+1} \vw_{L+1}^\T] (
            \E_{Q}[\phi(\vz_L) \phi(\vz_L)^\T]
            - \E_{P}[\phi(\vz_L) \phi(\vz_L)^\T]
        ) \right)  \nonumber \\
    &\qquad \le
        \tfrac{56 + 32\sqrt{2}}{\sqrt{M}} \rho_{\alpha, M}^{L} L^{1/2} K (K^2 \lor 1)^{L+\tfrac12}\left(\alpha^2 + 1\lor \tfrac1{D\ss{i}}\|\vx\|^2_2\right),
\end{align}
\end{lemma}

\begin{proof}

For convenience, define 
\begin{equation}
\mA = \E_{Q}[\phi(\vz_L) \phi(\vz_L)^\T] - \E_{P}[\phi(\vz_L) \phi(\vz_L)^\T]
\end{equation}

Recall $\mSigma_Q = \E_{Q}[\vw_{L+1} \vw_{L+1}^\T] - \E_Q[\vw_{L+1}]\E_Q[\vw_{L+1}]^\T$. 
By subtracting and adding $\E_Q[\vw_{L+1}]\E_Q[\vw_{L+1}]^\T + \mI$, we have
\begin{align}
    \abs{\tr \left( \E_{Q}[\vw_{L+1} \vw_{L+1}^\T] \mA \right)}
    &=
        \abs{\tr \left( (\mSigma_Q - \mI + \E_Q[\vw_{L+1}]\E_Q[\vw_{L+1}]^\T + \mI) \mA \right)} \\
    &=
        \abs{\tr \left( (\mSigma_Q - \mI)\mA + \E_Q[\vw_{L+1}]\E_Q[\vw_{L+1}]^\T \mA + \mA \right)} \\
    &\le
        \abs{\tr \left( (\mSigma_Q - \mI)\mA \right)} 
        + \abs{\tr \left(\E_Q[\vw_{L+1}]\E_Q[\vw_{L+1}]^\T \mA\right)} 
        + \abs{\tr \left( \mA \right)}. \label{eqn:tr_ww_A}
\end{align}
We deal with each of the three terms in \cref{eqn:tr_ww_A} separately. 

Consider the first term in \cref{eqn:tr_ww_A}, since $\mSigma_Q - \mI = \diag(\vsigma_Q^2 - \vone)$ is a diagonal matrix, by Cauchy-Schwarz
\begin{align}
    \abs{\tr \left( (\mSigma_Q - \mI)\mA \right)} 
    &= \abs{\langle \vsigma_Q^2 - \vone, \E_Q[\phi^2(\vz_L)] - \E_P[\phi^2(\vz_L)] \rangle} \\
    &\le \norm{\vsigma_Q^2 - \vone}_2 \norm{\E_Q[\phi^2(\vz_L)] - \E_P[\phi^2(\vz_L)]}_2 \\
    &\overset{\smash{\text{(i)}}}{\le} \left( (2+K)K \right) \left( 8 \sqrt{2} \rho_{\alpha, M}^{L-\tfrac12} L^{1/2} K (K^2 \lor 1)^{L-\tfrac12}\left(\alpha^2 + 1\lor \tfrac1{D\ss{i}}\|\vx\|^2_2\right) \right) \\
    &\overset{\smash{\text{(ii)}}}{\le} 32 \sqrt{2} \rho_{\alpha, M}^{L-\tfrac12} L^{1/2} K (K^2 \lor 1)^{L+\tfrac12}\left(\alpha^2 + 1\lor \tfrac1{D\ss{i}}\|\vx\|^2_2\right),
\end{align}
where in (i) we use \cref{lem:parameter_bound} and \cref{lem:diag-frob-norm} (proved later) while in (ii) we use $2+K \le 2(1+K) \le 4 (K\lor 1) = 4 (K^2\lor 1)^{1/2}$ and $K \le K \lor 1 = (K^2 \lor 1)^{1/2}$.

Next, consider the second term in \cref{eqn:tr_ww_A}, since $\tr(\vu \vv^T) = \vv^T \vu$ for $\vu, \vv \in \mathbb{R}^M$,
\begin{align}
    \abs{\tr \left(\E_Q[\vw_{L+1}]\E_Q[\vw_{L+1}]^\T \mA\right)} 
    &= \abs{\E_Q[\vw_{L+1}]^\T \mA \E_Q[\vw_{L+1}]} \\
    &\le \norm{\E_Q[\vw_{L+1}]}_2^2 \norm{\mA}_2 \\
    &\le K^2 \norm{\mA}_2,
\end{align}
where the first inequality follows from Cauchy-Schwarz and the definition of the operator norm while the second inequality follows from \cref{lem:parameter_bound}.
To bound $\norm{\mA}_2$, notice by \cref{lem:offdiag-operator-norm} (proved later), we have
\begin{align}
    \norm{\E_{Q}[\phi(\vz_L) \phi(\vz_L)^\T]}
    \leq 24\sqrt{M} \left(\gamma_{\alpha,M}(K^2 \lor 1)\right)^{L}\left(\alpha^2 + 1\lor \tfrac1{D\ss{i}}\|\vx\|^2_2\right)
\end{align}
and 
\begin{align}
    \norm{\E_{P}[\phi(\vz_L) \phi(\vz_L)^\T]}
    &\leq 24\sqrt{M} \left(\gamma_{\alpha,M}\right)^{L}\left(\alpha^2 + 1\lor \tfrac1{D\ss{i}}\|\vx\|^2_2\right) \\
    &\leq 24\sqrt{M} \left(\gamma_{\alpha,M}(K^2 \lor 1)\right)^{L}\left(\alpha^2 + 1\lor \tfrac1{D\ss{i}}\|\vx\|^2_2\right) 
\end{align}
Therefore, using the triangle inequality,
\begin{align}
    \norm{\mA}_2
    &\le \norm{\E_{Q}[\phi(\vz_L) \phi(\vz_L)^\T]}_2 + \norm{\E_{P}[\phi(\vz_L) \phi(\vz_L)^\T]}_2 \\
    &\le 48\sqrt{M} \left(\gamma_{\alpha,M}(K^2 \lor 1)\right)^{L}\left(\alpha^2 + 1\lor \tfrac1{D\ss{i}}\|\vx\|^2_2\right).
\end{align}
Putting in the bound on $\norm{A}_2$, a bound on the second term in \cref{eqn:tr_ww_A} is
\begin{align}
    \abs{\tr \left(\E_Q[\vw_{L+1}]\E_Q[\vw_{L+1}]^\T \mA\right)} 
    &\le 48\sqrt{M} K^2 \left(\gamma_{\alpha,M}(K^2 \lor 1)\right)^{L}\left(\alpha^2 + 1\lor \tfrac1{D\ss{i}}\|\vx\|^2_2\right) \\
    &\le 48\sqrt{M} \gamma_{\alpha,M}^L K (K^2 \lor 1)^{L+\tfrac12}\left(\alpha^2 + 1\lor \tfrac1{D\ss{i}}\|\vx\|^2_2\right),
\end{align}
where we use $K \le K \lor 1 = (K^2 \lor 1)^{1/2}$

Finally, consider the third term in \cref{eqn:tr_ww_A}. We have
\begin{align}
    \tr(\mA) 
    &= \left\lvert \mathrm{tr}\parens*{\E_{\Q}[\phi\ss{o}(\vz_{L})\phi\ss{o}(\vz_{L})^\T]-\E_{\P}[\phi\ss{o}(\vz_{L})\phi\ss{o}(\vz_{L})^\T]} \right\rvert \\
    &= \left\lvert \langle \vone, \E_{\Q}[\phi\ss{o}^2(\vz_{L}) ] - \E_{\P}[\phi\ss{o}^2(\vz_{L}) ] \rangle \right\rvert
    \\
    &\le \norm{\E_{\Q}[\phi\ss{o}^2(\vz_{L}) ] - \E_{\P}[\phi\ss{o}^2(\vz_{L}) ]}_1
    \\
    &\le \sqrt{M} \norm{\E_{\Q}[\phi\ss{o}^2(\vz_{L}) ] - \E_{\P}[\phi\ss{o}^2(\vz_{L}) ]}_2
    \\
    &\le 8 \sqrt{2} \sqrt{M} \rho_{\alpha, M}^{L-\tfrac12} L^{1/2} K (K^2 \lor 1)^{L-\tfrac12}\left(\alpha^2 + 1\lor \tfrac1{D\ss{i}}\|\vx\|^2_2\right)
\end{align}
where the last inequality follows from \cref{lem:diag-frob-norm}.

Here are the three bounds for reference:
\begin{align}
    & 32 \sqrt{2} \rho_{\alpha, M}^{L-\tfrac12} L^{1/2} K (K^2 \lor 1)^{L+\tfrac12}\left(\alpha^2 + 1\lor \tfrac1{D\ss{i}}\|\vx\|^2_2\right),
    \\
    & 48\sqrt{M} \gamma_{\alpha,M}^L K (K^2 \lor 1)^{L+\tfrac12}\left(\alpha^2 + 1\lor \tfrac1{D\ss{i}}\|\vx\|^2_2\right), \text{ and}
    \\
    & 8 \sqrt{2} \sqrt{M} \rho_{\alpha, M}^{L-\tfrac12} L^{1/2} K (K^2 \lor 1)^{L-\tfrac12}\left(\alpha^2 + 1\lor \tfrac1{D\ss{i}}\|\vx\|^2_2\right).
\end{align}

Finally, plugging into \cref{eqn:tr_ww_A}
\begin{align}
    \tfrac{1}{M}\abs{\tr \left( \E_{Q}[\vw_{L+1} \vw_{L+1}^\T] \mA \right)} 
    \le \tfrac{56 + 32\sqrt{2}}{\sqrt{M}} \rho_{\alpha, M}^{L} L^{1/2} K (K^2 \lor 1)^{L+\tfrac12}\left(\alpha^2 + 1\lor \tfrac1{D\ss{i}}\|\vx\|^2_2\right),
\end{align}
where we use $1\le \gamma_{\alpha, M} \le \rho_{\alpha, M}$, $L\ge 1$, and $M\ge 1$.
\end{proof}

\begin{lemma}\label{lem:inner_prod}
Let $\phi$ be 1-Lipschitz odd and $\alpha\in \mathbb{R}$. Then,
 \begin{equation}
    \tfrac1{M}\abs{\E[\lra{\vw, \alpha \vone} \lra{\vw, \phi(\vz_L)}]}
    \le 10 \tfrac{\abs{\alpha}}{\sqrt{M}} L(2 + \abs{\alpha} + \tfrac{1}{\sqrt{D\ss{i}}}\norm{\vx}_2)(2 + 2c)^{L - 1}K(K \lor 1)^{L + 1}.
\end{equation}
\end{lemma}
\begin{proof}
    For convenience let $\vw = \vw_{L+1}$.
    Note that
    \begin{align}
        \tfrac1{M}\abs{\E[\lra{\vw, \alpha \vone} \lra{\vw, \phi(\vz_L)}]}
        &= \tfrac{\abs{\alpha}}{M} \abs{\vone \E[\vw \vw^\T] \E[\phi(\vz_L)]} \\
        &\le \tfrac{\abs{\alpha}}{\sqrt{M}}\norm{\V[\vw] + \E[\vw]\E[\vw]^\T}_2 \norm{\E[\phi(\vz_L)]}_2 \\
        &\le \tfrac{\abs{\alpha}}{\sqrt{M}}
        (
            \norm{\V[\vw]}_2 + \norm{\E[\vw]\E[\vw]^\T}_2
        )
        \norm{\E[\phi(\vz_L)]}_2 \\
        &= \tfrac{\abs{\alpha}}{\sqrt{M}}
        (
            \norm{\V\ss{d}[\vw]}_\infty + \norm{\E[\vw]}_2^2
        )
        \norm{\E[\phi(\vz_L)]}_2.
    \end{align}
    Using \cref{lem:parameter_bound}, 
    \begin{equation}
        \norm{\V\ss{d}[\vw]}_\infty + \norm{\E[\vw]}_2^2
        \le \sigma\ss{max}^2 + K^2 \le (1 + K)^2 + K^2 \le 5 (K \lor 1)^2.
    \end{equation}
    Therefore, using \cref{lem:last-layer-bound},
    \begin{equation}
        \tfrac1{M}\abs{\E[\lra{\vw, \alpha \vone} \lra{\vw, \phi(\vz_L)}]}
        \le 10 \tfrac{\abs{\alpha}}{\sqrt{M}} L(2 + \abs{\alpha} + \tfrac{1}{\sqrt{D\ss{i}}}\norm{\vx}_2)(2 + 2c)^{L - 1}K(K \lor 1)^{L + 1}
    \end{equation}
\end{proof}

\subsection{Diagonal Terms}\label{app:variance-diagonal}

The main technical result needed for bounding the difference in the diagonal terms is the following bound on the Frobenius norm between the difference between these matrices.
\begin{lemma}\label{lem:diag-frob-norm}
    We have,
    \begin{align}
        \|\E_{\Q}[\phi\ss{o}(\vz_{L})^2 ]- \E_{\P}[\phi\ss{o}(\vz_{L})^2]\|_2
        &\le 8 \sqrt{2} \rho_{\alpha, M}^{L-\tfrac12} L^{1/2} K (K^2 \lor 1)^{L-\tfrac12}\left(\alpha^2 + 1\lor \tfrac1{D\ss{i}}\|\vx\|^2_2\right),
    \end{align}
    where $K = \sqrt{2 \KL(\Q, \P)}$ and 
    \begin{equation}
        \rho_{\alpha, M} =
        \begin{cases}
        12(2+\sqrt{2\pi}) \in (54,55) & \alpha=0, M<36 \\
        \tfrac23(13+2\sqrt{43}) \in (17,18) & \alpha=0, M\ge 36 \\
        28 + \sqrt{793} \in (56,57) & \alpha\neq 0, M<36 \\
        48 + \sqrt{2353} \in (96,97) & \alpha\neq 0, M\ge 36. \\
    \end{cases}
\end{equation}
\end{lemma}
\begin{proof}
We would like to combine the two expectations into a single expectation. To do this, we couple $\phi(\vz_{\ell})^2$ under $\P$ and $\Q$ by using the reparameterization trick and having them share the same noise. In particular, for $2 \leq \ell \le L$,  define
\begin{align}
     \vz^P_l&= \tfrac1{\sqrt{M}}\mathbfcal{E} \phi\ss{o}(\vz^P_{l-1}) + \tfrac{\alpha}{\sqrt{M}}\mathbfcal{E}\vone + \vep, \\
    \vz^Q_L
    &= \tfrac1{\sqrt{M}}(\mS^Q_l \circ \mathbfcal{E} + \mM^Q_l) \phi\ss{o}(\vz^Q_{l-1}) + \tfrac{\alpha}{\sqrt{M}}(\mS^Q_l \circ \mathbfcal{E} + \mM^Q_l)\vone + (\vs^Q_l \circ \vep + \vm^Q_l).
\end{align}
where $\mathbfcal{E}$ is a matrix of i.i.d.~standard Gaussian random variables, $\vep$ is a vector of i.i.d.~standard Gaussian random variables, $\mM^Q_L$ is a matrix consisting of the mean of each weight in layer $L$ under $\Q$, $\vm^Q_L$ is vector consisting of the mean of each bias in layer $L$ under $\Q$, $\mS^Q_L$ and $\vs^Q_L$ are a matrix and vector containing the standard deviations of each weight or bias respectively in layer $L$ under $\Q$. We then can rewrite,
\begin{align}
    \|\E_{\Q}[\phi\ss{o}(\vz_{L})^2 ]- \E_{\P}[\phi\ss{o}(\vz_{L})^2]\|_2 &= \norm{\E[\phi\ss{o}(\vz^Q_L)^2 - \phi\ss{o}(\vz^P_L)^2]}_2 
    \\
    & = \sqrt{\sum_{m=1}^M \left(\E[(\phi\ss{o}(z^Q_{m,L}) + \phi\ss{o}(z^P_{m, L})) (\phi\ss{o}(z^Q_{m, L}) - \phi\ss{o}(z^P_{m,L})))]\right)^2 } 
    \\
    & \leq \sqrt{\sum_{m=1}^M\E[(\phi\ss{o}(z^Q_{m,L}) + \phi\ss{o}(z^P_{m, L}))^2]\E[(\phi\ss{o}(z^Q_{m,L}) - \phi\ss{o}(z^P_{m, L}))^2]}
    \\
    & \leq \sqrt{\norm*{\E[(\phi\ss{o}(\vz^Q_L) + \phi\ss{o}(\vz^P_L))^2]}_\infty} \sqrt{\sum_{m=1}^M \E[(\phi\ss{o}(z^Q_{m,L}) - \phi\ss{o}(z^P_{m, L}))^2]}
    \\
    & = \sqrt{\norm{\E[(\phi\ss{o}(\vz^Q_L) + \phi\ss{o}(\vz^P_L))^2]}_\infty} \sqrt{ \E[ \norm{\phi\ss{o}(\vz^Q_L) - \phi\ss{o}(\vz^P_L)}^2_2] }
    \\
    &\le \sqrt{\norm{\E[(\phi\ss{o}(\vz^Q_L) + \phi\ss{o}(\vz^P_L))^2]}_\infty} \sqrt{ \E[ \norm{\vz^Q_L - \vz^P_L}^2_2] }
    \label{eqn:phi2Q_phi2P}
\end{align}
The first inequality is an element-wise application of Cauchy-Schwarz viewing the expectation as an inner product, the second inequality is a bound of the form $\sum |a_i b_i| \leq \sup |a_i| \sum |b_i|$, and the third inequality uses the Lipschitz property of $\phi\ss{o}$. We next bound the square of each of the two terms in \cref{eqn:phi2Q_phi2P}.

\paragraph{Bounding $\norm{\E[(\phi\ss{o}(\vz^Q_L) + \phi\ss{o}(\vz^P_L))^2]}_\infty$.} Using the inequality $(a_1+a_2)^2 \le 2(a_1^2 + a_2^2)$ and the triangle inequality we have
\begin{align}
    \norm{\E[(\phi\ss{o}(\vz^Q_L) + \phi\ss{o}(\vz^P_L))^2]}_\infty
    &\le \norm{\E[2\phi\ss{o}^2(\vz^Q_L) + 2\phi\ss{o}^2(\vz^P_L)]}_\infty \\
    &\le 2\norm{\E[\phi\ss{o}^2(\vz^Q_L)]}_\infty + 2\norm{\E[\phi\ss{o}^2(\vz^P_L)]}_\infty \\
    &\le 2\norm{\E[(\vz^Q_L)^2]}_\infty + 2\norm{\E[(\vz^P_L)^2]}_\infty,
\end{align}
where the last inequality follows from the oddness and Lipschitz property of $\phi\ss{o}$:
\begin{equation}
    \norm{\E[\phi\ss{o}^2(\vz^Q_L)]}_\infty = \norm{\E[(\phi\ss{o}(\vz^Q_L) - \phi\ss{o}(\mathbf{0}))^2]}_\infty
    \le \norm{\E[(\vz^Q_L)^2]}_\infty.
\end{equation}
Likewise, $\norm{\E[\phi^2(\vz^P_L)]}_\infty \le \norm{\E[(\vz^P_L)^2]}_\infty$. We can upper  by applying  \cref{lem:z-squared-recursion} to each of the two terms:
\begin{align}
    \norm{\E_{\Q}[(\vz_L)^2]}_\infty \leq \left(\gamma_{\alpha,M}(K^2 \lor 1)\right)^{L}\left(\alpha^2 + 1\lor \tfrac1{D\ss{i}}\|\vx\|^2_2\right) - \alpha^2
\end{align} 
and
\begin{align}
    \norm{\E_{\P}[(\vz_L)^2]}_\infty \leq \gamma_{\alpha,M}^{L}\left(\alpha^2 + 1\lor \tfrac1{D\ss{i}}\|\vx\|^2_2\right) -\alpha^2.
\end{align}
Therefore, the first term in \cref{eqn:phi2Q_phi2P} is upper bounded by
\begin{align}
    \norm{\E[(\phi\ss{o}(\vz^Q_L) + \phi\ss{o}(\vz^P_L))^2]}_\infty 
    &\le 2 \gamma_{\alpha,M}^L ((K^2 \lor 1)^L + 1 ) (\alpha^2 + 1\lor \tfrac1{D\ss{i}}\|\vx\|^2_2 ) - 4\alpha^2 \\
    &\le 4 \gamma_{\alpha,M}^L (K^2 \lor 1)^L (\alpha^2 + 1\lor \tfrac1{D\ss{i}}\|\vx\|^2_2 ), \label{eqn:diag-frob-norm-part1}
\end{align}
where we use that $(K^2 \lor 1)^L + 1 \le 2(K^2 \lor 1)$ and drop the $-4\alpha^2$ term for simplicity later on.

\paragraph{Bounding $\E[ \norm{\vz^Q_L - \vz^P_L}^2_2]$}

It remains to upper bound the second term in \cref{eqn:phi2Q_phi2P}. To do so, we will setup a linear, non-homogenous recursion relation. To start, notice for $2 \leq \ell \le L$, by adding and subtracting $\mathbfcal{E}\phi\ss{o}(\vz^Q_{\ell-1})$ and then applying the triangle inequality, we have
\begin{equation}
    \begin{split}
    \norm{\vz^Q_\ell - \vz^P_\ell}_2 
    \le & \tfrac1{\sqrt{M}}\norm{\mS^Q_\ell \had \mathbfcal{E} + \mM^Q_\ell - \mathbfcal{E}}_2 \norm{\phi\ss{o}(\vz^Q_{\ell-1})}_2 
    + \alpha\norm{\mS^Q_\ell \had \mathbfcal{E} + \mM^Q_\ell - \mathbfcal{E}}_2 \\
    & + \tfrac1{\sqrt{M}}\norm{ \mathbfcal{E}}_2 \norm{\phi\ss{o}( \vz^Q_{\ell-1}) - \phi\ss{o}(\vz^P_{\ell-1})}_2 
    + \norm{\vs^Q_\ell \had \vep + \vm^Q_\ell - \vep}_2, \label{eqn:zq_minus_zp} 
    \end{split}
\end{equation}
where we note that $\norm{\vone}_2=\sqrt{M}$ cancels the $\tfrac{1}{\sqrt{M}}$ term in the second term. 
To obtain a tighter bound, we consider the cases of $\alpha=0$ and $\alpha\neq0$ separately by defining $c_\alpha = 3$ if $\alpha=0$ and $c_\alpha = 4$ if $\alpha\neq0$. Then, for any $a_1, a_2, a_3, a_4 \in \mathbb{R}$, we have $(a_1 + a_2 + a_3 + \alpha a_4)^2 \le c_\alpha(a_1^2 + a_2^2 + a_3^2 + \alpha^2 a_4^2)$. Applying this expression to \cref{eqn:zq_minus_zp}, squaring both sides, and then taking the expectation we have
\begin{equation}
    \begin{split}
    \E[\norm{\vz^Q_\ell - \vz^P_\ell}_2^2]
    \le & \tfrac{c_\alpha}{M}\E[\norm{\mS^Q_\ell \had \mathbfcal{E} + \mM^Q_\ell - \mathbfcal{E}}_2^2] \E[\norm{\phi\ss{o}(\vz^Q_{\ell-1})}_2^2]
    + c_\alpha \alpha^2\E[\norm{\mS^Q_\ell \had \mathbfcal{E} + \mM^Q_\ell - \mathbfcal{E}}_2^2] \\
    & + \tfrac{c_\alpha}{M}\E[\norm{ \mathbfcal{E}}_2^2] \E[\norm{\phi\ss{o}( \vz^Q_{\ell-1}) - \phi\ss{o}(\vz^P_{\ell-1})}_2^2] 
    + c_\alpha \E[\norm{\vs^Q_\ell \had \vep + \vm^Q_\ell - \vep}_2^2].\label{eqn:E_zq_minus_zp} 
    \end{split}
\end{equation}
We now turn to bounding each of the terms in \cref{eqn:E_zq_minus_zp}. First, by \cref{lem:expectation_squared_bound_update},
\begin{equation}
\tfrac{1}{M}\E[\norm{\mathbfcal{E}}_2^2] \le \eta_M.
\end{equation}
Next, we have
\begin{align}
\norm{\mS^Q_\ell \had \mathbfcal{E} + \mM^Q_\ell - \mathbfcal{E}}_2 &\leq \norm{\mS^Q_\ell \had \mathbfcal{E} + \mM^Q_\ell - \mathbfcal{E}}\ss{F} \\
&\leq \norm{\mS^Q_\ell \had \mathbfcal{E}- \mathbfcal{E}}\ss{F} + \norm{\mM^Q_\ell}\ss{F} \\
\implies \E[\norm{\mS^Q_\ell \had \mathbfcal{E} + \mM^Q_\ell - \mathbfcal{E}}_2^2] &\leq 2\E[\norm{\mS^Q_\ell \had \mathbfcal{E}- \mathbfcal{E}}\ss{F}^2 + 2\E[\norm{\mM^Q_\ell}\ss{F}^2].
\end{align}
We can then bound each term by $\KL(\Q,\P)$. To do this, first notice we can write
\begin{align}
\E[\norm{\mS^Q_\ell \had \mathbfcal{E}- \mathbfcal{E}}\ss{F}^2] = \E\left[\sum_m \sum_{m'} (\vsigma^Q_\ell-1)^2\epsilon_{m,m'}^2\right] = \sum_m \sum_{m'} (\vsigma^Q_\ell-1)^2\E\left[\epsilon_{m,m'}^2\right] = \sum_m \sum_{m'} (\vsigma^Q_\ell-1)^2. \label{eqn:mat-norm-bound}
\end{align}
By an application of \cref{lem:parameter_bound}, we then have,
\begin{align}
\E[\norm{\mS^Q_\ell \had \mathbfcal{E}- \mathbfcal{E}}\ss{F}^2] \leq K^2 \;\; \text{and} \;\; \E[\norm{ \mM^Q_\ell }\ss{F}^2] \leq K^2, \label{eqn:first-terms-bound}
\end{align}
where, as before, $K = \sqrt{2\KL(\Q,\P)}$. Therefore, we obtain
\begin{align*}
    \E[\norm{\mS^Q_\ell \had \mathbfcal{E} + \mM^Q_\ell - \mathbfcal{E}}_2^2] \leq  4K^2.
\end{align*}
We can apply an identical argument to $\norm{\vs^Q_\ell \had \vep + \vm^Q_\ell - \vep}_2$ to conclude that 
\begin{align}
\E[\norm{\vs^Q_\ell \had \vep + \vm^Q_\ell - \vep}_2^2] \le  4K^2.
\end{align}
Finally, we have that
\begin{align}
    \E[\norm{\phi\ss{o}(\vz^Q_{\ell-1})}_2^2] 
    = \E\left[\sum_{m=1}^M \phi\ss{o}^2(z^Q_{\ell-1,m}) \right]
    = \sum_{m=1}^M \E[\phi\ss{o}^2(z^Q_{\ell-1,m}) ]
    \le M \norm{\E[\phi\ss{o}^2(\vz^Q_{\ell-1})]}_\infty 
    \le M \norm{\E[(\vz^Q_{\ell-1})^2]}_\infty,
\end{align}
where the last inequality follows from the oddness and Lipschitzness of $\phi\ss{o}$. Dividing by $M$ and applying \cref{lem:z-squared-recursion} to the last expression, we can conclude that
\begin{align}
    \tfrac1{M}\E[\norm{\phi\ss{o}(\vz^Q_{\ell-1})}_2^2] \le \left(\gamma_{\alpha,M}(K^2 \lor 1)\right)^{\ell - 1}\left(\alpha^2 + 1\lor \tfrac1{D\ss{i}}\|\vx\|^2_2\right) - \alpha^2.
\end{align}
We now plug these expressions into \cref{eqn:zq_minus_zp}, noting the cancellation of $\alpha^2$ terms, to obtain
\begin{equation}
    \E[\norm{\vz^Q_\ell - \vz^P_\ell}_2^2] \le 
    4c_\alpha K^2 \left( \left(\gamma_{\alpha,M}(K^2 \lor 1)\right)^{\ell - 1}\left(\alpha^2 + 1\lor \tfrac1{D\ss{i}}\|\vx\|^2_2\right) + 1 \right) + c_\alpha \eta_M \, \E[\norm{\vz^Q_{\ell-1} - \vz^P_{\ell-1}}_2^2].\label{eqn:diagonal-recursion}
\end{equation}

We now set up the recursion, with base case
\begin{align}
    \E[\norm{\vz^Q_1 - \vz^P_1}_2^2] &\leq  4c_\alpha K^2 \left(\left(\alpha^2 + 1\lor \tfrac1{D\ss{i}}\|\vx\|^2_2\right)+1\right). \label{eqn:diagonal-base}
\end{align}
Taken together, \cref{eqn:diagonal-recursion} and \cref{eqn:diagonal-base} define a linear, non-homogeneous recursion in $\E[\norm{\vz^Q_l - \vz^P_l}_2^2]$ with variable coefficients.
By unrolling the recursion, we find that\footnote{
    Suppose that $x_n \le b_n + a x_{n-1}$ for $n \ge 2$ and $x_1 \le b_1$.
    We then show that $x_n \le \sum_{n'=1}^n a^{n - n'} b_n$ for all $n \ge 1$.
    For the base case, note that the case $n=1$ is true because $x_1 \le b_1$
    For the induction step, suppose that the case $n \in \N$ is true.
    Then $
        x_{n+1}
        \le b_{n + 1} + a \sum_{n'=1}^n a^{n - n'} b_n
        \le a^{n + 1 - (n+1)} b_{n + 1} + \sum_{n'=1}^n a^{n + 1 - n'} b_n
        =\sum_{n'=1}^{n + 1} a^{n + 1 - n'} b_n
    $, so the case $n + 1$ is also true.
    We conclude that $x_n \le \sum_{n'=1}^n a^{n - n'} b_n$ for all $n \ge 1.$
}
\begin{align}
    \E[\norm{\vz^Q_{L} - \vz^P_{L}}_2^2] 
    &\le \sum_{\ell=1}^{L} (c_\alpha \eta_M)^{L - \ell} \left(4 c_\alpha K^2 \left(\left(\gamma_{\alpha,M}(K^2 \lor 1)\right)^{\ell - 1}\left(\alpha^2 + 1\lor \tfrac1{D\ss{i}}\|\vx\|^2_2\right)+1\right)\right) \\
    &\le (c_\alpha \eta_M \lor \gamma_{\alpha,M})^{L - 1} 4 c_\alpha K^2 \sum_{\ell=1}^{L} (K^2 \lor 1)^{\ell - 1}\left( \alpha^2 + 1\lor \tfrac1{D\ss{i}}\|\vx\|^2_2+1\right),
\end{align}
where in the second inequality we factor out $(\gamma_{\alpha,M}(K^2 \lor 1))^{\ell - 1}$ (which is possible since it is greater than 1) and then combine $(4c)^{L-\ell} \gamma_{\alpha,M}^{\ell -1} \le (4c \lor \gamma_{\alpha,M})^{L - 1}$. We can bound the sum by $L$ times the largest term in the sum, which is the $l=L$ term since $K^2 \lor 1 \ge 1$:
\begin{align}
    \E[\norm{\vz^Q_{L} - \vz^P_{L}}_2^2] 
    &\le L \left[ (c_\alpha \eta_M \lor \gamma_{\alpha,M})^{L - 1} 4 c_\alpha K^2 (K^2 \lor 1)^{L - 1} \left(\alpha^2 + 1\lor \tfrac1{D\ss{i}}\|\vx\|^2_2+1\right) \right] \\
    &\le 32 (c_\alpha \eta_M \lor \gamma_{\alpha,M})^{L - 1} L K^2 (K^2 \lor 1)^{L - 1} \left(\alpha^2 + 1\lor \tfrac1{D\ss{i}}\|\vx\|^2_2\right)
\end{align}
where the second inequality uses that $1\lor \tfrac1{D\ss{i}}\|\vx\|^2_2 + 1 \le 2 (1\lor \tfrac1{D\ss{i}}\|\vx\|^2_2)$ and $c_\alpha \le 4$. 
Define $\rho_{\alpha,M} = c_\alpha \eta_M \lor \gamma_{\alpha, M}$. Recall from \cref{lem:expectation_squared_bound_update} that $\eta_{M} \in (5, 19)$ depends on if $M \ge 36$. Recall also from \cref{lem:z-squared-recursion} that $\gamma_{\alpha, M} \in (17, 97)$ depends if $M \ge 36$ and if $\alpha=0$. Comparing cases, we see
\begin{equation}
    \rho_{\alpha, M} =
    \begin{cases}
    12(2+\sqrt{2\pi}) \in (54,55) & \alpha=0, M<36 \\
    \tfrac23(13+2\sqrt{43}) \in (17,18) & \alpha=0, M\ge 36 \\
    28 + \sqrt{793} \in (56,57) & \alpha\neq 0, M<36 \\
    48 + \sqrt{2353} \in (96,97) & \alpha\neq 0, M\ge 36. \\
    \end{cases}
\end{equation}
Notice $\rho_{\alpha, M}$ is equal to $\gamma_{\alpha, M}$ except in the case of $\alpha=0$ and $M<36$.
\begin{equation}
    \E[\norm{\vz^Q_{L} - \vz^P_{L}}_2^2] 
    \le 32 \rho_{\alpha, M}^{L - 1} L K^2 (K^2 \lor 1)^{L - 1} \left(\alpha^2 + 1\lor \tfrac1{D\ss{i}}\|\vx\|^2_2\right). \label{eqn:diag-frob-norm-part2}
\end{equation}

\paragraph{Combining terms.}
Plugging Equations (\ref{eqn:diag-frob-norm-part1}) and (\ref{eqn:diag-frob-norm-part2}) into \cref{eqn:phi2Q_phi2P} and simplifying we have
\begin{align}
    \|\E_{\Q}[\phi\ss{o}(\vz_{L})^2 ]- \E_{\P}[\phi\ss{o}(\vz_{L})^2]\|_2
    &\le 8 \sqrt{2} \rho_{\alpha, M}^{L-\tfrac12} L^{1/2} K (K^2 \lor 1)^{L-\tfrac12}\left(\alpha^2 + 1\lor \tfrac1{D\ss{i}}\|\vx\|^2_2\right).
\end{align}

\end{proof}

Given this lemma, we can succinctly prove \cref{lem:app-diagonal-variance}.
\begin{proof}[Proof of \cref{lem:app-diagonal-variance}]
We have
\begin{align}
    |D^{i}_{\Q}- D^{i}_{\P}| &= \frac{1}{M}\left\lvert\mathrm{tr}\parens*{
        \parens*{\mSigma_Q- \mI}\E_{\Q}[\phi\ss{o}(\vz_{L})\phi\ss{o}(\vz_{L})^\T]
            + (\E_{\Q}[\phi\ss{o}(\vz_{L})\phi\ss{o}(\vz_{L})^\T]-\E_{\P}[\phi\ss{o}(\vz_{L})\phi\ss{o}(\vz_{L})^\T])
    }
    \right\rvert
    \\
    &\le \tfrac{1}{M}\left\lvert\mathrm{tr}\parens*{
        \parens*{\mSigma_Q- \mI}\E_{\Q}[\phi\ss{o}(\vz_{L})\phi\ss{o}(\vz_{L})^\T]
        } \right\rvert
        + \tfrac{1}{M}\left\lvert\mathrm{tr}\parens*{\E_{\Q}[\phi(\vz_{L})\phi\ss{o}(\vz_{L})^\T]-\E_{\P}[\phi\ss{o}(\vz_{L})\phi\ss{o}(\vz_{L})^\T]}\right\rvert \label{eq:diag_trace_term}
\end{align}
Next we bound each of the two terms in \cref{eq:diag_trace_term} separately, starting with the first term. First, note that $\mSigma_Q - \mI$ is a diagonal matrix and let $\vsigma^2 = \diag(\mSigma_Q)$. Since the trace of a matrix product is the element-wise inner product of the matrices, the first term in \cref{eq:diag_trace_term} becomes
\begin{align}
    \left\lvert\mathrm{tr}\parens*{
        \parens*{\mSigma_Q- \mI}\E_{\Q}[\phi\ss{o}(\vz_{L})\phi\ss{o}(\vz_{L})^\T]} \right\rvert
    &= \left\lvert\langle \vsigma_Q^2 - \vone, \E_{\Q}[\phi\ss{o}^2(\vz_{L})]]\rangle \right\rvert
    \\
    &\leq \|\vsigma_Q^2 - \vone\|_1\|\E_{\Q}[\phi\ss{o}^2(\vz_{L}) ] \|_\infty 
    \\
    &\leq \sqrt{M} \|\vsigma_Q^2 - \vone\|_2 \|\E_{\Q}[\phi\ss{o}^2(\vz_{L})]\|_\infty
\end{align}
where the first inequality is an application of H\"{o}lder's inequality. Now $\|\vsigma_Q^2 -\vone\|_{2} \leq \left (2 + K\right)K$ by \cref{lem:parameter_bound} (iv.), where $K = \sqrt{2\KL(\Q,\P)}$. To upper bound $\norm{\E_Q[\phi^2(\vz_L)]}_\infty$, notice we can use the oddness and Lipschitz property of $\phi\ss{o}$ to write
    \begin{equation}
        \norm{\E[\phi\ss{o}^2(\vz^Q_L)]}_\infty = \norm{\E[(\phi\ss{o}(\vz^Q_L) - \phi\ss{o}(\mathbf{0}))^2]}_\infty
        \le \norm{\E[(\vz^Q_L)^2]}_\infty.
    \end{equation}
    We can upper bound this using \cref{lem:z-squared-recursion}:
    \begin{align}
        \norm{\E_{\Q}[(\vz_L)^2]}_\infty \leq \left(\gamma_\alpha(K^2 \lor 1)\right)^{L}\left(\alpha^2 + 1\lor \tfrac1{D\ss{i}}\|\vx\|^2_2\right).
    \end{align} 
Therefore, the first term in \cref{eq:diag_trace_term} is upper bounded by
\begin{align}
    \left\lvert\mathrm{tr}\parens*{
        \parens*{\mSigma_Q- \mI}\E_{\Q}[\phi\ss{o}(\vz_{L})\phi\ss{o}(\vz_{L})^\T]} \right\rvert
    &\le \sqrt{M}
        \left(2 + K\right)K
        \left(\gamma_\alpha(K^2 \lor 1)\right)^{L}\left(\alpha^2 + 1\lor \tfrac1{D\ss{i}}\|\vx\|^2_2\right) \\
    &\le \sqrt{M}
        \gamma_\alpha^{L} K \left(K^2 \lor 1\right)^{L+\tfrac12}\left(\alpha^2 + 1\lor \tfrac1{D\ss{i}}\|\vx\|^2_2\right),
\end{align}
where we use that $K + 2 \le 3 (K^2 \lor 1)^{1/2}$ in the second inequality. 

The second term in \cref{eq:diag_trace_term} can be upper bounded by
\begin{align}
    \left\lvert \mathrm{tr}\parens*{\E_{\Q}[\phi\ss{o}(\vz_{L})\phi\ss{o}(\vz_{L})^\T]-\E_{\P}[\phi\ss{o}(\vz_{L})\phi\ss{o}(\vz_{L})^\T]} \right\rvert 
    &= \left\lvert \langle \vone, \E_{\Q}[\phi\ss{o}^2(\vz_{L}) ] - \E_{\P}[\phi\ss{o}^2(\vz_{L}) ] \rangle \right\rvert
    \\
    &\le \norm{\E_{\Q}[\phi\ss{o}^2(\vz_{L}) ] - \E_{\P}[\phi\ss{o}^2(\vz_{L}) ]}_1
    \\
    &\le \sqrt{M} \norm{\E_{\Q}[\phi\ss{o}^2(\vz_{L}) ] - \E_{\P}[\phi\ss{o}^2(\vz_{L}) ]}_2
    \\
    &\le \sqrt{M} 8 \sqrt{2} \rho_{\alpha, M}^{L-\tfrac12} L^{1/2} K (K^2 \lor 1)^{L-\tfrac12}\left(\alpha^2 + 1\lor \tfrac1{D\ss{i}}\|\vx\|^2_2\right)
\end{align}
where the last inequality follows from \cref{lem:diag-frob-norm}. 

Plugging the bounds on both terms into \cref{eq:diag_trace_term} and simplifying we have
\begin{align}
    |D^{i}_{\Q}(\vx) - D^{i}_{\P}(\vx)| 
    &\leq  
    \tfrac{1+8\sqrt{2}}{\sqrt{M}} 
    \rho_{\alpha, M}^{L} L^{1/2} K (K^2 \lor 1)^{L+\tfrac12}\left(\alpha^2 + 1\lor \tfrac1{D\ss{i}}\|\vx\|^2_2\right) \\
    &\leq  
    \tfrac{16 + \sqrt{2}}{\sqrt{M}} 
    \rho_{\alpha, M}^{L} L^{1/2} \KL(\Q,P)^{1/2} (2 \KL(\Q,\P) \lor 1)^{L+\tfrac12}\left(\alpha^2 + 1\lor \tfrac1{D\ss{i}}\|\vx\|^2_2\right)
\end{align}
where we use $1\le L^{1/2}$, $\rho_{\alpha, M}^{L-\tfrac12}\le \rho_{\alpha, M}^L$, $\gamma_{\alpha,M} \le \rho_{\alpha,M}$,
and $(K^2 \lor 1)^{L-\tfrac12} \le (K^2 \lor 1)^{L+\tfrac12}$ to more simply group the terms. 

\end{proof}

\subsection{Off-Diagonal Terms}\label{app:variance-off-diagonal}

The key result in bounding the off-diagonal is an upper bound on the operator norm of the last layer outer product.
\begin{lemma}\label{lem:offdiag-operator-norm}
\begin{equation}
    \  \|\E_{\Q}[\phi\ss{o}(\vz_{L}(\vx))\phi\ss{o}(\vz_{L}(\vx))^\T]\|_2 
    \leq 24\sqrt{M} \left(\gamma_{\alpha,M}(K^2 \lor 1)\right)^{L}\left(\alpha^2 + 1\lor \tfrac1{D\ss{i}}\|\vx\|^2_2\right),
\end{equation}
where $K=\sqrt{\KL(\Q,\P)}$ and $\gamma_{\alpha,M} \in (17,97)$ is defined in \cref{lem:z-squared-recursion}.
\end{lemma}
\begin{proof}
By \cref{rem:outer_offset}, 
\begin{align}
      \  \|\E_{\Q}[\phi\ss{o}(\vz_{L}(\vx))\phi\ss{o}(\vz_{L}(\vx))^\T]\|_2 &\leq  12\sqrt{M}\left(
                \norm{\E[\vz_L^2]}_\infty
                + K^2 \left((\alpha^2 + \norm{\E[\vz_{L-1}^2]}_\infty) \lor 1 \right) \right). \label{eqn:apply_outer_offset}
\end{align}
Note that in the $\alpha=0$ case we could save a factor of 2 by instead applying \cref{lem:outer}, but for simplicity we consider any $\alpha\in\mathbb{R}$ here. We could also reduce the constant by separately considering the case of $M \ge 3$ but we do not for simplicity. 
We now upper bound each of the two terms above. By \cref{lem:z-squared-recursion},
\begin{equation}
    \alpha^2 + \norm{\E[(\vz_L)^2]}_\infty \leq \left(\gamma_{\alpha,M}(K^2 \lor 1)\right)^{L}\left(\alpha^2 + 1\lor \tfrac1{D\ss{i}}\|\vx\|^2_2\right). \label{eqn:z2_L}
\end{equation}
Similarly,
\begin{align}
    K^2 \left((\alpha^2 + \norm{\E[\vz_{L-1}^2]}_\infty) \lor 1 \right)
    &\leq K^2 \left(\left(\gamma_{\alpha,M}(K^2 \lor 1)\right)^{L-1}\left(\alpha^2 + 1\lor \tfrac1{D\ss{i}}\|\vx\|^2_2\right) \lor 1\right) \\
    &\leq (K^2 \lor 1) \left(\gamma_{\alpha,M}(K^2 \lor 1)\right)^{L-1}\left(\alpha^2 + 1\lor \tfrac1{D\ss{i}}\|\vx\|^2_2\right) \\
    &\leq \left(\gamma_{\alpha,M}(K^2 \lor 1)\right)^{L}\left(\alpha^2 + 1\lor \tfrac1{D\ss{i}}\|\vx\|^2_2\right),
\end{align}
where we use that $\gamma_{\alpha,M}\ge1$ in the last two inequalities. Adding the two terms together in \cref{eqn:apply_outer_offset}, noting that we drop the $\alpha^2$ term in the \cref{eqn:z2_L} to more simply group the terms, we have
\begin{equation}
    \\|\E_{\Q}[\phi\ss{o}(\vz_{L}(\vx))\phi\ss{o}(\vz_{L}(\vx))^\T]\|_2 
    \leq 24\sqrt{M} \left(\gamma_{\alpha,M} (K^2 \lor 1)\right)^{L}\left(\alpha^2 + 1\lor \tfrac1{D\ss{i}}\|\vx\|^2_2\right).
\end{equation}

\end{proof}

Given \cref{lem:offdiag-operator-norm}, we can prove \cref{lem:app-off-diagonal-variance} using the definition of operator norm.
\begin{proof}[Proof of \cref{lem:app-off-diagonal-variance}]
We have 
\begin{align}
    |O^i_{\Q}(\vx)| = \frac{1}{M}|\E_{\Q}[\vw_{L+1}]^\T\E_{\Q}[\phi\ss{o}(\vz_{L})\phi\ss{o}(\vz_{L})^\T]\E_{\Q}[\vw_{L+1}]|
    \leq \frac{1}{M}\|\E_{\Q}[\vw_{L+1}]\|_2^2\norm{\E_{\Q}[\phi\ss{o}(\vz_{L})\phi\ss{o}(\vz_{L})^\T]}_2,
\end{align}
by the definition of the operator norm. \Cref{lem:parameter_bound} tells us that $\|\E_{\Q}[\vw_{L+1}]\|_2^2 \leq K^2 = 2\KL(\Q,\P)$. Combining this estimate with \cref{lem:offdiag-operator-norm} gives
\begin{align}
    |O^i_{\Q}(\vx)| 
    &\leq \tfrac{1}{M} \left(K^2\right) \left(24 \sqrt{M} \left(\gamma_{\alpha,M} (K^2 \lor 1)\right)^{L}\left(\alpha^2 + 1\lor \tfrac1{D\ss{i}}\|\vx\|^2_2\right) \right)\\
    &\leq \tfrac{48}{\sqrt{M}}
    \KL(\Q,\P) \left(\gamma_{\alpha,M} (2\KL(\Q,\P) \lor 1)\right)^{L}\left(\alpha^2 + 1\lor \tfrac1{D\ss{i}}\|\vx\|^2_2\right).
\end{align}
\end{proof}

\section{Quantitative Bounds on the KL divergence for General  Likelihoods}\label{app:kl-bounds}
The sketch of the proof given to upper bound $\KL(\Q^*,\P)$ in the main text relied on that any convergent sequence is bounded. Unfortunately, this is not quantitative, in that it does not allow an upper bound on a computable upper bound on $\KL(\Q^*,\P)$. In this section, we show how a modification of the approach described in \cref{sec:proof-sketch} can yield a computable upper bound on $\KL(\Q^*,\P)$.

We take the same assumptions as in \cref{sec:proof-sketch}:

\begin{enumerate}[label=(\roman*)]\itemsep0pt
   \item The likelihood factorizes over data points, i.e.~$\log \mathcal{L}(\vtheta) = \sum_{n=1}^N \log p(\vy_n | \vf_{\vtheta}(\vx_n))$, for some function $p$;
    \item there exists a $C$ such that $\log p(\vy_n | \vf_{\vtheta}(\vx_n)) \leq C$;
    \item for any fixed $\vy_n$, $\log p(\vy_n | \vf_{\vtheta}(\vx_n))$ can be lower bounded by a quadratic function in $\vf_{\vtheta}(\vx_n)$.
\end{enumerate}

Also, the proof follows the same initial approach: by the optimality of $\Q^*$, we have
\begin{align}
	0 &\leq \text{ELBO}(Q^*) -  \text{ELBO}(P) \\
	&= \mathbb{E}_{\vtheta\sim Q^*}[\log \mathcal{L}(\vtheta)]- \KL(Q^*, P) - \mathbb{E}_{\vtheta\sim P}[\log \mathcal{L}(\vtheta)].     
\end{align}
Rearranging and using the assumptions on $\log \mathcal{L}(\vtheta)$,
\begin{align}
    \!\!\!\KL(Q^*, P) &\leq \mathbb{E}_{\vtheta\sim Q^*}[\log \mathcal{L}(\vtheta)] -\mathbb{E}_{\vtheta\sim P}[\log \mathcal{L}(\vtheta)]\\
    & \leq \!CN \!-\! \mathbb{E}_{\vtheta\sim P}[\textstyle\sum_{n=1}^N \!\log p(\vy_n | \vf_{\vtheta}(\vx_n))] \\
    & \leq \! CN \!-\! \E_{\vtheta \sim \P}[\textstyle\sum_{n=1}^N h_n(\vf_{\vtheta}(\vx_n))]
\end{align} 
where $h_n$ is quadratic. Since $h_n$ is quadratic, $\E_{\P}[h_n(\vf_{\vtheta}(\vx_n))]$ is a linear combination of the first and second moments of $\vf_{\vtheta}(\vx_n)$. As the top layer weight matrix has mean $0$, $\E_{\P}[\vf_{\vtheta}(\vx_n)]=0$ independent of width. We therefore need only to upper bound $\E_{\P}[\vf_{\vtheta}(\vx_n)^2]$ independent of width.

It suffices to prove this in the case $D\ss{o}=1$. In the general case, the resulting bound can simply be summed over output dimension. For the variance, note that
\begin{align}
    \V[f(\vx)]
    &= \tfrac1M\V[\lra{\vw_{L+1}, \phi(\vz_L)}] + \V[b_{L+1}] \\
    &= \tfrac1M\E[\tr[\vw_{L+1} \vw_{L+1}^\T \phi(\vz_L) \phi^\T(\vz_L)]] + 1 \\
    &= \tfrac1M\E[\norm{\phi(\vz_L)}^2_2] + 1 \\
    &\le \tfrac1M\E[\norm{\vz_L}^2_2] + 1.
\end{align}
Similarly,
\begin{align}
    \tfrac1M\E[\norm{\vz_\ell}^2_2]
    &= \tfrac1{M^2}\E[\tr[\mW_{\ell} \mW^\T_{\ell} \phi(\vz_{\ell - 1}) \phi^\T(\vz_{\ell - 1})] ]+ \tfrac1M\E[\norm{\vb_{\ell}}^2_2] \\
    &= \tfrac1M\E[\norm{\phi(\vz_{\ell - 1})}_2^2] + 1 \\
    &\le \tfrac1M\E[\norm{\vz_{\ell - 1}}_2^2] + 1.
\end{align}
Therefore,
\begin{equation}
    \V[f(x)] \le L + 1 + \tfrac{1}{D\ss{i}}\norm{\vx}_2^2. \label{eqn:var-prior-bound}
\end{equation}
\begin{example}[Gaussian Likelihood]
Consider a Gaussian likelihood so that,
\begin{equation}
    \log p(y_n | f(\vx_n)) = -\log 2 \pi \sigma^2 - \tfrac{1}{2\sigma^2}(y-f(\vx_n))^2.
\end{equation}
We then have $\log p(y_n | f(\vx_n)) \leq -\log 2 \pi \sigma^2=:C$. Also, 
\begin{align}
    \mathbb{E}_{\P}[\log p(y_n | f(\vx_n))] = CN - \frac{\sum_{n=1}^N y_n^2+\V[f(\vx_n)]}{2\sigma^2}
\end{align}
Hence by \cref{eqn:var-prior-bound},
\begin{align}
     \KL(Q^*, P) \leq \frac{(L+1)N+ \sum_{n=1}^{N} y_n^2 + \|\vx_n\|^2_2 }{2\sigma^2}. \label{eqn:kl-upper-bound-gaussian}
\end{align}
\end{example}

\begin{example}[Student $t$ likelihood]
    For a Student $t$ likelihood (used in robust regression) with $\nu >0$ degrees of freedom, we have
    \begin{equation}
        \log p(y_n
        |f(\vx_n)) = c(\nu) - \frac{\nu+1}{2} \log \parens*{1+\frac{(f(\vx_n)-y_n)^2}{\nu}}.
    \end{equation}
    As the second term is non-negative, this is upper bounded by $c(\nu)$. 
    Applying $\log (1+a) \leq a$, we have
    \begin{equation}
        \log p(y_n
        |f(\vx_n)) \geq c(\nu) - \frac{\nu+1}{2}\frac{ (f(x_n)-y_n)^2}{\nu}
    \end{equation}
    which provides the desired quadratic lower bound.
\end{example}

\begin{example}[Logistic likelihood]
    For a logistic likelihood, we have
    \begin{equation}
        \log p(y_n
        |f(\vx_n)) = y_n\log\parens*{\frac{1}{1+e^{-f(\vx_n)}}} + (1-y_n)
       \log\left(\frac{e^{-f(\vx_n)}}{1+e^{-f(\vx_n)}}\right)
    \end{equation}
    As both terms are non-positive, we have $        \log p(y_n
        |f(\vx_n))\leq 0$. 
    As $g(a)=\log(1+e^{-a})$ is three times differentiable, we have that for any $a \in \R$, there exists a $\xi_a \in \R$ such that,
    \begin{equation}
        \log(1+e^{-a}) = \log 2 - \frac{a}{2} + \frac{a^2}{8} + \frac{g^{(3)}(\xi_a)}{3!}a^3.
    \end{equation}
    As $\mathrm{sign}(g^{(3)}(\xi_a)) = \mathrm{sign}(\xi_a) =\mathrm{sign}(a) = \mathrm{sign}(a^3)$, the final term is non-negative, so
    \begin{equation}
        \log(1+e^{-a}) \leq \log 2 - \frac{a}{2} + \frac{a^2}{8}.
    \end{equation}
    Hence
    \begin{align}
        \log p(y_n=1
        |f(\vx_n)) &\geq -\log 2 + \frac{f(\vx_n) }{2}-\frac{f(\vx_n) ^2}{8} \text{ and} \\
        \log p(y_n=0
        |f(\vx_n)) &=     \log p(y_n=1
        |f(\vx_n))-f(\vx_n) \geq -\log 2 -\frac{f(\vx_n) }{2}-\frac{f(\vx_n) ^2}{8}.
    \end{align}
\end{example}

\section{Proof of Convergence in Distribution of Finite Marginals of the Variational Posterior for Odd Activation Functions}\label{app:convergence-in-dist}

In this section we derive a generalization of \cref{thm:main} that incorporates a bias.

\begin{theorem}
    Consider $N$ one-dimensional data points $(\vx_n, y_n)_{n=1}^N$ ($D\ss{i}\in \N$ and $D\ss{o}=1$) and let $\overline{y}$ be the average of the observed values.
    Let $Q^*$ be the optimal mean-field variational posterior for a neural network with $L$ hidden layers and $M$ neurons per hidden layer.
    Suppose $\phi\ss{e}=\alpha$ for some $\alpha \in \R$ and $\phi\colon \R \to \R$ is $1$-Lipschitz.
    Also suppose that the likelihood is Gaussian with variance parameter $\sigma^2$.
    Then, along any finite-dimensional distribution, as $M \to \infty$, $f \sim Q^*$ converges weakly to the sum of the NNGP and an independent Gaussian with distribution 
    \begin{equation}
        \Normal\parens*{ \parens*{\beta \times 1 + (1 - \beta) \times \frac{1}{1 + \alpha^2 + \frac{\sigma^2}{N}}} \overline{y}, \frac{\sigma^2}{\sigma^2 + N}}
        \quad \text{where} \quad
        \beta = \frac{\alpha^2}{\alpha^2 + \frac{\sigma^2}{N}}.
    \end{equation}
\end{theorem}

Observe that the mean of the independent Gaussian is a convex combination of the maximum likelihood solution $\overline{y}$ of fitting a univariate Gaussian to the data and the posterior mean $(1 + \alpha^2 + \sigma^2/N)^{-1} \overline{y}$ of observing the data as noisy observations for the final bias with noise variance $\sigma^2 + N \alpha^2$.
The coefficient of the convex combination, $\beta \in [0, 1)$, measures the strength of $\alpha$ relative to the observation noise and number of observations.
As $\alpha \to \infty$, the mean converges to the maximum likelihood solution;
and as $\alpha \to 0$, the mean converges to the posterior mean.

\begin{proof}
    The proof conceptually proceeds in two steps.
    In the first conceptual step, we show that the parameters of all layers but the last layer converge to the prior.
    Intuitively, this means that the overall solution converges to mean-field inference in the Bayesian linear regression model with features $x \mapsto \E[\phi(z_{L,m})(x)]$.
    In the second conceptual step, we show that this mean-field solution of the Bayesian linear regression model converges to the prior, thereby completing the proof.
    Whereas conceptually the proof proceeds in two steps, below we split up these two steps into six steps, as follows:
    \begin{enumerate}[leftmargin=5em]
        \item[Step 1:] We prove that the parameters of the hidden layers converge to the prior.
        \item[Step 2:] By assuming two claims, Claim 1 and Claim 2, we almost conclude the proof. It only remains to compute the limiting distribution of the final bias.
        \item[Step 3:] We prove Claim 1.
        \item[Step 4:] Using Claim 1, we compute the limiting distribution of the final bias.
        \item[Step 5:] Using Claim 1 and the limiting distribution of the final bias, we prove Claim 2.
        \item[Step 6:] We reconcile Step 2 with the limiting distribution of the bias to finally conclude the proof.
    \end{enumerate}
    
    Throughout the proof and unlike in the theorem statement, we do \emph{not} incorporate $\alpha$ in $\phi$, but write $\phi_\alpha = \phi + \alpha$ for the version of $\phi$ that does include $\alpha$.
    Moreover, we write $\vtheta_{L+1}$ for the parameters of the final layer and let $\vtheta_{1:L}$ be all remaining parameters, so $\vtheta = (\vtheta_{L+1}, \vtheta_{1:L})$.
    We decompose the parameters of the final layer as $\vtheta_{L+1} = (\vw_{L+1}, b_{L+1})$, so also $\vtheta = (\vw_{L+1}, b_{L+1}, \vtheta_{1:L})$.
    
    We will consider the limits $M \to \infty$ and $\KL \to 0$.
    Using results from earlier sections in the appendix, we establish the following bounds which ignore proportionality constants irrelevant for the limits $M \to \infty$ and $\KL \to 0$:
    \begin{align}
        \norm{\E_Q[\phi(\vz_{L+1})]}_2
        \overset{\text{(\cref{lem:last-layer-bound})}}&{\lesssim}\!\!\!\!\! \sqrt{\KL(Q_{\vtheta_{1:L}}, P_{\vtheta_{1:L}})},
        \label{eq:dist:phi}
        \\
        \abs{\E_Q[\tfrac{1}{\sqrt{M}} \lra{\vw_{L+1}, \phi(\vz_{L+1})}]}
        \overset{\text{(\cref{thm:app-convergence-of-mean})}}&{\lesssim}\!\!\!\!\! \tfrac{1}{\sqrt{M}} \KL(Q_{\vw_{L+1}, \vtheta_{1:L}}, P_{\vw_{L+1}, \vtheta_{1:L}}),
        \label{eq:dist:mean}
        \\
        \abs{\E_Q[b_{L+1}]} \overset{\text{(\cref{lem:parameter_bound})}}&{\lesssim} \sqrt{\KL(Q_{b_{L+1}}, P_{b_{L+1}})},
        \label{eq:dist:bias}
        \\
        \tfrac1M \norm{\E_{Q}[\phi(\vz_L)\phi(\vz_L)^\T]}_2
        \overset{\text{(\cref{lem:offdiag-operator-norm})}}&{\lesssim}
        \tfrac{1}{\sqrt{M}},
        \label{eq:dist:outer}
        \\
        \norm{\E_Q[\phi_\alpha(\vz_L)^2]}_\infty
        \overset{\text{(\cref{lem:z-squared-recursion})}}&{\lesssim}
        1,
        \label{eq:dist:square}
        \\
        \abs{\E_{Q}[\phi(z_{L,m})\phi(z_{L,m'})]}
        \overset{\text{(\cref{lem:cov,lem:z-squared-recursion}, $m \neq m'$)}}&{\lesssim}
        \sqrt{\KL(Q_{\vtheta_{1:L}}, P_{\vtheta_{1:L}})},
        \label{eq:dist:cov}
        \\
        \!\!\!\!\!\!\!\!\tfrac1M\abs{\tr \left( \E_{Q}[\vw_{L+1} \vw_{L+1}^\T] (
            \E_{Q}[\phi(\vz_L) \phi(\vz_L)^\T]
            \!-\! \E_{P}[\phi(\vz_L) \phi(\vz_L)^\T]
        ) \right)}
        \overset{\text{(\cref{lem:second_moment_conditional_prior})}}&{\lesssim} \tfrac{1}{\sqrt{M}} \sqrt{\KL(Q_{\vw_{L+1}, \vtheta_{1:L}},\! P_{\vw_{L+1}, \vtheta_{1:L}})},
        \label{eq:dist:complicated_trace}
        \\
        \tfrac1{M}\abs{\E[\lra{\vw_{L+1}, \alpha \vone} \lra{\vw_{L+1}, \phi(\vz_L)}]}
        \overset{\text{(\cref{lem:inner_prod})}}&{\lesssim}
        \tfrac{1}{\sqrt{M}} \sqrt{\KL(Q_{\vw_{L+1}, \vtheta_{1:L}}, P_{\vw_{L+1}, \vtheta_{1:L}})}.
        \label{eq:dist:complicated_inner_prod}
    \end{align}
    Recall that any KL divergence between parameters of an optimal variational posterior and the prior is bounded uniformly over $M$ (\cref{app:kl-bounds}).
    This means, e.g., that the mean of any weight of any variational posterior can be considered as an unknown but bounded constant (\cref{app:parameter-bounds}).

    \textbf{Step 1 (convergence of hidden layers).}
    Let $P$ be the prior and let $Q^*$ be the optimal mean-field posterior.
    Decompose the KL divergence as follows:
    \begin{equation}
        \KL(Q^*, P)
        = \KL(Q^*_{\vtheta_{L+1}}, P_{\vtheta_{L+1}}) + \KL(Q^*_{\vtheta_{1:L}}, P_{\vtheta_{1:L}})
        = \E_{Q^*}[\log p(\vy \cond \vtheta)] - \operatorname{ELBO}(Q^*).
    \end{equation}
    Let $Q'$ be the modification of $Q^*$ where the distribution of $\vtheta_{1:L}$ is set to the prior.
    Then
    \begin{align}
        \KL(Q^*_{\vtheta_{1:L}}, P_{\vtheta_{1:L}})
        &=
            \E_{Q^*}[\log p(\vy \cond \vtheta)]
            - \operatorname{ELBO}(Q^*)
            - \KL(Q^*_{\vtheta_{L+1}}, P_{\vtheta_{L+1}}) \\
        &=
            \E_{Q^*}[\log p(\vy \cond \vtheta)]
            - \operatorname{ELBO}(Q^*)
            + \operatorname{ELBO}(Q')
            - \E_{Q'}[\log p(\vy \cond \vtheta)].
    \end{align}
    Therefore, by optimality of $Q^*$,
    \begin{align}
        &\KL(Q^*_{\vtheta_{1:L}}, P_{\vtheta_{1:L}}) \nonumber \\
        &\quad\le\vphantom{\sum^N_{n}} \E_{Q^*}[\log p(\vy \cond \vtheta)] - \E_{Q'}[\log p(\vy \cond \vtheta)] \\
        &\quad= -\frac1{2\sigma^2}\sum_{n=1}^N(
            \E_{Q^*}[(y_n - f(\vx_n))^2]
            - \E_{Q'}[(y_n - f(\vx_n))^2]
        ) \\
        &\quad= -\frac1{2\sigma^2}\sum_{n=1}^N(
            \E_{Q^*}[f^2(\vx_n) - 2 y_n f(\vx_n) ]
            - \E_{Q^*}[\E_{P}[f^2(\vx_n) - 2 y_n f(\vx_n) \cond \vtheta_{L+1}]]
        ) \\
        &\quad\le \frac1{2\sigma^2}\sum_{n=1}^N\Big(
            \abs{\E_{Q^*}[f^2(\vx_n) -\E_{P}[f^2(\vx_n)\cond \vtheta_{L+1}]]}
            + 2 \abs{y_n} \abs{\E_{Q^*}[f(\vx_n) -\E_{P}[f(\vx_n)\cond \vtheta_{L+1}]]}
        \Big).
    \end{align}
    Define $\tilde b = \tfrac1{\sqrt{M}}\lra{\vw_{L+1}, \alpha \vone} + b_{L+1}$
    and let $\tilde f(\vx_n) = f(\vx_n) - \tilde b$.
    Note that $\tilde b$ is $\sigma(\vtheta_{L+1})$-measurable (i.e.~$\tilde b$ is deterministic after conditioning on the top layer parameters).
    Rearrange as
    \begin{align}
        &\E_{Q^*}[f^2(\vx_n) -\E_{P}[f^2(\vx_n)\cond \vtheta_{L+1}]] \nonumber \\
        &\quad=
            \E_{Q^*}[
                \tilde f^2(\vx_n) + 2 \tilde b \tilde f(\vx_n) + \tilde b^2
                -\E_{P}[\tilde f^2(\vx_n) + 2 \tilde b \tilde f(\vx_n) + \tilde b^2\cond \vtheta_{L+1}]
            ] \\
        &\quad=
        \E_{Q^*}[
                \tilde f^2(\vx_n)
                -\E_{P}[\tilde f^2(\vx_n) \cond \vtheta_{L+1}]
            ]
        + 2\, \E_{Q^*}[\tilde b (\tilde f(\vx_n) - \E_P[\tilde f(\vx_n) \cond \vtheta_{L+1}])].
    \end{align}
    This gives
    \begin{align}
        &\KL(Q^*_{\vtheta_{1:L}}, P_{\vtheta_{1:L}}) \nonumber \\
        &\quad\le \frac1{2\sigma^2}\sum_{n=1}^N\Big(
            \abs{\E_{Q^*}[\tilde f^2(\vx_n) -\E_{P}[\tilde f^2(\vx_n)\cond \vtheta_{L+1}]]}
            + 2 \abs{y_n} \abs{\E_{Q^*}[\tilde f(\vx_n) -\E_{P}[\tilde f(\vx_n)\cond \vtheta_{L+1}]]} \nonumber \\
        &\hphantom{\quad\le \frac1{2\sigma^2}\sum_{n=1}^N\Big(}\; + 2\, \abs{\E_{Q^*}[\tilde b (\tilde f(\vx_n) - \E_P[\tilde f(\vx_n) \cond \vtheta_{L+1}])]} \Big).
    \end{align}
    By oddness of $\phi$, it can be seen that $\E_P[\tilde f(\vx_n) \cond \vtheta_{L+1}] = 0$.
    Therefore,
    \begin{align}
        &\KL(Q^*_{\vtheta_{1:L}}, P_{\vtheta_{1:L}}) \nonumber \\
        &\quad\le \frac1{2\sigma^2}\sum_{n=1}^N\Big(
            \abs{\E_{Q^*}[\tilde f^2(\vx_n) -\E_{P}[\tilde f^2(\vx_n)\cond \vtheta_{L+1}]]}
            + 2 \abs{y_n} \abs{\E_{Q^*}[\tilde f(\vx_n)]} + 2 \abs{\E_{Q^*}[\tilde b \tilde f(\vx_n)]} \Big).
    \end{align}
    Here
    \begin{align}
        \E_{Q^*}[\tilde f^2(\vx_n) -\E_{P}[\tilde f^2(\vx_n)\cond \vtheta_{L+1}]]
        &= \tfrac1M\tr \E_{Q^*}[\vw_{L+1} \vw_{L+1}^\T] (
            \E_{Q^*}[\phi(\vz_L) \phi(\vz_L)^\T]
            - \E_{P}[\phi(\vz_L) \phi(\vz_L)^\T]
        ), 
    \end{align}
    which is $O(1/\sqrt{M})$ by \cref{eq:dist:complicated_trace}, and
    \begin{align}
        \abs{\E_{Q^*}[\tilde b \tilde f(x_n)]}
        &\le
            \abs{\E_{Q^*}[b_{L+1}]} \abs{\E_{Q^*}[\tilde f(x_n)]}
            + \tfrac1{M}\abs{\E_{Q^*}[\lra{\vw_{L+1}, \alpha \vone} \lra{\vw_{L+1}, \phi(\vz_{L})}]},
    \end{align}
    which is $O(1/\sqrt{M})$ by \cref{eq:dist:bias,eq:dist:mean} applied to the first term and \cref{eq:dist:complicated_inner_prod} applied to the second term. Therefore, $\KL(Q^*_{\vtheta_{1:L}}, P_{\vtheta_{1:L}}) = O(1/\sqrt{M})$.
    
    \textbf{Step 2 (beginning of conclusion of proof).}
    Let
    \begin{equation}
        d_M = \sum_{n=1}^N (y_n - \E_{Q^*}[b_{L+1}]).
    \end{equation}
    Although $d_M$ depends on $Q^*_{b_{L+1}}$ which in turn depends on $M$, note that $\E_{Q^*}[b_{L+1}]$ and therefore $d_M$ can be treated like unknown but bounded constants.
    Set $c_M = \frac{\alpha d_M}{\alpha^2 N + \sigma^2}$.
    We make two claims:
    \begin{align}
        \KL(Q^*_{\vw_{L+1} - \frac{c_M}{\sqrt{M}} \vone}, P_{\vw_{L+1}}) &\to 0,
         \tag{Claim 1} \quad \\
        c_M &\to c_\infty \tag{Claim 2}
    \end{align}
    where $Q^*_{\vw_{L+1} - \frac{c_M}{\sqrt{M}} \vone}$ is the distribution of $\vw_{L+1} - \frac{c_M}{\sqrt{M}} \vone$ under $Q^*$ and $c_\infty \in \R$ is some constant.
    The claims will be proven in the next parts.
    Assuming the claims, denote $\vw_{L+1}' = \vw_{L+1} - \frac{c_M}{\sqrt{M}}  \vone$ and $\phi_\alpha = \phi + \alpha$ and decompose
    \begin{align}
        \tfrac{1}{\sqrt{M}}\lra{\vw_{L+1}, \phi_\alpha(\vz_L)}
        &= \tfrac{1}{\sqrt{M}}\lra{\vw'_{L+1}, \phi_\alpha(\vz_L)} + \tfrac{c_M}{M} \lra{\vone, \phi_\alpha(\vz_L)}.
    \end{align}
    By Chebyshev's inequality, we have $\tfrac{c_M}{M} \lra{\vone, \phi(\vz_L)} \to 0$ in probability under $Q^*$:
    \begin{align}
        \E_{Q^*}[\tfrac{1}{M} \lra{\vone, \phi(\vz_L)}]
        &\le \tfrac{1}{\sqrt{M}} \norm{\E_{Q^*}[\phi(\vz_L)]}_2 \\
        \V_{Q^*}[\tfrac1M\lra{\vone, \phi(\vz_L)}]
        &= \tfrac1{M^2} \lra{\vone, \E_{Q^*}[\phi(\vz_L)\phi(\vz_L)^\T] \vone}
        - \tfrac{1}{M^2} \lra{\vone, \E_{Q^*}[\phi(\vz_L)]}^2 \\
        &\le \tfrac1M \norm{\E_{Q^*}[\phi(\vz_L)\phi(\vz_L)^\T]}_2 + \tfrac1M \norm{\E_{Q^*}[\phi(\vz_L)]}_2^2,
    \end{align}
    which are both $O(1/\sqrt{M})$ using \cref{eq:dist:phi,eq:dist:outer}.
    Since $\tfrac{c_M}{M} \lra{\vone, \phi(\vz_L)} \to 0$ in probability, $\tfrac{c_M}{M} \lra{\vone, \phi_\alpha(\vz_L)} = \tfrac{c_M}{M} \lra{\vone, \phi(\vz_L)} + \alpha c_M \to \alpha c_\infty$ in probability by Claim 2.
    
    Suppressing the dependence on $\vx$, define $g_{\vtheta} = \tfrac{1}{\sqrt{M}}\lra{\vw'_{L+1}, \phi_\alpha(\vz_L)}$ and note that it is a deterministic function of $(\vw'_{L+1}, \vtheta_{1:L})$.
    Also write $f_{\vtheta} = \tfrac{1}{\sqrt{M}}\lra{\vw_{L+1}, \phi_\alpha(\vz_L)}$ and note that it the \emph{same} deterministic function of $(\vw_{L+1}, \vtheta_{1:L})$.
    (That it is the same deterministic function will allow us to use the data processing inequality below.)
    Let $I \in \mathbb{N}$ and $\mX = (\vx_1, \ldots, \vx_I) \in (\R^{D\ss{i}})^I$.
    Let $Q^{(M)}_{g}$ be the finite-dimensional distribution of $g$ at $\mX$ under the optimal mean-field solution $Q^*$ at width $M$
    and
    let $P^{(M)}_{f}$ be the finite-dimensional distribution of $f$ at $\mX$ under the prior $P$ at width $M$.
    Using Pinsker's inequality and the data processing inequality for KL-divergences we have
    \begin{align}
        \mathrm{TV}(Q^{(M)}_{g}, P^{(M)}_{f}) \leq \sqrt{\tfrac{1}{2}\KL(Q^{(M)}_{g}, P^{(M)}_{f})} \leq \sqrt{\tfrac{1}{2}\KL(\Q^{(M)}_{\vw'_{L+1}, \vtheta_{1:L}},\P^{(M)}_{\vw_{L+1}, \vtheta_{1:L}})}, \label{eqn:pinsker0}
    \end{align}
    which goes to zero as $M \to \infty$ by the previous part and Claim 1.
    Let $d$ be the L\'evy--Prokhorov metric on $\R^I$ with the Borel $\sigma$-algebra.
    Since $\R^I$ is separable, the L\'evy--Prokhorov metric metrizes weak convergence.
    Moreover, since $\R^I$ is separable, the L\'evy--Prokhorov metric is upper bounded by the total variation distance.
    By triangle inequality we then bound the distance between $Q^{(M)}_{g}$ and the NNGP, which we denote by $P\ss{NN}$:
    \begin{align}
        d(Q^{(M)}_{g}, P\ss{NN})
        \leq
            d(Q^{(M)}_{g}, P^{(M)}_{f})
            + d(P^{(M)}_{f}, P\ss{NN})
        \leq
            \operatorname{TV}(Q^{(M)}_{g}, P^{(M)}_{f})
            + d(P^{(M)}_{f}, P\ss{NN}).
    \end{align}
    As $M \to \infty$, the first term converges to zero by \cref{eqn:pinsker0} and the second term converges to zero because $P_f^{(M)}$ converges to the NNGP \citep{matthews_2018}.
    Hence, $Q^{(M)}_{g}$ converges weakly to the NNGP, $P\ss{NN}$.
    Since $f = g + \tfrac{c_M}{M} \lra{\vone, \phi_\alpha(\vz_L)}$ and we previously showed that $\tfrac{c_M}{M} \lra{\vone, \phi_\alpha(\vz_L)} \to \alpha c_\infty$ in probability under $Q^*$, we conclude that, along any finite-dimensional distribution, $f$ converges to the NNGP plus the constant $\alpha c_\infty$.
    It remains to add the limiting distribution of the bias, which we do in the last step of the proof.
    
    \textbf{Step 3 (proof of Claim 1).}
    We previously claimed that
    \begin{equation}
        \KL(Q^*_{\vw_{L+1} - \frac{c_M}{\sqrt{M}} \vone}, P_{\vw_{L+1}}) \to 0.
    \end{equation}
    We now prove this claim.
    Let $q^*(\vtheta) = q^*(\vw_{L+1})q^*(b_{L+1})q^*(\vtheta_{1:L})$ be the density of $Q^*$ w.r.t.\ the Lebesgue measure.
    Let $\L(q(\vw_{L+1}),q(b_{L+1}),q(\vtheta_{1:L}))$ be the ELBO:
    \begin{align}
        &\L(q(\vw_{L+1}),q(b_{L+1}),q(\vtheta_{1:L})) \nonumber \\
        &\qquad= \E_q[\log p(\vy \cond \vtheta)]
            - \KL(q(\vw_{L+1}),p(\vw_{L+1}))
            - \KL(q(b_{L+1}),p(b_{L+1}))
            - \KL(q(\vtheta_{1:L}),p(\vtheta_{1:L})).
    \end{align}
    Because $Q^*$ is optimal, $q^*(\vw_{L+1})$ maximizes the function $q(\vw_{L+1}) \mapsto \L(q(\vw_{L+1}),q^*(b_{L+1}),q^*(\vtheta_{1:L}))$.
    We therefore parametrize $q(\vw_{L+1}) = \Normal(\vmu, \diag(\vnu))$ and set the gradients of $(\vmu, \vnu) \mapsto \L(q(\vw_{L+1}),q^*(b_{L+1}),q^*(\vtheta_{1:L}))$ to zero to find equations which characterize the mean and variance of $q^*(\vw_{L+1}) = \Normal(\vmu^*, \diag(\vnu^*))$.
    Consider the joint density $q(\vtheta) = q(\vw_{L+1})q^*(b_{L+1})q^*(\vtheta_{1:L})$.
    Denote $\phi_\alpha = \phi + \alpha$.
    Compute
    \begin{align}
        &\E_q[\log p(\vy \cond \vtheta)]
        - \KL(q(\vw_{L+1}),p(\vw_{L+1})) \nonumber \\
        &\;
        = -\frac1{2\sigma^2}\sum_{n=1}^N
        \parens*{
            \tfrac1M\lra{
                \vmu \vmu^\T + \diag(\vnu),
                \E_{Q^*}[\phi_\alpha(\vz_L)\phi_\alpha(\vz_L)^\T]
            }
            - 2 \tfrac{1}{\sqrt{M}}\lra{
                    \vmu,
                    \E_{Q^*}[(y_n - b_{L+1})\phi_\alpha(\vz_L)]
            }
            + \E_{Q^*}[y_n - b_{L+1}]^2
        } \nonumber\\
        &\quad\qquad - \frac12 \norm{\vmu}_2^2 - \frac12 \sum_{m=1}^M(\nu_m - 1 - \log(\nu_m)).
    \end{align}
    Denote
    \begin{equation}
        \mA = \frac1M \sum_{n=1}^N \E_{Q^*}[\phi_\alpha(\vz_L(\vx_n))\phi_\alpha(\vz_L(\vx_n))^\T],
        \quad
        \vb = \frac1{\sqrt{M}}\sum_{n=1}^N(y_n - \E_q[b_{L+1}]) \E_{Q^*}[\phi_\alpha(\vz_L(\vx_n))].
    \end{equation}
    Then the part of the ELBO depending on $\vmu$ and $\vnu$ can be written as
    \begin{equation}
         -\frac1{2\sigma^2}
        \sbrac*{
            \lra{
                \vmu \vmu^\T + \diag(\vnu),
                \mA
            }
            - 2 \lra{
                    \vmu,
                    \vb
            }
            + (y_n - \E_{Q^*}[b_{L+1}])^2
        }  - \frac12 \norm{\vmu}_2^2 - \frac12 \sum_{m=1}^M(\nu_m - 1 - \log(\nu_m)).
    \end{equation}
    Setting the gradient with respect to $\vmu$ to zero gives
    \begin{equation}
        -\frac1{\sigma^2}\mA \vmu^* + \frac1{\sigma^2}\vb - \vmu^* = 0
        \implies
        \vmu^* = (\mA + \sigma^2 \mI)^{-1}\vb.
    \end{equation}
    Similarly, setting the gradient with respect $\nu_m$ to zero gives
    \begin{equation}
        -\frac{1}{2\sigma^2} A_{mm} - \frac12 + \frac12\frac1{\nu_m^*} = 0
        \implies
        \nu_m^* = \frac{1}{1 + \frac{1}{\sigma^2}A_{mm}}.
    \end{equation}
    Therefore,
    \begin{align}
        \KL(Q^*_{\vw_{L+1} - \tfrac{c_M}{\sqrt{M}} \vone}, P_{\vw_{L+1}})
        &= \frac12 \norm{\vmu^* - \tfrac{c_M}{\sqrt{M}} \vone}_2^2
        + \frac12\sum_{m=1}^M (\nu_m^* - 1 - \log(\nu_m^*)) \\
        &= \frac{1}{2 \sigma^2} \norm{\vmu^* - \tfrac{c_M}{\sqrt{M}} \vone}_2^2
        + \frac{1}{2} \sum_{m=1}^M\parens*{
            \log(1 + \tfrac{1}{\sigma^2}A_{mm})
            - \frac{\tfrac{1}{\sigma^2}A_{mm}}{1 + \tfrac{1}{\sigma^2}A_{mm}}
        } \\
        &\overset{\text{(i)}}{\le} \frac{1}{2 \sigma^2} \norm{\vmu^* - \tfrac{c_M}{\sqrt{M}} \vone}_2^2
        + \frac{1}{2} M\parens*{
            \log(1 + \tfrac1{M}\tfrac{M}{\sigma^2}A_{m^*})
            - \frac{\tfrac1{M}\tfrac{M}{\sigma^2}A_{m^*}}{1 + \tfrac1{M}\tfrac{M}{\sigma^2}A_{m^*}}
        } \\
        &\overset{\text{(ii)}}{\le} \frac{1}{2 \sigma^2} \norm{\vmu^* - \tfrac{c_M}{\sqrt{M}} \vone}_2^2
        + \frac{1}{2 \sigma^2} \frac{M}{\sigma^2}A^2_{m^*}
    \end{align}
    where in (i) $A_{m^*} = \max_{m \in [M]} A_{mm}$ and we use that $a - 1 - \log(a)$ is increasing in $a$ for $a<1$ and that $\nu^*_m<1$. In (ii) we used the inequality\footnote{
        Note that the inequality is equivalent to
        $
            \log(1 + x) - \frac{x}{1 + x} \le x^2
        $
        for all $x > 0$, which follows from $\log(1 + x) < x$ for all $x > 0$.
    }
    \begin{equation}
        x \parens*{\log(1 + \tfrac{c}{x}) - \frac{\tfrac{c}{x}}{1+\tfrac{c}{x}}} \le \frac{c^2}{x}
        \quad 
        \text{for all $c \ge 0$ and $x > 0$}
    \end{equation}
    with $c = \tfrac{M}{\sigma^2}A_{m^*} \ge 0$ and $x = M > 0$.
    By \cref{eq:dist:square}, $M A^2_{m^*} = O(1/M)$, so the claim is shown if
    \begin{equation}
        \norm{(\mA + \sigma^2 \mI)^{-1} \vb - \tfrac{c_M}{\sqrt{M}} \vone}_2 \to 0.
    \end{equation}
    Using that
    \begin{equation}
        \norm{(\mA + \sigma^2 \mI)^{-1} \vb - \tfrac{c_M}{\sqrt{M}} \vone}_2
        \le \norm{(\mA + \sigma^2 \mI)^{-1}}_2 \norm{\vb - \tfrac{c_M}{\sqrt{M}} (\mA + \sigma^2 \mI)\vone}_2
        \le \sigma^{-2} \norm{\vb - \tfrac{c_M}{\sqrt{M}} (\mA + \sigma^2 \mI)\vone}_2,
    \end{equation}
    it suffices to show that
    $\sqrt{M}\norm{\vb - \tfrac{c_M}{\sqrt{M}} (\mA + \sigma^2 \mI)\vone}_\infty \to 0$:
    \begin{align}
        &\max_{m \in [M]} \Bigg| \sum_{n=1}^N(y_n - \E_q[b_{L+1}])\E_{Q^*}[\phi_\alpha(z_{L,m}(\vx_n))] \\
        &\qquad\qquad- \frac{\alpha d_M}{\alpha^2 N + \sigma^2}
        \parens*{\sigma^2 +\frac1M\sum_{m'=1}^M \sum_{n=1}^N \E_{Q^*}[\phi_\alpha(z_{L,m}(\vx_n))\phi_\alpha(z_{L,m'}(\vx_n))]}\Bigg| \to 0. \nonumber
    \end{align}
    Expand
    \begin{align}
        &\frac1M\sum_{m'=1}^M \sum_{n=1}^N \E_{Q^*}[\phi_\alpha(z_{L,m}(\vx_n))\phi_\alpha(z_{L,m'}(\vx_n))] \nonumber \\
        &\quad= 
            \frac1M\sum_{m'=1}^M \sum_{n=1}^N
            (
                \E_{Q^*}[\phi(z_{L,m}(\vx_n)) \phi(z_{L,m'}(\vx_n))]
                + \alpha( \E_{Q^*}[\phi(z_{L,m}(\vx_n))] + \E_{Q^*}[\phi(z_{L,m'}(\vx_n))]) + \alpha^2
            ) \\
        &\quad= 
            \frac1M\sum_{n=1}^N
                \E_{Q^*}[\phi(z_{L,m}(\vx_n))^2]
                 + \frac1M\sum_{m'=1}^M \sum_{n=1}^N \alpha^2
             \\
        &\quad\qquad
        + \frac1M\sum_{m'\neq m}^M \sum_{n=1}^N
                \E_{Q^*}[\phi(z_{L,m}(\vx_n)) \phi(z_{L,m'}(\vx_n))]
        +\frac1M\sum_{m'=1}^M \sum_{n=1}^N
            \alpha( \E_{Q^*}[\phi(z_{L,m}(\vx_n))] + \E_{Q^*}[\phi(z_{L,m'}(\vx_n))]). \nonumber
    \end{align}
    Collecting the first two terms of the expansion, note that, by definition of $d_M$ and \cref{eq:dist:square},
    \begin{equation}
        \max_{m \in [M]} \Bigg|\sum_{n=1}^N (y_n - \E_{Q^*}[b_{L+1}]) \alpha
        - \frac{\alpha d_M}{\alpha^2 N + \sigma^2}
        \parens*{\sigma^2
            +\frac1M \sum_{n=1}^N \E_{Q^*}[\phi(z_{L,m}(\vx_n))^2] + \frac1M \sum_{n=1}^N \sum_{m'=1}^M \alpha^2
        }\Bigg|
        \to 0.
    \end{equation}
    It therefore remains to show that the remainder also goes, which follows from the following three limits:
    \begin{align}
        \max_{m \in [M]}\,\Bigg|\sum_{n=1}^N (y_n - \E_{Q^*}[b_{L+1}]) \E_{Q^*}[\phi(z_{L,m}(\vx_n))]\Bigg| \overset{\text{\cref{eq:dist:phi,eq:dist:bias} with } \KL(Q^*_{\vtheta_{1:L}}, P_{\vtheta_{1:L}}) \to 0}&{\to} 0, \\
        \max_{m \in [M]}\,\Bigg|\frac1M \sum_{n=1}^N \sum_{m' \neq m}^M \E_{Q^*}[\phi(z_{L,m}(\vx_n))\phi(z_{L,m'}(\vx_n))]\Bigg| \overset{\text{\cref{eq:dist:cov} with } \KL(Q^*_{\vtheta_{1:L}}, P_{\vtheta_{1:L}}) \to 0}&{\to} 0, \\
        \max_{m \in [M]}\,\Bigg|\frac{\alpha}M \sum_{n=1}^N \sum_{m'=1}^M \E_{Q^*}[\phi(z_{L,m}(\vx_n)) + \phi(z_{L,m'}(\vx_n))]\Bigg| \overset{\text{\cref{eq:dist:phi} with } \KL(Q^*_{\vtheta_{1:L}}, P_{\vtheta_{1:L}}) \to 0}&{\to} 0.
    \end{align}
    
    \textbf{Step 4 (convergence of the bias).}
    The starting point is to note that
    \begin{align}
        &\tfrac{1}{\sqrt{M}}
        \E_{Q^*}[\lra{\vw_{L+1}, \phi_\alpha(\vz_{L})}] \nonumber \\
        &\quad=
            \tfrac{1}{\sqrt{M}}
            \lra{\E_{Q^*}[\vw_{L+1} - \tfrac{c_M}{\sqrt{M}} \vone],\E_{Q^*}[\phi_\alpha(\vz_{L})]}
            +
            \tfrac{c_M}{M}
            \lra{\vone,\E_{Q^*}[\phi(\vz_{L})]}
            + \alpha c_M
    \end{align}
    where
    \begin{equation}
        \alpha c_M
        = \frac{\alpha^2 d_M}{\alpha^2 N + \sigma^2}
        = \frac{\alpha^2}{\alpha^2 N + \sigma^2} \sum_{n=1}^N (y_n - \E_{Q^*}[b_{L+1}])
        \eqqcolon
        \beta ( \overline{y} - \E_{Q^*}[b_{L+1}])
    \end{equation}
    with
    \begin{equation}
        \overline{y} = \frac1N\sum_{n=1}^N y_n,
        \qquad
        \beta = \frac{\alpha^2}{\alpha^2 + \sigma^2 / N}.
    \end{equation}
    Moreover,
    \begin{align*}
        \abs{
            \tfrac{1}{\sqrt{M}}
            \lra{\E_{Q^*}[\vw_{L+1} - \tfrac{c_M}{\sqrt{M}} \vone],\E_{Q^*}[\phi_\alpha(\vz_{L})]}
        }
        &\le \tfrac{1}{\sqrt{M}} (2\KL(Q^*_{\vw_{L+1} - \frac{c_M}{\sqrt{M}} \vone}, P_{\vw_{L+1}}))^{1/2} \norm{\E_{Q^*}[\phi_\alpha(\vz_{L})]}_2 \\
        &\le \tfrac{1}{\sqrt{M}} (2\KL(Q^*_{\vw_{L+1} - \frac{c_M}{\sqrt{M}}\vone}, P_{\vw_{L+1}}))^{1/2} (\sqrt{M} \abs{\alpha} + \norm{\E_{Q^*}[\phi(\vz_{L})]}_2),
    \end{align*}
    which is $o(1)$ by the Claim 1 and \cref{eq:dist:phi}; and
    \begin{align*}
        \abs{
            \tfrac{c_M}{M}
            \lra{\vone,\E_{Q^*}[\phi(\vz_{L})]}
        }
        &\le \tfrac{c_M}{\sqrt{M}} \norm{\E_{Q^*}[\phi(\vz_{L})]}_2,
    \end{align*}
    which is also $o(1)$ by \cref{eq:dist:phi}.
    We conclude that 
    \begin{equation}
        \tfrac{1}{\sqrt{M}}
        \E_{Q^*}[\lra{\vw_{L+1}, \phi_\alpha(\vz_{L})}]
        = \beta ( \overline{y} - \E_{Q^*}[b_{L+1}]) + o(1).
    \end{equation}
    
    We proceed like in the second part:
    Because $Q^*$ is optimal, $q^*(b_{L+1})$ maximizes the function $q(b_{L+1}) \mapsto \L(q^*(\vw_{L+1}),q(b_{L+1}),q^*(\vtheta_{1:L}))$.
    We therefore parametrize $q(b_{L+1}) = \Normal(\mu_M, \nu_M)$ and set the gradients of $(\mu_M, \nu_M) \mapsto \L(q(\vw^*_{L+1}),q(b_{L+1}),q^*(\vtheta_{1:L}))$ to zero to find equations which characterize the mean and variance of $q^*(b_{L+1}) = \Normal(\mu_M^*, \nu_M^*)$.
    Consider the joint density $q(\vtheta) = q^*(\vw_{L+1})q(b_{L+1})q^*(\vtheta_{1:L})$.
    Denote $\phi_\alpha = \phi + \alpha$.
    Compute
    \begin{align}
        &\E_q[\log p(\vy \cond \vtheta)]
        - \KL(q(b_{L+1}),p(b_{L+1})) \nonumber \\
        &\quad
        = -\frac1{2\sigma^2}\sum_{n=1}^N
        \parens*{
            \nu_M + \mu_M^2
            - 2 \mu_M (
                y_n
                - \tfrac{1}{\sqrt{M}}\E_{Q^*}[\lra{
                    \vw_{L+1},
                    \phi_\alpha(\vz_L)
                }]
            )
            + \E_{Q^*}[(
                y_n
                - \tfrac{1}{\sqrt{M}}\lra{
                    \vw_{L+1},
                    \phi_\alpha(\vz_L)
                }
            )^2]
        } \nonumber\\
        &\quad\qquad - \frac12\mu_M^2 - \frac12(\nu_M - 1 - \log(\nu_M)) \\
        &\quad
        = -\frac1{2\sigma^2}\sum_{n=1}^N
        \sbrac*{
            \nu_M + \mu_M^2
            - 2 \mu_M (
                y_n
                - \beta ( \overline{y} - \mu^*_M) + o(1)
            )
            + \E_{Q^*}[(
                y_n
                - \tfrac{1}{\sqrt{M}}\lra{
                    \vw_{L+1},
                    \phi_\alpha(\vz_L)
                }
            )^2]
        } \nonumber\\
        &\quad\qquad - \frac12\mu_M^2 - \frac12(\nu_M - 1 - \log(\nu_M)).
    \end{align}
    Setting the derivative w.r.t.\ $\nu_M$ to zero and solving gives $\nu_M^* = \sigma^2/(\sigma^2 + N)$.
    Setting the derivative w.r.t.\ $\mu_M$ to zero, we find
    \begin{align}
        0 &= -\frac1{2\sigma^2}\sum_{n=1}^N
        \sbrac*{
            2\mu_M^*
            - 2 (
                y_n
                - \beta ( \overline{y} - \mu_M^*)+o(1)
            )
        } - \mu_M^*\\
        &= -\frac N{2\sigma^2}
        \sbrac*{
            2\mu_M^*
            - 2 (
                y_n
                - \beta ( \overline{y} - \mu_M^*)
            )
        } - \mu_M^*  + o(1)\\
        &= -\frac N{\sigma^2}
        \sbrac*{
            \mu_M^*
            - (
                \overline{y}
                - \beta ( \overline{y} - \mu_M^*)
            )
        } - \mu_M^* + o(1) \\
        &= -\frac N{\sigma^2}
        \sbrac*{
            (1 - \beta) \mu_M^*
            - (1 - \beta) \overline{y}
        } - \mu_M^* + o(1)\\
        &= - (1 - \beta) \frac N{\sigma^2}
        \sbrac*{
            \mu_M^*
            - \overline{y}
        } - \mu_M^* + o(1) \\
        &= - \frac{1}{\alpha^2 + \frac{\sigma^2}{N}}
        \sbrac*{
            \mu_M^*
            - \overline{y}
        } - \mu_M^* + o(1) \\
        \implies
        \mu_M^* &= \frac{1}{(\alpha^2 + \frac{\sigma^2}{M})^{-1} + 1}\parens*{
            \frac{1}{\alpha^2 + \frac{\sigma^2}{M}} \overline{y}
            + o(1)
        }.
    \end{align}
    Therefore, taking $M \to \infty$, under $Q^*$,
    \begin{equation}
        b_{L+1} \distto \Normal\parens*{\frac{1}{1 + \alpha^2 + \frac{\sigma^2}{N}} \overline{y}, \frac{\sigma^2}{\sigma^2 + N}}.
    \end{equation}
    
    \textbf{Step 5 (proof of Claim 2).}
    By the previous step, we have that
    \begin{align}
        \alpha c_M
        &= \beta (\overline{y} - \mu^*_M) \\
        &\to \beta \parens*{1 - \frac{1}{1 + \alpha^2 + \frac{\sigma^2}{N}}} \overline{y} \\
        &= c_\infty,
    \end{align}
    which proves Claim 2.
    
    \textbf{Step 6 (end of conclusion of proof).}
    In Step 2, we showed that, under $Q^*$, along any finite-dimensional distribution, $\tfrac{1}{M}\lra{\vw_{L+1}, \phi_\alpha(\vz_L)}$ converges to the NNGP plus the constant $\alpha c_\infty$.
    We now add the limiting distribution of the final bias to this, which concludes the proof.
    For this, we note that $\alpha c_\infty$ simply adds to the mean of the bias:
    \begin{equation}
        \alpha c_M + \mu^*_M \to \parens*{\beta \times 1 + (1 - \beta) \times \frac{1}{1 + \alpha^2 + \frac{\sigma^2}{N}}} \overline{y},
    \end{equation}
    which agrees with the theorem statement.
\end{proof}

\section{Example Showing that the Mean of the Variational Posterior Need not Converge if Activation Functions are not Odd}\label{app:counterexample}

We begin by recalling that any function $\phi: \R \to \R$ can be decomposed into the sum of an even and an odd function in a (unique) way. In particular, we have,
\begin{align}
    \phi(a) = \frac{\phi(a)+\phi(-a)}{2}+\frac{\phi(a)-\phi(-a)}{2}.
\end{align}
We define $\phi_{e}(a)=\frac{\phi(a)+\phi(-a)}{2}$ and $\phi_{o}(a) =\frac{\phi(a)-\phi(-a)}{2}$ to be the even and odd parts of $\phi$ respectively.

The main goal of this section is to prove the following theorem, which shows that for certain activations including ReLU, the variational posterior need not converge to the prior as the width tends to infinity.

\begin{theorem}\label{thm:app-counterexample}
 We consider a Bayesian neural network prior with a standard Gaussian distribution over the weights defines in \cref{eqn:nn-output,eqn:nn-recursion,eqn:nn-base}. Let $\mathcal{Q}$ denote the set of all mean field variational distributions over feed-forward neural networks with activation function $\phi$, $L$ hidden layers, and $M$ neurons per layer.
    Assume the following conditions:
    \begin{enumerate}[label=\roman*.,itemsep=0pt,topsep=0pt,partopsep=0pt,parsep=0pt]
        \item $\phi$ is continuous.
        \item
            $\phi(a) = O(|a|+1)$.
            This condition is equivalent to the linear envelop condition \citet{matthews_2018}.
        \item There exist $a,a' \in \R$ such that $\phi_e(a) \neq \phi_e(a')$. In words, it is not true that $\phi$ is equal to the sum of an odd function and a constant.
        \item There exists an $a \in \R$ such that $\phi(a)^2+\phi(-a)^2 \neq 2\phi(0)^2$. This is equivalent to $\phi^2$ is not equal to an odd function plus a constant.
    \end{enumerate}
    For a dataset $D=(\vx_n, y_n)\ss{n=1}^N$, with $y_n \in \R^{D\ss{i}}$ and a given homoscedastic Gaussian likelihood let $\Q^* \in \mathcal{Q}$ be the optimal variational posterior for this dataset and likelihood. Then, there exists a dataset $D$ and a homoscedastic Gaussian likelihood such that $\Q^*$ satisfies 
    \begin{align}
        |\E_{\Q^*}[f_{\theta}(\vx)]-\E_{\Q^*}[f_{\theta}(\vx')]| \geq c
    \end{align}
    where $c$ is a constant independent of $M$. 
\end{theorem}

Our proof strategy will be to find a sequence of variational distributions (indexed by $M$) with a mean function that does not tend to a constant, and show that there exists a dataset such that sequence of ELBOs defined by this sequence  converges to a number that is higher than the ELBOs of any sequence of variational distributions that have a mean that tends to a constant function.

To this end, we introduce the following sequence of distributions, that will serve as our candidate set of distributions with non-constant means and `good' ELBOs:
\begin{definition}
 For $C \in \R$, define $\Q^C$ to be the mean field Gaussian distribution over weights of a neural network with $L$ hidden layers and $M$ neurons such that $\mathbf{W\ss{L-1}},\dotsc, \mathbf{W\ss{0}}, \mathbf{b\ss{L}}, \dotsc, \mathbf{b\ss{0}} \sim \Normal(\vnull, \mI)$ and $\mathbf{W\ss{L}}\sim \Normal(\frac{\sqrt{C}}{\sqrt{M}}\mathbf{1},\mI)$.
\end{definition}

\subsection{Sketch of Construction}
We now sketch the ideas behind the counterexample. 
The KL divergence $\KL(\Q^{C},\P)=\frac{C}{2}$ (\cref{prop:kl-to-prior}). Hence it suffices to show that $\Q^{C}$ has a expected log likelihood term at least $\frac{C}{2}$ better than any constant predictor. We note that under $\Q^{C}$, the expected value of each post-activation in the last hidden layer will be the same by exchangeability. In the case of odd activations, by symmetry this was $0$, but for other activations this expectation will generally depend on $\vx$. By choosing the final weight layer to be parallel to $\mathbf{1}$, we will make the variation in the mean as large as possible (as the mean of the last layer is parallel to the expected value of the post-activations). If the $y$ values happen to fall on this line, $\Q^{C}$ will obtain a much better mean-square error than any constant predictor. We can upper bound the variance of $\Q^{C}$ at the data, as this is the same as the prior variance. The proof is then completed by choosing values for parameters to show that $\Q^{C}$ is better than any hypothetical variational approximation with near constant mean.

\subsection{Preliminary Definitions and Results}

In this section, we define several quantities and state the necessary preliminaries to construct the counter-example. We include the proofs when they are brief, but defer the proof of \cref{prop:kernel-diagonal}, which is more involved until after constructing the counterexample.

We first define a function to represent the expected value of the post-activations in the final hidden layer,
\begin{definition}
Define $ \lambda_M\colon \R^{D\ss{i}} \to \R$ by
\begin{align}
    \lambda_M(\vx) = \E_{\P}[\phi(\vz_{L-1}^{M})].
\end{align}
Further, define
\begin{align}
    \lambda(\vx) = \E[
    \vz_{L-1}(\vx)], \qquad \vz_{L-1} \sim \Normal(0, k^{L-1}(\vx,\vx))
\end{align}
where $k^{L-1}$ is defined by the recursion,
\begin{align}\label{eqn:kernel-recursion}
    k^0(\vx,\vx) = 1 + \|\vx\|^2_2, \qquad k^{\ell}(\vx,\vx) = 1 + h(k^{\ell-1}(\vx,\vx)),
\end{align}
with $h\colon (0,\infty) \to \R$ defined by $h(a) = \E_{z\sim \Normal(0,1)}[\phi(az)^2]$. 
\end{definition}
In words, $\lambda_M(\vx)$ is the expected value of the output of the each neuron in the final hidden layer of the network under the prior at input $\vx$ (after applying the activation) for a network of width $M$ and intuitively $\lambda(\vx)$ is the limit of $\lambda_M$ as $M \to \infty$ (this will be carefully proven in \cref{prop:convergence-lambda}). $k^\ell(\cdot, \cdot)$ is the kernel function associated to the BNN with this activation under the prior as $M \to \infty$.

In order to construct the counterexample, we need three preliminary results. The first is the following calculation, 

\begin{proposition}
\label{prop:kl-to-prior}
Let $\P$ denote the prior for a network with $L$ hidden layers and $M$ neurons per hidden layer, i.e.~ $\P = \Normal(\vnull, \mI)$. Then, $\KL(\Q^C,\P) = \tfrac{C}{2}$.
\end{proposition}
\begin{proof}
By independence  and the form of the variational posterior, we have, 
\begin{align}
\KL(\Q^C,\P) 
&= \KL\parens*{\Normal(\sqrt{C/M}\mathbf{1}, \mI\ss{M}),\Normal(\vnull, \mI\ss{M})}
= \frac{1}{2}\left\|\sqrt{C/M}\mathbf{1}\right\|_2^2 
= \frac{C}{2M}\left\|\mathbf{1}\right\|_2^2 
= \frac{C}{2}. \qedhere
\end{align}
\end{proof}

\begin{proposition}\label{prop:convergence-lambda}
    Suppose $\phi$ satisfies conditions i.\ and ii.\ in \cref{thm:app-counterexample}.
    Then, for all $\vx \in \R^{D\ss{i}}$,
    $\lambda_M(\vx ) \to \lambda(\vx )$ and $1 + \E[\phi(\vz_{L-1}^{M}(\vx)])^2] \to k(\vx, \vx)$.
\end{proposition}

\Cref{prop:convergence-lambda} is essentially a corollary of results in \citet{matthews_2018}.

\begin{proof}
    Fix $\vx\in \R^{D\ss{i}}$.
    Let $z \sim \Normal(0, k^{L-1}(\vx,\vx))$ and let $\distto$ denote convergence in distribution.
    By \citet[Theorem 4]{matthews_2018} under $\P$, and hence also under $\Q$, $z_{L-1,m}^M(\vx) \distto z$ for each $m$.
    Since $\phi$ is continuous, by the continuous mapping theorem \citep[Theorem 25.7]{billingsley2008probability}, $\phi(z_{L-1,m}^M(\vx)) \distto \phi(z)$.
    To strengthen convergence in distribution to convergence of the means, we note that $(\phi(z_{L-1, m}^M(\vx)))_{M=1}^\infty$ is uniformly integrable \citep[Lemma 21 by][]{matthews_2018} and apply \citet[Theorem 25.12]{billingsley2008probability}:
    $\lambda_{M}(\vx) = \E[\phi(z_{L-1,m}^M(\vx))] \to \E[\phi(z)] = \lambda(\vx)$.
    Noting that $k^{L}(\vx,\vx) = 1 + \E[\phi(z)^2]$, the proof to show the second limit is exactly the same.
\end{proof}

The final two propositions we need will show states that for activations satisfying conditions i-iv. in \cref{thm:app-counterexample}, $\lambda(x)$ takes at least two values:
\begin{proposition}\label{prop:kernel-diagonal}
Suppose $\phi$ is continuous and $\phi(a)^2+\phi(-a^2) \neq c$. Define $\kappa^{\ell}: \R^{D\ss{i}} \to \R$ by $\kappa^{\ell}(\vx) = k^{\ell}(\vx,\vx)$ with $k^{\ell}(\vx,\vx)$ defined by  \cref{eqn:kernel-recursion}. Then for all $\ell \in \mathbb{N} \cup \{0\}$, $\kappa^{\ell}(\R^{D\ss{i}})$ contains an open interval.
\end{proposition}

\begin{proposition}\label{lem:odd-fn1}
 Suppose $\phi$ satisfies condition i-iii. of \cref{thm:app-counterexample}. Then for any open interval $I \subset (0, \infty)$ there exists an $a,a' \in I$ such that
\begin{align}
    \gamma(a) \neq \gamma(a').
\end{align}
with $\gamma: (0,\infty) \to \R$ defined by $\gamma(a) = \E_{z\sim \Normal(0,1)}[\phi(az)]$.
Further $\gamma(I)$ contains an open interval.
\end{proposition}
\begin{remark}\label{rem:combining-props}
Note that taken together, these imply that if $\phi$ satisfies $i-iv.$ then the image of $\lambda$ contains an open interval. This can be seen by applying \cref{prop:kernel-diagonal} with $\ell = L-1$, to conclude the diagonal of $k^{L-1}$ contains and open interval, then noting that $\lambda(\vx) = \E_{z\sim \Normal(0,1)}[\phi(\sqrt{k^{L-1}(\vx,\vx)})z]$, so we may apply \cref{lem:odd-fn1}.
\end{remark}

\subsection{Construction}

\begin{proof}[Construction of counterexample in \cref{thm:app-counterexample}]
Suppose we have a data-set consisting of two points, $((\vx,y), (\vx',y'))$, with $\vx, \vx'$ such that $\lambda(\vx) \neq \lambda(\vx')$. The existence of such an $\vx, \vx'$ is guaranteed by \cref{prop:kernel-diagonal,lem:odd-fn1}, see \cref{rem:combining-props}.

Choose $y=\sqrt{C}\lambda(\vx),y'=\sqrt{C}\lambda(\vx')$ with $C\in (0,\infty)$ to be chosen later. We supress the dependce on $C$ for convenience and write $\Q:=\Q^{C}$. Let $\Q^*$ be any posterior with ELBO better than $\Q$, then
\begin{align}
    \frac{1}{2\sigma^2}&\left(\E_{\Q^*}[(y-f_{\theta}(\vx))^2]+\E_{\Q^*}[(y'-f_{\theta}(\vx'))^2]\right) \\ 
    &\leq \KL(\Q,\P)-  \KL(\Q^*,\P) +     \frac{1}{2\sigma^2}\left(\E_{\Q}[(y-f_{\theta}(\vx))^2]+\E_{\Q}[(y'-f_{\theta}(\vx'))^2]\right) \\
    & \leq \KL(\Q,\P) +     \frac{1}{2\sigma^2}\left(\E_{\Q}[(y-f_{\theta}(\vx))^2]+\E_{\Q}[(y'-f_{\theta}(\vx'))^2]\right) \\
    & = \frac{C}{2} + \frac{1}{2\sigma^2}\left(\E_{\Q}[(y-f_{\theta}(\vx))^2]+\E_{\Q}[(y'-f_{\theta}(\vx'))^2]\right).
\end{align}
where the first inequality comes from rearranging both ELBOs, the second uses non-negativity of the KL divergence and makes use of \cref{prop:kl-to-prior}. 

Using convexity of the squared function to apply Jensen's inequality to the left hand side, and multiplying both sides by $2\sigma^2$ we have,
\begin{align}
    (y-\E_{\Q^*}[f_{\theta}(\vx)])^2+(y'-\E_{\Q^*}[f_{\theta}(\vx')])^2 & \leq \E_{\Q^*}[(y-f_{\theta}(\vx))^2]+\E_{\Q^*}[(y'-f_{\theta}(\vx'))^2] \\
     & \leq \sigma^2 C +\E_{\Q}[(y-f_{\theta}(\vx))^2]+\E_{\Q}[(y'-f_{\theta}(\vx'))^2].
\end{align}
Since $\lambda(\vx) \neq \lambda(\vx')$ there exists a $\beta>0$ such that $|\lambda(\vx)-\lambda(\vx')| \geq \beta/\sqrt{2}$. Using our choice of $y,y'$ 
\begin{align}
    \E_{\Q^*}[(y-f_{\theta}(\vx))^2]&= \E_{\Q^*}{[\sqrt{C}\lambda(\vx)-f_{\theta}(\vx)]^2}\\
    & =(\sqrt{C}\lambda(\vx)-\sqrt{C}\lambda_M(\vx))^2+ \kappa_M(\vx),
\end{align}
where $\kappa_M$ is the variance function for the width $M$ neural net prior. By \cref{prop:convergence-lambda}, for $M$ sufficiently large, we have 
\begin{align}
    (\sqrt{C}\lambda(\vx)-\sqrt{C}\lambda_M(\vx))^2+ \kappa_M(\vx) \leq \kappa(\vx) + \beta^2/2,
\end{align}
where $\kappa$ is variance function of the limiting NNGP kernel. The same argument can be applied to $\vx'$. This gives us the upper bound, for $M$ sufficiently large, 
\begin{align}
    \E_{\Q}[(y-f_{\theta}(\vx))^2]& \leq \kappa(\vx)+\kappa(\vx') + \beta^2.
\end{align}
Combining with our earlier equation, we have
\begin{align}\label{eqn:bound-on-mse}
    (\sqrt{C}\lambda(\vx)-\E_{\Q^*}[f_{\theta}(\vx)])^2+(\sqrt{C}\lambda(\vx')-\E_{\Q^*}[f_{\theta}(\vx')])^2 \leq \sigma^2C + \kappa(\vx) + \kappa(\vx') + \beta^2.
\end{align}
Therefore,
\begin{align}
    &|\E_{\Q^*}[]f_{\theta}(\vx)]- \E_{\Q^*}[f_{\theta}(\vx')]| \nonumber \\
    &\quad=
        |
            (\E_{\Q^*}[f_{\theta}(\vx)] - \sqrt{C}\lambda(\vx)) 
            - (\E_{\Q^*}[f_{\theta}(\vx')] - \sqrt{C}\lambda(\vx'))
            + (\sqrt{C}\lambda(\vx)  - \sqrt{C}\lambda(\vx')
        | \\
    &\quad\ge
        |\sqrt{C}\lambda(\vx)  - \sqrt{C}\lambda(\vx')|
        -
            |
                (\E_{\Q^*}[f_{\theta}(\vx)] - \sqrt{C}\lambda(\vx))
                - (\E_{\Q^*}[f_{\theta}(\vx')] - \sqrt{C}\lambda(\vx'))
            | \\
    &\quad\overset{\smash{\text{(i)}}}{\ge}
        |\sqrt{C}\lambda(\vx)  - \sqrt{C}\lambda(\vx')|
        -
            \sqrt{2}
            \sqrt{
                (\E_{\Q^*}[f_{\theta}(\vx)] - \sqrt{C}\lambda(\vx))^2
                + (\E_{\Q^*}[f_{\theta}(\vx')]- \sqrt{C}\lambda(\vx'))^2
            } \\
    &\quad\ge
        \sqrt{C}\beta/\sqrt{2}
        -
            \sqrt{2}
            \sqrt{\sigma^2C + \kappa(\vx) + \kappa(\vx') + \beta^2} \label{eqn:lb-on-diff}
\end{align}
using in (i) $|(x - a) - (y - b)| \le |x - a| + |y - b| \le \sqrt{2}\sqrt{(x - a)^2 + (y - b)^2}$.
To finish the proof, choose $\sigma^2 = 1/C$ and take $C$ large enough that the first term in \cref{eqn:lb-on-diff} is larger than the second.
\end{proof}

\subsection{Proof of \cref{prop:kernel-diagonal,lem:odd-fn1}}

Having completed the construction of the counterexample, it remains to prove \cref{prop:kernel-diagonal,lem:odd-fn1} in order to verify that we can indeed select two points $\vx, \vx'$ such that $\lambda(\vx) \neq \lambda(\vx')$. Note that for any typical activation, this could simply be verified numerically by working through the recursion for kernel functions, so the main purpose of the following proofs is generality.

\begin{proposition}\label{prop:e-is-continuous}
    Define $\alpha\colon (0,\infty) \to \R$ by $\alpha(a) = \E_{z\sim \Normal(0,1)}[\phi(az)]$. Suppose $\phi$ satisfies conditions i, ii of \cref{thm:app-counterexample}. Then $\alpha$ is continuous.
\end{proposition}
\begin{proof}
    Let $(a_i)_{i \ge 1} \subseteq \R$ be convergent to some $a \in \R$.
    We show that $\E[\phi(a_i z)] \to \E[\phi(a z)]$.
    If we can interchange limit and expectation, then the result follows from continuity of $\phi$.
    To show that we can interchange limit and expectation, we demonstrate an integrable dominating function.
    Since $(a_i)_{i \ge 1}$ is convergent, it is bounded, hence contained in some interval $[-K, K] \subseteq \R$.
    Using that $\phi(a) = O(|a|+1)$, let $|\phi(a)| \le C |a|$ for all $|a| \ge R$ for some $C > 0$.
    Let $M$ be the maximum of $|\phi|$ on $[-R, R]$, which is finite because $|\phi|$ is continuous.
    Then estimate
    \begin{equation}
        |\phi(a_i z)|
        \le \sup_{a \in [-K, K]} |\phi(a z)|
        \le M + C \sup_{a \in [-K, K]} |a z|
        \le M + C K |z|,
    \end{equation}
    which is integrable because $z \sim \Normal(0,1)$.
\end{proof}

\begin{proof}[Proof of \cref{lem:odd-fn1}]
Our proof follows \citet{StackExchangeGaussiansDense}.
Without loss of generality, we may assume $I=(s,t)$ is bounded with $s>0$.

Towards contradiction, suppose there exists a $c \in \R$ such that $\alpha(a)=c$ for all $a \in I$. This supposition can be written as,
\begin{align}\label{eqn:constant-expectation}
    \int \phi(az) e^{-z^2/2} dz = c' \quad \forall a \in I,
\end{align}
where $c' = c\sqrt{2\pi}$. Define $u = az$, $b = 1/(2a^2)$ and $I'= \left(\frac{1}{2t^2}, \frac{1}{2s^2}\right)$. Then we can rewrite \cref{eqn:constant-expectation} as 
\begin{align}
    \int \phi(u) e^{-bu^2} du = c' \quad \forall b \in I'.
\end{align}
We have $|u^{2n}\phi(u)e^{-bu^2}|\leq u^{2n}|\phi(u)|e^{-bu^2} \leq   u^{2n}|\phi(u)|e^{-\frac{u^2}{2t^2}}$, which is integrable since $\phi(u)=O(|x|+1)$. Hence we may take $n$ derivatives and apply Leibniz's rule,
\begin{align}\label{eqn:orthogonal-even}
    \frac{\partial^n}{\partial^n b} \int \phi(u) e^{-bu^2} du =  \int u^{2n} \phi(u) e^{-bu^2} du = 0 \quad \forall b \in I', \forall n \in \mathbb{N}.
\end{align}
For a (arbitrary) $b \in I'$, define $w(u) = e^{-bu^2}$ and $L^2(\R, w)$ to be the Hilbert space with inner product $\langle f,g \rangle = \int f(u)g(u)w(u)du$. 
Following 
\citet{126471}, we note that the set of compactly supported functions is dense in $L^2(\R, w)$, and by the Weirstrauss theorem the set of polynomial is dense in the set of compact functions. As a dense subset of a dense set is again dense, we conclude the set of polynomial is dense in $L^2(\R, w)$. 

Writing $\phi(u) = \phi_e(u) + \phi_o(u)$, we have,
\begin{align}
    \int_{-\infty}^\infty u^{2i-1} \phi_e(u) w(u) du &= 0 \quad \forall i \in \mathbb{N}
\end{align}
as the integrand is odd and we integrate over a symmetric domain. Also,
\begin{align}
    \int u^{2n} \frac{\phi(u)+ \phi(-u)}{2}w(u) du &=
    \frac{1}{2}\int u^{2i} \phi(u)w(u) du + \frac{1}{2}\int u^{2n} \phi(-u)w(u) du \\
    &=\frac{1}{2}\int u^{2i} \phi(u)w(u) du + \frac{1}{2}\int u^{2n} \phi(u)w(u) du \\
    &=0 + 0, \quad \forall i \in \mathbb{N},
\end{align}
by \cref{eqn:orthogonal-even}. Since the polynomial are a basis, we have $\phi_e(u) = \sum_{i=0}^\infty \alpha_i u^i$ for some $(\alpha_i)_{i=1}^\infty$. Taking the inner product of both sides with respect to $u^i$, shows that $\alpha_i=0$ for all $i >0$, hence  $\phi_e(u)= c^{\prime\prime}$ for some $c^{\prime\prime} \in \R$, with equality in $L^2(\R,w)$. But as $\phi_e$ is continuous, this implies it is equal to a $c^{\prime\prime}$ everywhere.

As $I$ contains an open set, it contains a ball. As $\alpha$ is continuous (\cref{prop:e-is-continuous}) the image of this ball under $\alpha$ is connected i.e.~an interval. As the interval contains at least two points, it contains an open interval. 
\end{proof}

\begin{proof}[Proof of \cref{prop:kernel-diagonal}]
The proof proceeds by induction. 

\paragraph{Base case:}
We have $\kappa^{0}(\vx) = 1+ \|\vx\|^2_2$. As $\R^{D\ss{i}}$ contains an open set, it contains an open ball. The image of this open ball under the map $\kappa$, which is continuous, must be connected, hence an interval. Therefore, if it contains at least two points, it contains an open interval. But in any ball, there are two points with different norms, so this must be the case. 

\paragraph{Inductive step:}
Defining $h(a) = \E[\phi(az)]$ and following the recursion for kernels for deep networks \citet[Lemma 2]{matthews_2018}, we have $\kappa^{\ell}(\R^{D\ss{i}}) = 1+ h(\kappa^{\ell-1}(\R^{D\ss{i}}))$. By the inductive hypothesis, $\kappa^{\ell-1}(\R^{D\ss{i}})$ contains an open interval, $I$, the image of which must be connected as $h$ is continuous ($\phi^2$ is continuous and is $O(|a|^2+1)$ so we may apply \cref{prop:e-is-continuous}). Hence, it suffices to show that $\phi^2$ satisfies the conditions of \cref{lem:odd-fn1}. This follows from the assumption that $\phi(a)^2+\phi(-a^2) \neq c$.
\end{proof}
\section{Proof of Constants for Mean Result in Single Hidden-Layer Network}\label{app:1hl-good-constants}

We have,
\begin{align}
    \|\E[\vf(\vx)]\|_2 = \tfrac1{\sqrt{M}} \|\E[\mW_2]\E[\phi(\tfrac{1}{\sqrt{D\ss{i}}}\mW_1\vx+\vb_1)]\| \leq \|\E[\mW_2]\|\ss{F}\|\E[\phi(\tfrac{1}{\sqrt{D\ss{i}}}\mW_1\vx+\vb)]\|_2 \label{eqn:1hl-init-bound}
\end{align}
Defining $\mW' = \mW_1 - \E[\mW_1]$ and $\vb' = \vb_1 - \E[\vb_1]$, we have 
\begin{align}
    \|\E[\phi(\tfrac{1}{\sqrt{D\ss{i}}}\mW_1\vx+\vb)]\|_2  &= \|\E[\phi(\tfrac{1}{\sqrt{D\ss{i}}}\mW_1\vx+\vb_1)] -\E[\phi(\mW'\vx+\vb')] \|_2 \\
    & \leq \E\|\phi(\tfrac{1}{\sqrt{D\ss{i}}}\mW_1\vx+\vb_1)] -\phi(\tfrac{1}{\sqrt{D\ss{i}}}\mW'\vx+\vb')\| \\
    & \leq \E\|(\tfrac{1}{\sqrt{D\ss{i}}}\mW_1-\mW)\vx+(\vb-\vb')\|\\
    & \leq \|\E[\mW_1]\|\ss{F}\tfrac{1}{\sqrt{D\ss{i}}}\|\vx\|_2 + \|\E[\vb_1]\|_2 \label{eqn:1hl-lower-bound}
\end{align}
Combining \cref{eqn:1hl-init-bound} and \cref{eqn:1hl-lower-bound}, 
\begin{align}
    \|\E[\vf(\vx)]\|_2 &\leq  \|\E[\mW_2]\|\ss{F}\|\E[\mW_1]\|\ss{F}\tfrac{1}{\sqrt{D\ss{i}}}\|\vx\|_2 + \|\E[\mW_2]\|\ss{F}\|\E[\vb_1]\|_2.\label{eqn:1hl-norm-bound}
\end{align}
We also know from \cref{lem:parameter_bound} that
\begin{equation}
    \|\E[\mW_1]\|\ss{F}^2+\|\E[\mW_2]\|\ss{F}^2+ \|\E[\vb_1]\|_2 \leq \KL(\Q,\P).\label{eqn:1hl-kl-constraint}
\end{equation}

Combining \cref{eqn:1hl-norm-bound} and \cref{eqn:1hl-kl-constraint} we obtain the following upper bound phrased as an optimization problem which is convex, 
\begin{align}
   \|\E[\vf(\vx)]\|_2 &\leq 
   \max_{\alpha} \quad \alpha_1\alpha_2 \tfrac{1}{\sqrt{D\ss{i}}}\|\vx\|_2+\alpha_1\alpha_3 \\
    & \;\; \text{s.t.} \; \;
    \frac{1}{2}\sum \alpha_i^2 = \KL(\Q,\P),\quad  \alpha_i  \geq 0.
\end{align}

We can solve this optimization via Lagrange multipliers. We form the Lagrangian,
\begin{align}
    \alpha_1\alpha_2 \tfrac{1}{\sqrt{D\ss{i}}}\|\vx\|_2 + \alpha_1\alpha_3 - \lambda \parens*{\frac{1}{2}\sum \alpha_i^2 - \KL(\Q,\P)}.
\end{align}
For convenience, name $c = \tfrac{1}{\sqrt{D\ss{i}}}\|\vx\|_2$. Differentiating with respect to each variable and setting to $0$ gives,
\begin{align}
    c\alpha_2 + \alpha_3 - \lambda \alpha_1 = 0 \\
   \alpha_2  =\frac{c}{\lambda}\alpha_1  \\
    \alpha_3 = \frac{1}{\lambda}\alpha_1.
\end{align}
Plugging these back in to the constraint,
\begin{align}
    \frac{1}{2}\alpha_1^2\parens*{1+\frac{1}{\lambda^2} + \frac{c^2}{\lambda^2}} = \KL(\Q,\P)
\end{align}
Solving for $\lambda$ yields,
\begin{align}
    \lambda = \alpha_1 \sqrt{\frac{1+c^2}{\KL(\Q,\P)-\frac{1}{2}\alpha_1^2}} 
\end{align}
Hence
\begin{align}
    \alpha_2  =c\sqrt{\frac{\KL(\Q,\P)-\frac{1}{2}\alpha_1^2}{1+c^2}}, \quad 
    \alpha_3 = \sqrt{\frac{\KL(\Q,\P)-\frac{1}{2}\alpha_1^2}{1+c^2}}.
\end{align}
We can now plug these back into the remaining constraint to give,
\begin{align}
\sqrt{1+c^2}\sqrt{\KL(\Q,\P)-\frac{1}{2}\alpha_1^2} - \alpha_1^2\sqrt{\frac{1+c^2}{\KL(\Q,\P)-\frac{1}{2}\alpha_1^2}} = 0
\end{align}
Simplifying slightly,
\begin{align}
\parens*{\KL(\Q,\P)-\frac{1}{2}\alpha_1^2}^2 = \alpha_1^4
\end{align}
We recognize this as a quadratic form in $\alpha_1^2$, which can be solved yielding $\alpha_1 = \sqrt{\frac{2}{3}\KL(\Q,\P)}$. So $\alpha_2 = c \sqrt{\frac{\frac{2}{3}\KL(\Q,\P)}{1+c^2}}$ and $\alpha_3 = \sqrt{\frac{\frac{2}{3}\KL(\Q,\P)}{1+c^2}}.$ The result then follows from a short calculation.

\section{Convergence of Mean of Linear Networks}\label{app:convergence-linear-nets}

We consider the case of networks with only affine layers. In particular, we suppose $\phi(a) = a$ for all $a \in \R$. In this case, we can prove upper and lower bounds on the discrepancy between the mean function at two points in the input space.
\subsection{Upper Bounds}
Then for two points $\vx, \vx' \in \R^{D\ss{i}}$,
\begin{align}
    \E[\vf(\vx)] - \E[\vf(\vx')] & = \E[\vf(\vx) - \vf(\vx')] \\
    & =D\ss{i}^{-1/2}M^{-L/2}\E\Big[\prod_{\ell=1}^{L+1} \mW_\ell\Big](\vx -\vx') \\
    & \leq D\ss{i}^{-1/2}M^{-L/2}\norm*{\prod_{\ell=1}^{L+1} \E[\mW_\ell]}_2\|\vx -\vx'\|.
\end{align}

Using sub-multiplicativity of operator norm and that $\|\cdot \|_2 \leq \| \cdot \|_F$, 
\begin{align}
        \|\E[\vf(\vx)] - \E[\vf(\vx')]\|_2 & \leq D\ss{i}^{-1/2}M^{-L/2}\prod_{\ell=1}^{L+1} \norm{\E[\mW_\ell]}\ss{F}\|\vx -\vx'\|_2. \label{eqn:linear-norm-product}
\end{align}
We can now apply the arithmetic-geometric mean inequality, to conclude 
\begin{align}
    \prod_{\ell=1}^{L+1}\norm*{\E[\mW_\ell]}_F \leq  \parens*{\frac{\sum_{\ell=1}^{L+1} \norm{\E[\mW_\ell]}_F}{L+1}}^{L+1} \label{eqn:am-gm-linear}
\end{align}
Using the $\ell^2$-$\ell^1$-inequality, we have, and \cref{lem:parameter_bound},
\begin{align}
    \sum_{\ell=1}^{L+1} \norm{\E[\mW_\ell]}_F \leq \sqrt{L+1} \sqrt{    \sum_{\ell=1}^{L+1} \norm{\E[\mW_\ell]}_F^2} \leq \sqrt{L+1} \sqrt{2\KL(\Q,\P)}. \label{eqn:l1-l2-params}
\end{align}
Combining \cref{eqn:l1-l2-params}, \cref{eqn:am-gm-linear} and \cref{eqn:linear-norm-product} we obtain,
\begin{align}
     \|\E[\vf(\vx)] - \E[\vf(\vx')]\|_2  \leq D\ss{i}^{-1/2}M^{-L/2}\parens*{\frac{2\KL(\Q,\P)}{L+1}}^{\frac{L+1}{2}}\|\vx -\vx'\|_2
\end{align}
\subsection{Lower Bounds}
We consider the case when $\vx'=\vnull$ and $\vx = \ve_1$. We consider the $\Q$ with variance of each parameter equal to $1$ and mean of all bias parameters equal to $0$. We select $\E[W_\ell]$ to be the matrix with entry $1,1$, $c$ and entries $0$ elsewhere. Then, $\frac{1}{2}(L+1)c^2 = \KL(\Q,\P)$, so $c=\frac{\sqrt{2\KL(\Q,\P)}}{L+1}$. Also, 
\begin{align}
       \|\E[\vf(\vx)] - \E[\vf(\vnull)]\|_2 = D\ss{i}^{-1/2}M^{-L/2}c^{L+1} = D\ss{i}^{-1/2}M^{-L/2}(2\KL(\Q,\P))^{\frac{L+1}{2}}(L+1)^{-(L+1)}.
\end{align}
This bound differs from the upper bound by a factor of $(L+1)^{-\frac{L+1}{2}}$.

\section{Lower Bound on Convergence for Non-Linear Networks}\label{app:convergence-lower-bound}

\begin{theorem}
    Assume the following:
    \begin{enumerate}[label=(\roman*),topsep=0pt]
        \item $D\ss{o} = 1$.
        \item $\phi$ is a sum of an odd function and a constant: $\phi\ss{e}$ is constant.
        \item $\phi$ is twice continuously differentiable with $\norm{\phi'}_\infty \le 1$ and $\norm{\phi''}_\infty < \infty$.
        \item $\phi^2$ is not a sum of an odd function and a constant: $\phi^2(a) + \phi^2(-a) \neq 2 \phi^2(0)$ for some $a \in \R$.
    \end{enumerate}
    Then, if $\phi\ss{o}$ is non-linear, there exist two inputs $\vx, \vx' \in \R^{D\ss{i}}$ and a constant $c > 0$ such that, for every $K > 0$, there exists a sequence of mean-field distributions $(\Q_M)_{M \ge 1}$ with $\KL(\Q_M, \P) = K $, one for every network width $M \ge 1$, that achieves
    \begin{equation}
        \lim_{M \to \infty} \sqrt{M}\abs{\E_{\Q_M}[f(\vx)] - \E_{\Q_M}[f(\vx')]} = c K.
    \end{equation}
\end{theorem}

\begin{proof}
    Consider the distribution of $\vz_{L-1}$ under the prior.
    Let $k$ be the covariance function of the NNGP associated to $\tfrac1{\sqrt{M}}\lra{\vep, \phi(\vz_{L-1})} + Z$ as $M \to \infty$, where $\vep$ is a vector with i.i.d.\ $\Normal(0, 1)$ entries and $Z \sim \Normal(0, 1)$.
    Henceforth, denote $k(\vx) = k(\vx, \vx)$.

    Note that $\phi' = \phi\ss{e}' + \phi\ss{o}' = \phi\ss{o}'$ is an even function.
    Suppose that $\vx \mapsto \E_{P}[\phi'(\sqrt{k(\vx)} Z)]$ is a constant function.
    By \cref{prop:convergence-lambda,prop:kernel-diagonal} in combination with the assumed conditions on $\phi$, it follows that $k(\R^{D\ss{i}})$ contains an open interval.
    Therefore, since $\phi'$ is a continuous even function and $\vx \mapsto \E_{P}[\phi'(\sqrt{k(\vx)} Z)]$ is a constant function, by an argument similar to the proof of \cref{lem:odd-fn1}, $\phi'$ must be equal to a constant function.
    However, since $\phi\ss{o}$ is non-linear, $\phi' = \phi\ss{o}'$ cannot be equal to a constant function.
    We conclude that $\vx \mapsto \E_{P}[\phi'(\sqrt{k(\vx)} Z)]$ cannot be equal to a constant function: there exist two inputs $\vx, \vx' \in \R^{D\ss{i}}$ such that $\E_{P}[\phi'(\sqrt{k(\vx)} Z)] \neq \E_{P}[\phi'(\sqrt{k(\vx')} Z)]$.
    Let $c > 0$ be the constant $c = \abs{\E_{P}[\phi'(\sqrt{k(\vx)} Z)] - \E_{P}[\phi'(\sqrt{k(\vx')} Z)]}.$

    Let $K > 0$.
    Consider the sequence of mean-field distributions $(\Q_M)_{M \ge 1}$ constructed by setting everything equal to the prior except for $\E_{Q_M}[\vb_L]= \mu \vone$ and $\E_{\Q_M}[\mW_{L+1}] = \mu \vone^\T$ with $\mu = \sqrt{K/M}$.
    Then indeed $\KL(Q_M, P) = K$.
    Moreover,
    \begin{equation}
        \tfrac1{\sqrt{M}}\mW_{L+1} \phi(\tfrac1{\sqrt{M}}\mW_L \phi(\vz_{L-1}) + \vb_L) + \vb_{L+1}
        \disteq
        \tfrac1{\sqrt{M}}\lra{\mu \vone + \vep, \phi(\tfrac1{\sqrt{M}}\mathbfcal{E} \phi(\vz_{L-1}) + \mu \vone + \vep')} + \vep''
    \end{equation}
    where $\vep$, $\vep'$, $\vep''$, and $\mathbfcal{E}$ are respectively three vectors and a matrix with i.i.d.\ $\Normal(0,1)$ entries and where $\vz_L$ is distributed under the prior.
    Using the observation that the elements of $\vz_L$ are identically distributed, compute
    \begin{equation}
        \E_{Q_M}[f(\vx)] = \tfrac\mu{\sqrt{M}} \E_{P}[\lra{\vone, \phi(\tfrac1{\sqrt{M}}\mathbfcal{E} \phi(\vz_{L-1}) + \mu \vone + \vep')}]
        = \sqrt{M} \mu\, \E_{P}[\phi(\tfrac1{\sqrt{M}}\lra{\vep''', \phi(\vz_{L-1})} + Z + \mu)]
    \end{equation}
    where $\vep'''$ is a vector with i.i.d.\ $\Normal(0, 1)$ entries and $Z \sim \Normal(0, 1)$.
    Note that $\sqrt{M}\mu = \sqrt{K}$ and
    call $Y_M = \tfrac1{\sqrt{M}}\lra{\vep''', \phi(\vz_{L-1})} + Z$.
    From Theorem 4 by \citet{matthews_2018} in combination with the assumed conditions on $\phi$, it follows that $Y_M \smash{\distto \sqrt{k(\vx)}} Z$ where $k$ is the earlier defined covariance function of the corresponding NNGP.
    
    We finally establish the asymptotic behaviour of $\E_{Q_M}[f(\vx)] - \E_{Q_M}[f(\vx')]$ in two steps.
    First, by a second-order Taylor expansion of $\mu \mapsto \E_P[\phi(Y_M + \mu)]$ around $\mu = 0$, using (a) $\norm{\phi'} < \infty $ and $\norm{\phi''}_\infty < \infty$ to interchange derivative and expectation and (b) $\norm{\phi''}_\infty < \infty$ and $\mu^2 = o(\tfrac1{\sqrt{M}})$ to determine the order of the error term,
    \begin{equation}
        \E_{Q_M}[f(\vx)] 
        = \sqrt{K} \Big[
            \E_{P}[\phi(Y_M)]
            + \mu \,\E_{P}[\phi'(Y_M)]
            + o(\tfrac1{\sqrt{M}})
        \Big].
    \end{equation}
    Note that $Y_M \disteq -Y_M$, so $\E_{P}[\phi(Y_M)] = \phi\ss{e} + \E_{P}[\phi\ss{o}(Y_M)] = \phi\ss{e}$.
    Second, since $Y_M \smash{\distto \sqrt{k(\vx)}} Z$ and $\phi'$ is continuous and bounded, $\E_P[\phi'(Y_M)] \to \E_P[\phi'(\sqrt{k(\vx)} Z)]$.
    Therefore,
    \begin{equation}
         \E_{Q_M}[f(\vx)] 
        = \sqrt{K} \Big[
            \phi\ss{e}
            + \mu \Big(\E_{P}[\phi'(\sqrt{k(\vx)} Z)] + o(1) \Big)
            + o(\tfrac1{\sqrt{M}})
        \Big],
    \end{equation}
    so, again using that $\sqrt{M} \mu = \sqrt{K}$,
    \begin{equation}
        \sqrt{M}
        (\E_{Q_M}[f(\vx)] 
        - \E_{Q_M}[f(\vx')])
        =
            K \Big(\E_{P}[\phi'(\sqrt{k(\vx)} Z)] - \E_{P}[\phi'(\sqrt{k(\vx')} Z)] + o(1) \Big)
            + o(1)
        = c K + o(1),
    \end{equation}
    which concludes the proof.
\end{proof}

\section{Experimental Setup}\label{app:exp-setup}

In this section we describe the details of our experiments.

\paragraph{Additional experiments}
Figure \cref{fig:many_datasets_var} shows the RMSE between the MFVI posterior posterior predictive variance and the prior predictive variance. 

\begin{figure}[t]
	\centering
	\includegraphics[width=.47\textwidth]{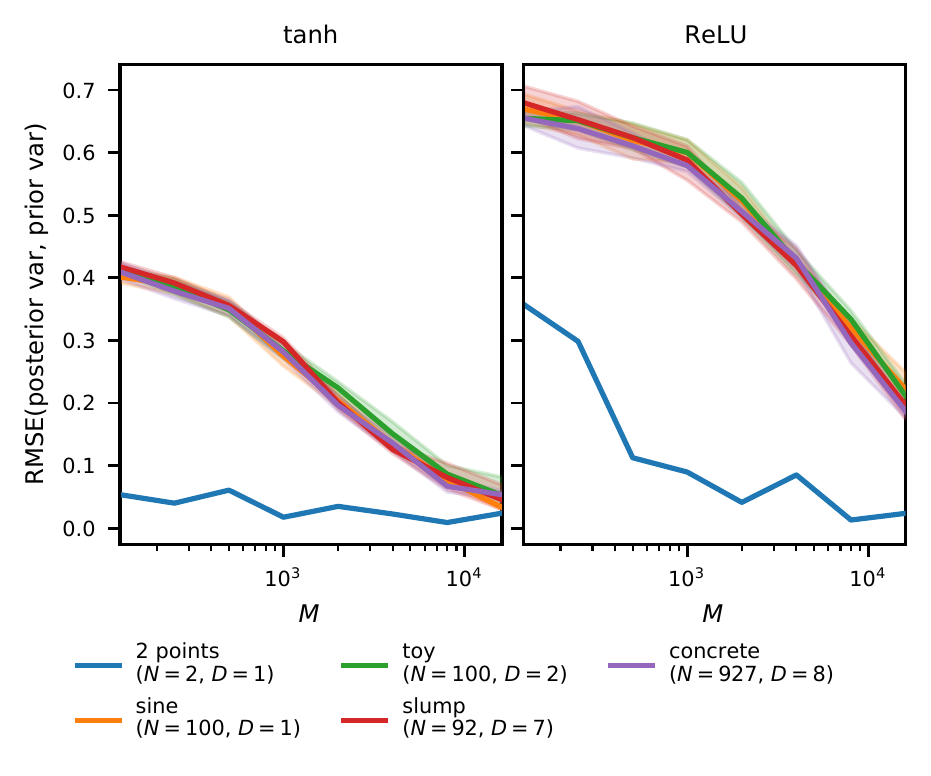}
	\caption{Root mean squared error (RMSE) of the posterior variance to the prior variance.}
	\label{fig:many_datasets_var}
\end{figure}

\paragraph{Architecture.} We only train single-layer ($L=1$), feed-forward networks of varying width $M$, as defined in Equations \ref{eqn:nn-output}-\ref{eqn:nn-base}. Notice we use the NTK parameterization, which scales the post-activations by the network width \citep{jacot_2020}. Depending on the experiment, we use ReLU, tanh, or erf activation functions, where $\erf{z} = 2/\sqrt{\pi} \int_0^z e^{-t^2}~dt$ is the error function. 

\paragraph{Prior and Likelihood.} We use a $\vtheta \sim \Normal(\vnull, \mI)$ prior for the neural network weights and a heteroscedastic Gaussian likelihood with a known variance. Except for the counterexample dataset, we set this variance to $0.025$.

\paragraph{Variational Family.} We use a mean-field Gaussian family for variational inference. In other words, we model each weight and bias parameter of the neural network by an independent Gaussian distribution. 

\paragraph{Initialization of Variational Parameters.} We initialize the variational mean and variance parameters from a normal-inverse-gamma family. Specifically, for any weight (or bias) of the neural network $\theta$, let $\mathcal{N}(\mu_Q, \sigma_Q^2)$ denote its variational distribution. We randomly initialize $\mu_Q\sim\mathcal{N}(0,1)$ and $\sigma_Q^2 \sim \mathcal{I}\mathcal{G}(\nu+1, \nu)$. It follows from the laws of total expectation and variance that $\mathbb{E}[\theta] = 0$ and $\mathbb{V}[\theta] = 2$. This allows the for a width-independent initialization of the weights, as is standard for the NTK parameterization, while allowing the hyperparameter $\nu$ to control the concentration of $\sigma_Q^2$ around its initial mean of one (i.e., $\mathbb{E}[\sigma_Q^2]=1$ and $\mathbb{V}[\sigma_Q]=1/(\nu-1)$). We set $\nu=100$ in our experiments. 

\paragraph{Datasets.} 
\begin{itemize}
    \item \textit{2 points} ($N=2$, $D\ss{i}=1$): This dataset consists of two points: $(-1,-1)$ and $(1,1)$. This dataset is used in all figures in this paper and is shown in \cref{fig:posteriors}.
    
    \item \textit{sine} ($N=100$, $D\ss{i}=1$): This is a synthetic dataset generated by $y = \sin(x) + \epsilon$, where $\epsilon \sim \mathcal{N}(0, .025)$ and $x\sim\text{Unif}(-5,5)$. 
    
    \item \textit{toy} ($N=100$, $D\ss{i}=2$): This is a synthetic dataset generated by $y = x_0 \sin(x_1) + \epsilon$, where $\epsilon \sim \mathcal{N}(0, .025)$,  $x_0\sim\text{Unif}(-5,5)$, and $x_1\sim\text{Unif}(-5,5)$.
    
    \item \textit{counterexample} ($N=2$, $D\ss{i}=1$): This synthetic dataset consists of two observations, $(0,8.24)$ and $(1,11.66)$. It is constructed to meet the conditions of the example discussed in \cref{app:counterexample}. The mean-field posterior predictive of a network with ReLU activation need not converge the prior when trained on this dataset. As part of the construction, we set the observational noise variance of the likelihood to $2.34 \times 10^{-3}$. This dataset is shown in \cref{fig:counterexample}.
    
    \item \textit{slump} ($N=103$, $D\ss{i}=7$): This is the Concrete Slump Test Data Set, available in the UCI Machine Learning Repository \citep{uci_slump}.
    
    \item \textit{concrete} ($N=1030$, $D\ss{i}=8$): This is the Concrete Compressive Strength Data Set, available in the UCI Machine Learning Repository \citep{uci_concrete}.
\end{itemize}
All variables (inputs $x$ and observations $y$) are z-scored standardized (i.e., by subtracting their mean and dividing by their standard deviation). For the synthetic datasets of only two training observations we do not construct test observations. For the larger synthetic datasets, we sample 100 test observations. For the real datasets, we use 10\% of the observations as test observations.

\paragraph{Training Procedure.} We use $20{,}000$ steps of stochastic gradient descent with a batch size of 100, a learning rate of 0.001, and a momentum of 0.9 for optimization. Note that since the post-activations are already scaled by $1/\sqrt{M}$ in the network definition, we do not scale the learning rate with the network width (see, e.g., Appendix F of \citep{lee_2019} for a discussion of the learning rates under the NTK parameterization). We use gradient clipping and cosine annealing of the learning rate, with warm restarts every 500 steps \citep{loshchilov_2017}. To evaluate the ELBO, we use the analytical form of the KL divergence and the reparameterization trick \citep{kingma_2014} with 16 samples to approximate the expected log likelihood term. 

\paragraph{Optimal bias.} In the case of a Gaussian likelihood, we can solve for the optimal variational distribution over bias when all variational parameters are set to the prior. This enables a smaller bound on $KL(Q^*,P)$ in practice. Let $\widetilde{P}$ be the standard prior distribution except with the distribution over the output bias replaced by a normal distribution $\mathcal{N}(\mu_b, \sigma^2_b)$. We will choose $\mu_b$ and $\sigma^2_b$ to maximize the ELBO. 
\begin{align}
\mathrm{ELBO}(\widetilde{P}) &= -\frac{N}{2}\log 2 \pi \sigma^2 - \frac{1}{2\sigma^2} \sum_{n=1}^N \E\left[\parens*{y_n - f(x_n)}^2\right] - \frac{1}{2}\parens*{\mu_b^2 + \sigma_b^2 -1 - \log(\sigma_b^2)} 
\\
& = C - \frac{1}{2\sigma^2} \sum_{n=1}^N \parens*{\E\left[ f(x_n)^2\right] - 2y_n\E\left[ f(x_n)\right]} -  \frac{1}{2}\parens*{\mu_b^2 + \sigma_b^2 - \log(\sigma_b^2)} 
\\
& = C' - \frac{1}{2\sigma^2} \sum_{n=1}^N \parens*{\Var[f(x_n)] + \E\left[ f(x_n)\right]^2 - 2y_n\E\left[ f(x_n)\right]} -  \frac{1}{2}\parens*{\mu_b^2 + \sigma_b^2 - \log(\sigma_b^2)} 
\\
& = C' - \frac{1}{2\sigma^2} \sum_{n=1}^N \parens*{\sigma_b^2 + \mu_b^2 - 2y_n\mu_b} -  \frac{1}{2}\parens*{\mu_b^2 + \sigma_b^2 - \log(\sigma_b^2)}
\\ 
& = C' - \frac{N}{2\sigma^2}\sigma_b^2 - \frac{N}{2\sigma^2} \mu_b^2 - \frac{\mu_b}{\sigma^2}\sum_{n=1}^Ny_n -  \frac{1}{2}\parens*{\mu_b^2 + \sigma_b^2 - \log(\sigma_b^2)},
\end{align}
where $C$ and $C'$ are constants. 
Differentiating with respect to $\mu_b$ and setting to $0$ we have,
\begin{align}
   \mu_b   = \frac{\sum_{n=1}^Ny_n}{N+ \sigma^2} 
\end{align}
and 
Similarly, we can differentiate with respect to $\sigma^2_b$ and set to 0, to obtain,
\begin{align}
    \sigma_b^2 = \frac{\sigma^2}{N+ \sigma^2}. 
\end{align}

\paragraph{Computing a bound on $KL(Q^*,P)$.}
By the optimality of $Q^*$ we have $\operatorname{ELBO}(Q^*) \ge \operatorname{ELBO}(\widetilde{P})$. As in step 3 of \cref{sec:proof-sketch}, it follows that
\begin{align}
    \KL(Q^*, P) &\le \E_{Q^*}[\mathcal{L}(\vtheta)] - \E_{\widetilde{P}}[\mathcal{L}(\vtheta)] + \KL(\widetilde{P}, P)
    \\
    &\le -\E_{\widetilde{P}}[\mathcal{L}(\vtheta)] + \KL(\mathcal{N}(\mu_b,\sigma^2_b), \mathcal{N}(0,1))
    \\
    &\le \frac{1}{2\sigma^2}\sum_{n=1}^N \E_{\widetilde{P}}[(y_n - f(\vx_n))^2] + \frac{1}{2}\parens*{\mu_b^2 + \sigma_b^2 - \log(\sigma_b^2)}. \label{eqn:kl-ub-exp}
\end{align}
For our experiments we compute the expectation by Monte Carlo sampling. Using $\widetilde{P}$ instead of $P$ lowers the upper bound on $\KL(Q^*, P)$ for any dataset for which the increase in the log likelihood from using the optimal bias more than offsets the increase in the KL divergence to the prior (e.g., datasets that are shifted by a constant from the prior mean of zero, as in the counterexample dataset). In the case of a one-hidden layer network, we can evaluate the expectation in \cref{eqn:kl-ub-exp} either using properties of the activation in closed form or up to special function, or via one-dimensional Gaussian quadrature more generally. Additionally, the expectation is independent of $M$.

\paragraph{Figures.} Here we explain a few details specific to each figure \begin{itemize}
    \item \textit{\cref{{fig:posteriors}}}: The shaded region represent $\pm 1$ standard deviation. 
    
    \item \textit{\cref{{fig:convergence}}}: We train on the ``2 points'' dataset. We use $1{,}000$ samples to estimate the posterior predictive mean on a grid of $25$ inputs spaced uniformly over $[-1,1]$. To reduce Monte-Carlo error, we also estimate the predictive mean under the prior with the same random seed. For each $M$, we use the same random seed, so the shaded regions reflect the randomness in the variational parameter initialization only (we use 10 random initializations). We then plot the largest absolute difference from the prior mean of zero. To compute the theoretical bound we use \cref{eqn:1hl-good-constant}, with the KL divergence estimated as in \cref{eqn:kl-ub-exp} and $\norm{\vx}_2^2\le 1$.

    \item \textit{Figures \ref{fig:many_datasets_mean_tanh} and \ref{fig:many_datasets_mean_relu} }: We use 5 different train/test splits (or 5 different random datasets in the case of the synthetic datasets). For each dataset and each $M$, we select the the model with the highest ELBO among two random restarts of the variational parameters.  The shaded regions represent 95\% confidence intervals estimated by boostrapping. To compute the RMSE to the prior, we use $1{,}000$ samples of the posterior predictive evaluated at 100 input points drawn randomly from a uniform distribution over $[-1,1]$ in each input dimension. 
    
    \item \textit{\cref{{fig:counterexample}}}: We train single-layer networks of width $4{,}096{,}000$.
    
\end{itemize}

\end{document}